\documentclass[journal]{IEEEtran}
\usepackage{amsmath,amsfonts}
\usepackage{algorithm}
\usepackage{algpseudocode}
\usepackage{color}
\usepackage{soul}
\usepackage{cite}
\usepackage{graphicx}
\usepackage{amsthm}
\usepackage{multirow}
\def\BibTeX{{\rm B\kern-.05em{\sc i\kern-.025em b}\kern-.08em
    T\kern-.1667em\lower.7ex\hbox{E}\kern-.125emX}}
\usepackage{subfigure}
\usepackage{amsmath}

\DeclareMathOperator*{\argmin}{arg\,min}
\usepackage{booktabs}
\usepackage{tabularx}
\usepackage{tabulary}
\usepackage{hyperref}
\usepackage{mathtools}
\DeclareUnicodeCharacter{0308}{\"}

\begin{document}

\title{Fast and Scalable Game-Theoretic Trajectory Planning with Intentional Uncertainties}
\author{
Zhenmin Huang, Yusen Xie, Benshan Ma, Shaojie Shen, and Jun Ma

\thanks{Zhenmin Huang, Yusen Xie, and Benshan Ma are with the Robotics and Autonomous Systems Thrust, The Hong Kong University of Science and Technology (Guangzhou), Guangzhou, China (email: zhuangdf@connect.ust.hk; bma224@connect.hkust-gz.edu.cn; yxie827@connect.hkust-gz.edu.cn).}
\thanks{Shaojie Shen is with the Department of Electronic and Computer Engineering, The Hong Kong University of Science and Technology, Hong Kong SAR, China (email: eeshaojie@ust.hk).}
\thanks{Jun Ma is with the Robotics and Autonomous Systems Thrust, The Hong Kong University of Science and Technology (Guangzhou), Guangzhou, China, also with the Division of Emerging Interdisciplinary Areas, The Hong Kong University of Science and Technology, Hong Kong SAR, China, and also with the HKUST Shenzhen-Hong Kong Collaborative Innovation Research Institute, Futian, Shenzhen, China (e-mail: jun.ma@ust.hk).}

}

\maketitle

\begin{abstract}

Trajectory planning involving multi-agent interactions has been a long-standing challenge in the field of robotics, primarily burdened by the inherent yet intricate interactions among agents.
While game-theoretic methods are widely acknowledged for their effectiveness in managing multi-agent interactions, significant impediments persist when it comes to accommodating the intentional uncertainties of agents.
In the context of intentional uncertainties, the heavy computational burdens associated with existing game-theoretic methods are induced, leading to inefficiencies and poor scalability. 
In this paper, we propose a novel game-theoretic interactive trajectory planning method to effectively address the intentional uncertainties of agents, and it demonstrates both high efficiency and enhanced scalability.
As the underpinning basis, we model the interactions between agents under intentional uncertainties as a general Bayesian game, and we show that its agent-form equivalence can be represented as a potential game under certain minor assumptions.
The existence and attainability of the optimal interactive trajectories are illustrated, as the corresponding Bayesian Nash equilibrium can be attained by optimizing a unified optimization problem. Additionally, we present a distributed algorithm based on the dual consensus alternating direction method of multipliers (ADMM) tailored to the parallel solving of the problem, thereby significantly improving the scalability.
The attendant outcomes from simulations and experiments demonstrate that the proposed method is effective across a range of scenarios characterized by general forms of intentional uncertainties. Its scalability surpasses that of existing centralized and decentralized baselines, allowing for real-time interactive trajectory planning in uncertain game settings.

The source code will be available on \url{https://github.com/zhuangdf/Potential-Bayesian-Game-release.}

\end{abstract}

\begin{IEEEkeywords}
Multi-agent system, game theory, Bayesian game, potential game, alternating direction method of multipliers (ADMM), non-convex optimization.
\end{IEEEkeywords}

\theoremstyle{definition}
\newtheorem{definition}{Definition}
\theoremstyle{definition}
\newtheorem{remark}{Remark}
\theoremstyle{definition}
\newtheorem{assumption}{Assumption}
\theoremstyle{definition}
\newtheorem{lemma}{Lemma}
\theoremstyle{definition}
\newtheorem{theorem}{Theorem}
\graphicspath{ {pictures_r1/} }

\section{Introduction}
Recent research has witnessed an increasing interest in trajectory planning problems with multi-agent interaction, which emerge from a wide range of fields like swarm robotics~\cite{turgut2008self,zhou2022swarm}, cooperative autonomous driving~\cite{huang2024universal}, multi-robot coordination~\cite{yu2021distributed}, and so on. Although fruitful results are obtained with various applications, this category of problems is particularly challenging and is far from being resolved, especially when a large number of agents are involved. 
Depending on the application, agents may share information and cooperate, or they may operate independently without communication. The latter case further complicates the problem, as each agent needs to reason about the intentions and the behaviors of others, which are highly interactive and coupled with its own behaviors. In light of this, methods based on game theory are gaining more and more attention, owing to their superiority in modeling the interactions between selfish and non-communicative agents. In a typical game-based trajectory planning method, each agent aims to obtain an optimal trajectory that maximizes its interests while subject to some mutual safety concerns like collision avoidance constraints, reaching a Nash Equilibrium (NE) or Generalized Nash Equilibrium (GNE), such that each agent's trajectory is optimal given all other agents' trajectories fixed. By obtaining the NE/GNE, prediction of rivals' behaviors and trajectory planning are performed in a coupled manner to avoid the ``frozen robot" problem~\cite{trautman2010unfreezing},  which often occurs in trajectory planning methods that perform prediction and planning in a decoupled fashion and recognize other agents as immutable moving objects. 





Despite current advances in this field, a series of persistent and significant challenges remain unsolved in developing game-based methods for multi-agent trajectory planning. 
In particular, it is certainly unclear from the available existing reported works on how to deal with the game effectively with the incorporation of uncertainties, which are inherent in multi-agent interactions due to factors like sensing noises, bounded rationalities, unknown intentions of the game players, etc.
Essentially, these ubiquitous uncertainties generally have significant impacts on the robustness of game-based planners, as they may lead to unpredictable behaviors of the game players. Notwithstanding its importance, incorporating uncertainties in game methods is far from being trivial, as it can easily render the game mathematically intractable. 
Many existing game-based methods simply avoid the mathematical difficulties by neglecting the uncertainties, as they focus on obtaining a single equilibrium with the strong assumption that all agents will follow the trajectories featured by that equilibrium~\cite{fridovich2020efficient,le2022algames,hang2021decision,hang2021cooperative,fisac2019hierarchical}. However, this assumption is potentially erroneous, as the discrepancy between the obtained equilibrium and the actual behaviors of agents could lead to potential risk or even catastrophic results. Meanwhile, other methods endeavor to model uncertainties and/or multimodality in multi-agent interaction from different perspectives, resulting in various trajectory planning methods based on stochastic game~\cite{schwarting2021stochastic,zhong2022chance}, maximal entropy game~\cite{mehr2023maximum,so2022multimodal}, multi-hypothesis game~\cite{laine2021multi,so2023mpogames}, and Bayesian game~\cite{deng2022lane,zhao2023non,zhang2022human,huang2024integrated}. Nonetheless, they are either grounded on very strong and specific assumptions, like the Gaussian assumption, and therefore can only address limited cases of uncertainties, or are only applicable to simple scenarios containing a limited number of agents with poor scalability. It is thus paramount that appropriate efforts should be pursued to develop a scalable game-based interactive planning method with general multimodal uncertainties fully addressed, and it is equally difficult to demonstrate the attainability or even the existence of an equilibrium under these game settings.

In this paper, we introduce a novel game-based interactive trajectory planning method to address the aforementioned challenges. We address the uncertainties and the multimodal nature of the trajectory planning problem with multi-agent interaction by formulating the interaction between different agents as a general Bayesian game, of which the corresponding equilibrium is shown to be optimal over the mathematical expectations of agents' underlying intentions. We show that the Bayesian game is essentially a potential game, and the corresponding Bayesian Nash Equilibrium (BNE) can be obtained by solving a unified optimization problem. Further, a novel distributed algorithm based on dual consensus alternating direction method of multipliers (ADMM) is introduced, which decouples the optimization problem and enables a parallelizable solving scheme. Evidently thus, the scalability is greatly enhanced, and computational efficiency is improved to meet the real-time requirement in trajectory planning applications. The contributions of the paper are listed as follows:
\begin{itemize}
\item[$\bullet$]
We provide the key insight that the Bayesian game for interactive trajectory planning is essentially a potential game with minor assumptions, and therefore it is equivalent to a unified optimization problem (instead of a series of coupled optimization problems). This merit demonstrates the existence of the equilibrium and favorably facilitates the determination of the equilibrium through off-the-shelf nonlinear optimization solvers. 
\end{itemize}
\begin{itemize}
\item[$\bullet$]
We show that interactive contingency planning, formulated as a contingency game by the existing literature, is also a potential game. Therefore, conclusions drawn on the existence and attainability of equilibrium naturally hold. These results further expand the scope of this paper to the field of contingency planning, and facilitate the development of a tractable and scalable algorithm for solving the interactive contingency planning problem.
\end{itemize}
\begin{itemize}
\item[$\bullet$]
Tailored to the optimization problem arising from the potential Bayesian game, we propose a novel distributed algorithm based on dual consensus ADMM, which exploits the sparsity in the optimization problem to obtain a parallelizable solving scheme. The proposed method demonstrates superiority in scalability and enables efficient solving for the BNE that meets the real-time requirements.
\end{itemize}
\begin{itemize}
\item[$\bullet$]
Simulations and real-world experiments on various scenarios are performed to showcase the superiority of the proposed method in terms of computational efficiency and scalability, and to illustrate its effectiveness in addressing general intentional uncertainties in multi-agent interactions.
\end{itemize}
\begin{itemize}
\item[$\bullet$]
We also release an open-source multi-process implementation of the proposed game solver that achieves real-time performance, which facilitates future research into game-based interactive trajectory planning.
\end{itemize}

The rest of the paper is organized as follows. In Section II, we present a brief overview of existing game methods for interactive trajectory planning. Preliminaries on the ADMM algorithm are included in Section III. In Section IV, we formulate the interactive trajectory planning problem with uncertainties as a Bayesian game, and show that it is essentially a potential game. In Section V, we introduce the distributed solving algorithm based on dual consensus ADMM for the potential game problem. In Section VI, we provide implementation details of the parallelizable game solver. Simulation results are presented in Section VII. The results of the real-world experiments are presented in Section VIII. Finally, conclusions, along with possible future works, are provided in Section IX.

\section{Related Works}
With the pervasive power in modeling the interaction between multiple agents, game-based methods have been well recognized as a promising way in handling interactive trajectory planning problems, with successful applications across different fields~\cite{spica2020real,wang2021game,li2022efficient,fisac2019hierarchical,dreves2018generalized,hang2021cooperative,hang2021decision,hang2022decision,liu2022three}. In~\cite{spica2020real}, a real-time game-theoretic planner is developed for autonomous drone racing, with a sensitivity-enhanced iterative best response algorithm to obtain an approximation of the NE. The same method has been applied to autonomous car racing with successful results~\cite{wang2021game}. In~\cite{li2022efficient,fisac2019hierarchical,hang2021cooperative,hang2022decision,hang2021decision,liu2022three,dreves2018generalized}, game-based planners are developed to generate effective autonomous driving strategies across different traffic scenarios, including lane-merging, roundabouts, and intersections. These studies collectively demonstrate the effectiveness of game-based approaches in addressing trajectory planning problems with multi-agent interactions. Nevertheless, they only focus on obtaining a single equilibrium, while uncertainties and the multimodal nature of game solutions are largely overlooked, which could lead to suboptimal solutions in terms of mathematical expectation and risky planning results.

Efforts have been put into tackling the uncertainties in multi-agent dynamic games from different perspectives. As noted in numerous studies, stochastic game~\cite{schwarting2021stochastic} mainly focuses on observation uncertainties, which uses local iterative dynamic programming in Gaussian belief space to solve a game-theoretic continuous POMDP. Meanwhile, maximum-entropy games~\cite{so2022multimodal,mehr2023maximum} attempt to address behavioral uncertainties by adding an extra entropy loss to the objective function, maintaining the diversity of agents' strategies. However, in these methods, the Gaussian assumption is made to render the game problem mathematically tractable, which limits the application scope. Some of the other works try to address intentional uncertainties of agents by investigating multi-hypothesis game methods~\cite{laine2021multi,peters2024contingency}. In~\cite{laine2021multi}, a multi-hypothesis game is introduced, which models the uncertainty about the intention of other agents by constructing multiple hypotheses about the objectives, and obtains an optimal trajectory with respect to the probability distribution of these hypotheses. In~\cite{peters2024contingency}, a contingency game is investigated for interactive contingency planning, which allows the robot to efficiently interact with other agents by generating strategic motion plans conditioned on multiple possible intents for other actors in the scene. However, these methods are far from being real-time, especially when a large number of agents are involved, and the scalability is not ideal. A branch of methods~\cite{deng2022lane,zhao2023non,zhang2022human,ali2019game,yu2018human} aims to tackle the multimodal uncertainties through Bayesian games, in which the corresponding BNE is shown to be two-way optimal over the expectations of different behavioral modes of both agents. Nonetheless, the game settings in these methods are quite simple and scenario-specific. In~\cite{huang2024integrated}, an integrated decision-making strategy is developed based on a Bayesian game and is shown to be generalizable to different traffic scenarios. However, the scalability is poor due to the game tree setting, which generally admits exponential complexity.
 
As readily reported in the existing literature, game-based methods are known to be computationally heavy. To alleviate the computational burden and fulfill the real-time requirements that often occur in interactive trajectory planning, efforts have been put into developing fast and general game solvers. Among different types of games, linear quadratic games (LQ games)~\cite{bacsar1998dynamic} in which agents are assumed to possess linear dynamics and quadratic objectives are known to admit analytical solutions. The fact that NE strategies of finite-horizon LQ games satisfy the coupled Riccati equation enables the development of a fast game solver based on discrete-time dynamic programming. Grounded on these facts, iterative Linear-Quadratic game~\cite{fridovich2020efficient} introduces a local algorithm for a variety of differential games, in which a series of approximated LQ games are solved iteratively to obtain the local NE strategies. Meanwhile, ALGAMES~\cite{le2022algames} introduce a novel game solver, which obtains the GNE by satisfying the first-order optimality conditions with a quasi-Newton root-finding algorithm with constraints enforced using an augmented Lagrangian formulation. It claims to exhibit real-time performance in traffic scenarios with up to 4 players. Nonetheless, the computational efficiency of these game solvers is far from being ideal. 

To further enhance the computational efficiency, potential game~\cite{monderer1996potential,rosenthal1973class} has been intensively investigated. Instead of solving for a series of coupled optimization problems, the equilibrium corresponding to a potential game can be obtained by solving a unified optimization problem with a collective objective function known as the potential function. This characteristic not only demonstrates the existence of equilibrium but also greatly enhances the attainability. Methods based on potential game have been applied to the field of autonomous driving with success~\cite{liu2023potential,liu2023safe,fabiani2019multi}. Efforts of utilizing dynamic potential game for interactive trajectory planning are featured by~\cite{bhatt2023efficient,kavuncu2021potential}. In potential iLQR~\cite{kavuncu2021potential}, the potential game is combined with a centralized iterative linear quadratic regulator (iLQR) algorithm to achieve fast interactive planning for drones. In~\cite{bhatt2025strategic}, multi-agent trajectory planning is formulated as a weighted constrained potential dynamic game to release the symmetric coupling requirement in the original potential game formulation. Although potential game is generally easy to solve, the scale of the unified optimization problem continues to grow with the number of agents increases, resulting in poor scalability. In distributed potential iLQR~\cite{williams2023distributed}, the large potential game is divided into a set of smaller local potential games according to agents' vicinity to enhance scalability. However, this approach is prone to suboptimality and is unable to generate trajectories over a long horizon, since the vicinity of agents may change. 


Additionally, to enhance the scalability while avoiding the suboptimality, recent studies utilize distributed optimization algorithms like ADMM~\cite{boyd2011distributed} to solve games with a large number of agents. There have been some rather promising preliminary studies on the application of ADMM for interactive trajectory planning, as featured by~\cite{zhang2021semi,rey2018fully,saravanos2023distributed}. Particularly in~\cite{saravanos2023distributed}, ADMM is utilized to solve the cooperative trajectory planning problem for an excessive number of swarm robots. Nevertheless, the computation efficiency is far from satisfying the real-time requirement. The attainability of NE through ADMM is established in~\cite{boorgens2021admm,salehisadaghiani2017distributed}. These works mainly focus on theoretical analysis without experiments showing whether real-time performance is attained.
Meanwhile, scenario-game ADMM~\cite{li2023scenario} proposes a new sample-based approximation to a class of stochastic, general-sum, pure Nash games, and applies a consensus ADMM algorithm to efficiently compute a GNE of the approximated game. By far, the method introduced in~\cite{jing2025decentralized} is the closest to the proposed method in topic, such that it combines scenario-game ADMM with potential game to generate ramp-merging decisions for autonomous vehicles. Nevertheless, it aims to address uncertainties resulting from sensing noise instead of behavioral and intentional uncertainties corresponding to game agents, and the game formulation is largely different from ours.

\section{Preliminaries}
In this section, we give a brief overview of ADMM, while readers are encouraged to refer to~\cite{boyd2011distributed} for more details. Standard ADMM considers an optimization problem in the following form:
\begin{equation}
\label{ADMM}
\begin{aligned}
\min_{x,z}\ &f(x)+g(z)\\
\textup{s.t.}\ &Ax+Bz = c.
\end{aligned}
\end{equation}
In particular, $x\in\mathbb{R}^p$ and $z\in\mathbb{R}^q$ are primal variables, $A\in\mathbb{R}^{r\times p}$ and $B\in\mathbb{R}^{r\times q}$ are weight matrices, and $c\in\mathbb{R}^r$ is a constant vector. $f:\mathbb{R}^p\rightarrow\mathbb{R}$ and $g:\mathbb{R}^q\rightarrow\mathbb{R}$ are convex, proper, and closed functions. The augmented Lagrangian of problem (\ref{ADMM}) is given as
\begin{equation}
\begin{aligned}
    \mathcal{L}_\rho(x,z,y)=f(x)+g(z)&+y^\top(Ax+Bz-c)\\
    &+(\rho/2)||Ax+Bz-c||^2_2
\end{aligned}
\end{equation}
where $y\in\mathbb{R}^r$ is the dual variable and $\rho$ is the Lagrange coefficient such that $\rho>0$. Standard ADMM update steps are given as
\begin{equation}
\begin{aligned}
x^{k+1}&=\argmin_x\mathcal{L}_\rho(x,z^k,y^k),\\
z^{k+1}&=\argmin_z\mathcal{L}_\rho(x^{k+1},z,y^k),\\
y^{k+1}&=y^k+\rho(Ax^{k+1}+Bz^{k+1}-c).
\end{aligned}
\end{equation}
It can be shown that if $\mathcal{L}_\rho$ contains a saddle point, then the iterates approach a feasible solution and the objective approaches an optimal value as $k\rightarrow\infty$. 

\section{Potential Game with Intentional Uncertainties}
\subsection{Potential Bayesian Game}
In this section, we formulate the interactive trajectory planning problem with intentional uncertainties as a Bayesian game and demonstrate its equivalence to a potential game to show the accessibility of the corresponding BNE.
For a scenario with $N$ agents involved, we denote the set of participating agents as $\mathcal{N}=\{1,2,...,N\}$. To formulate the interactions between agents as a Bayesian game, we assume that each agent $i\in\mathcal{N}$ has its own type $t_i\in\Tilde{T}_i$, where $\Tilde{T}_i$ denotes the set of types corresponding to agent $i$. Note that these types of agents can refer to intentions like $\{\textup{straight going}, \textup{left turning}\}$ or action styles like $\{\textup{aggressive}, \textup{conservative}\}$, where different types of agent $i$ will generally correspond to different underlying utility functions, and hence different behaviors. Intuitively, each agent does not know the exact types of other agents but can obtain an estimation of the probability distribution of their types through intention prediction based on historical trajectories. We assume that historical information is mutually observable, and therefore, all agents will reach similar conclusions about the probabilities of types of their rivals. Formally, we introduce the following assumption.
\begin{assumption}
A common distribution, $p$, over the probabilities of types of all agents can be established and is known to all agents.
\end{assumption}
Based on this assumption, we conclude that the interactions between agents can be formulated as a Bayesian game of common priors. We use $X^{t_i}$ to denote the concatenated states and control inputs corresponding to type $t_i$ of agent $i$. Namely, $X^{t_i}:=\{x^{t_i}_\tau,u^{t_i}_\tau\}_{\tau\in\mathcal{T}}$, where $x^{t_i}_\tau$ and $u^{t_i}_\tau$ are the state vector and input vector corresponding to type $t_i$ of agent $i$ at the time stamp $\tau$, and $\mathcal{T}=\{0,1,...,T\}$ is the set of all time stamps. Physical constraints such as control limits and dynamic constraints can be expressed as $X^{t_i}\in\mathcal{X}^{t_i}$, where $\mathcal{X}^{t_i}$ is the feasible set. A strategy of agent $i$, $\sigma_i$, can be defined as a map from the space of its type to the space of possible actions, namely $\sigma_i(t_i)\in\mathcal{X}^{t_i}$, and further, the overall strategy $\sigma$ is defined as $\sigma(t):=[\sigma_i(t_i)]_{i\in\mathcal{N}}$. The goal of solving for a Bayesian game is to obtain the corresponding BNE strategy $\sigma^*=[\sigma^*_i]_{i\in\mathcal{N}}$. We define the expected cost of agent $i$ with type $t_i$, given a strategy $\sigma$, as
\begin{equation}
\label{typecost1}
    C_i(\sigma|t_i):=\sum_{t_{-i}\in\Tilde{T}_{-i}}p(t_{-i}|t_i)\hat{C}_{t_i}(\sigma(t_{-i},t_i))
\end{equation}
where $t_{-i}$ is the concatenated type vector of all other agents except $i$, $p(t_{-i}|t_i)$ is the conditional distribution of $t_{-i}$ given $t_i$, and $\hat{C}_{t_i}$ is the cost function corresponding to type $t_i$ of agent $i$. A BNE is thus defined as follows.
\begin{definition}
A strategy $\sigma^*$ is a (pure-strategic) BNE if for all $i\in\mathcal{N}$, for all $t_i\in\Tilde{T}_i$, and all $\sigma_i$, the following inequality holds:
\begin{equation}
    C_i(\sigma^*|t_i)\leq C_i((\sigma_i,\sigma^*_{-i})|t_i).
\end{equation}
\end{definition}
In other words, a BNE strategy for agent $i$, $\sigma^*_i$, is optimal for all its type $t_i\in\Tilde{T}_i$ in terms of expected cost, given the BNE strategy $\sigma^*_{-i}$ of all other agents. 
The above definition clearly shows the significance of a BNE strategy $\sigma^*$ for an interactive trajectory planning problem with intentional uncertainties, in the sense that for an arbitrary agent $i$ and its arbitrary intention $t_i$, the strategy $\sigma^*_i(t_i)$ is optimal over the conditional distribution $p(t_{-i}|t_i)$ of the intentions of all its rivals. Informally speaking, $\sigma_i^*(t_i)$ features the best action to take for agent $i$ with intention $t_i$ under the common prior $p$.

To solve for the BNE, we first introduce the corresponding agent-form game of the original Bayesian game. In the agent-form equivalence, each type of each player in the original Bayesian game is itself an independent player, which is often referred to as a type-player. In other words, the set of players in an agent-form game is given by $\cup_{i\in\mathcal{N}}\Tilde{T}_i$. For a particular type-player $t_i\in\Tilde{T}_i$, its expected cost is given by
\begin{equation}
\label{typecost2}
    C_{t_i}(X):=\sum_{t_{-i}\in\Tilde{T}_{-i}}p(t_{-i}|t_i)\hat{C}_{t_i}(X^{t})
\end{equation}
where $X$ is the joint actions of all type-players, and $X^t$ is the joint actions of all type-players included in the type vector $t=(t_i,t_{-i})$. 
Noted that a type vector $t$ consists of exactly one type from each player $i\in\mathcal{N}$, it is not hard to tell that $X^t$ is the agent-form correspondence of the strategy $\sigma(t)$ in the original Bayesian game, and thus the equivalence between (\ref{typecost1}) and (\ref{typecost2}) is obvious.
We define the following NE for the corresponding agent-form game.
\begin{definition}
    A strategy, $X^*=[[X^{t_i*}]_{t_i\in\Tilde{T}_i}]_{i\in\mathcal{N}}$, is a (pure-strategic) NE corresponding to the agent-form game, if for all $i\in\mathcal{N}$, for all $t_i\in\Tilde{T}_i$, and all $X^{t_i}\in\mathcal{X}^{t_i}$, the following inequality holds:
    \begin{equation}
        C_{t_i}(X^*)\leq C_{t_i}(X^{t_i},X^{-t_i*}).
    \end{equation}
\end{definition}

In the above definition, $X^{-t_i}$ is the joint strategy of all type-players except for type-player $t_i$. Note that it is different from $X^{t_{-i}}$, which is the joint strategy of all type-players in a type vector $t_{-i}$. With the above definition, we introduce the following theorem.

\begin{theorem}\cite[Theorem~9.51]{maschler2020game} 
A strategy $\sigma^*$ is a (pure-strategic) BNE in a Bayesian game with common prior if and only if the strategy $[[\sigma^*(t_i)]_{t_i\in\Tilde{T}_i}]_{i\in\mathcal{N}}$ is an NE in its corresponding agent-form game.
\end{theorem}

Therefore, to solve for the BNE in a Bayesian game, we can instead solve for the NE in its agent-form equivalence, which is a solution to the following Nash Equilibrium Problems (NEPs):
\begin{flalign}&
\begin{aligned}
\textit{Problem 1:}&\\
&\ \ \ \ \ \ \ \ \ \begin{aligned}
\min_{X^{t_i}}\ \ \ &C_{t_i}(X^{t_i},X^{-t_i})\\
\textup{s.t.}\ \ \ &X^{t_i}\in\mathcal{X}^{t_i}
\end{aligned}
\end{aligned}
&\end{flalign}
which is an optimization control problem for all type-players $t_i\in\Tilde{T}_i,i\in\mathcal{N}$. Particularly, for an interactive trajectory planning problem, the expected cost $C_{t_i}(X^{t_i},X^{-t_i})$ in (\ref{typecost2}) can be further given as
\begin{equation}
\label{cost5}
\begin{aligned}
&C_{t_i}(X^{t_i},X^{-t_i})\\
&=\sum_{t_{-i}\in\Tilde{T}_{-i}}p(t_{-i}|t_i)[c_{t_i}(X^{t_i})+\sum_{j\neq i}c_{t_it_j}(X^{t_i},X^{t_j})]\\
\end{aligned}
\end{equation}
$c_{t_i}(X^{t_i})$ is ego cost term related to tracking error, magnitude of control inputs, etc. $c_{t_it_j}(X^{t_i},X^{t_j})$ is pair-wised coupled cost term between agent $i$ with type $t_i$ and agent $j$ with type $t_j$, which aims to avoid collision and ensure safety.

In general, solving Problem 1 for the corresponding NE is not trivial. Nonetheless, we will show that given the specific structure of the cost defined in (\ref{cost5}), together with several additional assumptions, the game featured by Problem 1 can be transformed into a potential game. As a result, the corresponding NE can be obtained by solving a unified optimization problem. We first introduce the definition of a potential game.
\begin{definition}(Exact Potential Game)
Let $N$ be the number of players, $\mathcal{X}$ be the set of action profile over the action sets $\mathcal{X}^i$ of each player, and $u_i:\mathcal{X}\rightarrow\mathbb{R}$ be the utility function for all $i\in\{1,2,...,N\}$. A game with the above action profile and payoff function is said to be an exact potential game if there exists a potential function $P:\mathcal{X}\rightarrow\mathbb{R}$ such that $\forall i, \forall X^{-i}\in\mathcal{X}^{-i},\forall X^i,\hat{X}^i\in\mathcal{X}^{i}$, the following condition holds:
\begin{equation}
P(X^i,X^{-i}) - P(\hat{X}^i,X^{-i}) = u_i(X^i,X^{-i}) - u_i(\hat{X}^i,X^{-i}).
\end{equation}
\end{definition}
Informally speaking, in a potential game, any change to the cost $u_i$ of agent $i$ resulting from the change of its own action $X^i$, is equal to the change of the potential function $P$.
It is shown that local pure strategic NE corresponding to a potential game can then be found by obtaining the local optima of the corresponding potential function\cite{kavuncu2021potential,liu2023potential}. Before showing the equivalence of the game featured by Problem 1 to a potential game, we further introduce two assumptions:
\begin{assumption}
The marginal probability of each type of each player is strictly greater than zero, namely $p(t_i)>0$ for all $t_i\in\Tilde{T}_i,i\in\mathcal{N}$.
\end{assumption}
\begin{assumption}
The coupled terms between two type-players are symmetric, namely $c_{t_it_j}(X^{t_i},X^{t_j})=c_{t_jt_i}(X^{t_j},X^{t_i})$ for all $t_i\in\Tilde{T}_i,t_j\in\Tilde{T}_j,i\in\mathcal{N},j\in\mathcal{N},j\neq i$.
\end{assumption}

With Assumption 2, it is trivial to show that Problem 1 is equivalent to the following problem.
\begin{flalign}
&
\begin{aligned}
\textit{Problem 2:} & \\
&\ \ \ \ \ \ \begin{aligned}
\min_{X^{t_i}}\ \ \ &p(t_i)C_{t_i}(X^{t_i},X^{-t_i})\\
\textup{s.t.}\ \ \ &X^{t_i}\in\mathcal{X}^{t_i}.
\end{aligned}
\end{aligned}
&
\end{flalign}
Namely, any solution to Problem 1 is also a solution to Problem 2 and vice versa. As a result, the game featured by Problem 2 has the same set of NE as the game featured by Problem 1. 

With Assumption 3, we introduce the following theorem.
\begin{theorem}
    The game featured by Problem 2 is a potential game with the following potential function:
\begin{equation}
\label{potential}
\begin{aligned}
P(X)&=\sum_i\sum_{t_i\in\Tilde{T}_i}p(t_i)c_{t_i}(X^{t_i})\\
&+\sum_{i,j\in\mathcal{N},j> i}\sum_{t_i\in\Tilde{T}_i,t_j\in\Tilde{T}_j}p(t_i,t_j)c_{t_it_j}(X^{t_i},X^{t_j}).
\end{aligned}
\end{equation}
\end{theorem}
\begin{proof}
See Appendix A.
\end{proof}

Theorem 2, together with other mathematical derivations, reveal that the corresponding BNE can be obtained by solving the following unified optimization problem:
\begin{flalign}&
\begin{aligned}
\textit{Problem 3:}& \\
&\ \ \begin{aligned}
\min_{X}\ &P(X)\\
\textup{s.t.}\ &X^{t_i}\in\mathcal{X}^{t_i},\ \forall i\in\mathcal{N},\ \forall t_i\in\Tilde{T}_i.
\end{aligned}
\end{aligned}
&\end{flalign}
\begin{remark}
The size of the optimization variable $X$ in Problem 3 is proportional to the total number of type-players $\sum_{i\in\mathcal{N}}|\Tilde{T}_i|$. As a result, computational time to solve Problem 3 in a centralized manner will continue to grow with an increasing number of type-players. For example, solving Problem 3 with classic DDP/iLQR algorithms~\cite{kavuncu2021potential} results in a time complexity of $O((\sum_{i\in\mathcal{N}}|\Tilde{T}_i|)^3)$.
\end{remark}

\subsection{Potential Contingency Game}

Contingency planning, in which an agent maintains multiple contingency plans conditioned on the outcome of an uncertain event, is gaining popularity due to its abilities in addressing uncertainties and enhancing safety. In this section, we provide the key insight that the proposed method can easily be adapted to perform contingency planning by revealing its connection with a multi-hypothesis contingency game~\cite{peters2024contingency}. Specifically, we show that the contingency game is also a potential game with an extra constraint. Therefore, the discussions over the accessibility of the corresponding equilibrium hold naturally.

Consider a set of hypotheses $\Theta$, where each hypothesis $\theta\in\Theta$ corresponds to a possible outcome of a traffic scenario. For each agent $i\in\mathcal{N}$, we assign a type $t^\theta_i$ corresponding to each hypothesis $\theta$. We formulate the contingency game similar to~\cite{peters2024contingency}, where the optimization problem corresponding to the ego agent, given other agents' strategies, is defined as
\begin{equation}
\label{EA}
\begin{aligned}
    \mathcal{S}^{EA}(X^{t^\Theta_{-EA}}):=\argmin_{X^{t^\Theta_\textup{EA}}}\ \ \ &\sum_{\theta\in\Theta}p(\theta)C_{t^\theta_\textup{EA}}(X^{t^\theta_\textup{EA}},X^{t^\theta_\textup{-EA}})\\
\textup{s.t.}\ \ \ &X^{t^\theta_{EA}}\in\mathcal{X}^{t^\theta_{EA}},\ \forall \theta\in\Theta,\\
&u^{t^\theta_\textup{EA}}_{\tau}=u^\textup{EA}_\tau ,\ \forall \tau<t_b, \forall \theta\in\Theta.
\end{aligned}
\end{equation}
In particular, $\textup{EA}$ stands for ego agent. $t_b$ is the branching time, and $p(\theta)$ is the probability of the hypothesis $\theta$. $X^{t^\Theta_{EA}}:=\{X^{t^\theta_{EA}}\}_{\theta\in\Theta}$ and $X^{t^\Theta_{-EA}}:=\{X^{t^\theta_{-EA}}\}_{\theta\in\Theta}$ are the collections of trajectories corresponding to the ego agent and other agents under all hypotheses, respectively. $u^\textup{EA}_\tau$ is the common control input at $\tau$ before the branching time.  By solving this optimization problem, the ego agent manages to maintain an optimal plan in response to other agents for each hypothesis, where extra constraints are imposed to merge the gap between different contingency plans before branching time $t_b$. Meanwhile, we use $\textup{OA}$ to represent other agent, and the optimization problem corresponding to the other agent is defined as follows:

\begin{equation}
\label{OA}
\begin{aligned}
&\mathcal{S}^{OA}_{i,\theta}(X^{t^\theta_{-i}}):=\argmin_{X^{t^\theta_i}}\ C_{t^\theta_i}(X^{t^\theta_i},X^{t^\theta_{-i}})\ \textup{s.t.}\ X^{t^\theta_i}\in\mathcal{X}^{t^\theta_i},\\
& \forall \theta\in\Theta, \forall i\in\mathcal{N}, i\neq \textup{EA}.
\end{aligned}
\end{equation}
 We provide the definition of the NE corresponding to the contingency game as follows.
\begin{definition}
    A strategy $\{X^{t^\Theta_{EA},*},X^{t^\Theta_{-EA},*}\}$ is an NE of the game defined by(\ref{EA}) and (\ref{OA}) if and only if \newline
    1) $X^{t^\Theta_{EA},*}\in\mathcal{S}^{EA}(X^{t^\Theta_{-EA},*})$, and \newline 2) $X^{t^\theta_{i},*}\in\mathcal{S}^{OA}_{i,\theta}(X^{t^\theta_{-i},*}),\forall \theta\in\Theta, \forall i\in\mathcal{N}, i\neq \textup{EA}$.
\end{definition}

In particular, the cost function is defined as 
\begin{equation}
    C_{t^\theta_i}(X^{t^\theta_i},X^{t^\theta_{-i}})=c_{t^\theta_i}(X^{t^\theta_i})+\sum_{j\neq i}c_{t^\theta_it^\theta_j}(X^{t^\theta_i},X^{t^\theta_j}).
\end{equation}
It is not difficult to tell that this definition is in accordance with the definition provided in (\ref{cost5}). Together with Assumptions 2 and 3, we introduce the following theorem.
\begin{theorem}
    The contingency game featured by (\ref{EA}) and (\ref{OA}) is a potential game with the following potential function:
\begin{equation}
\label{potential2}
\begin{aligned}
P'(X)&=\sum_i\sum_{\theta\in\Theta}p(\theta)c_{t^\theta_i}(X^{t^\theta_i})\\
&+\sum_{i,j\in\mathcal{N},j> i}\sum_{\theta}p(\theta)c_{t^\theta_it^\theta_j}(X^{t^\theta_i},X^{t^\theta_j}).
\end{aligned}
\end{equation}
\end{theorem}
\begin{proof}
    See Appendix B.
\end{proof}

Theorem 3 reveals that the corresponding NE can be obtained by the following unified optimization problem:
\begin{flalign}&
\begin{aligned}
\textit{Problem 4:}&\\
&\begin{aligned}
\min_{X}\ &P'(X)\\
\textup{s.t.}\ &X^{t^\theta_i}\in\mathcal{X}^{t^\theta_i},\ \forall i\in\mathcal{N},\ \forall \theta\in\Theta,\\
&u^{t^\theta_\textup{EA}}_{\tau}=u^\textup{EA}_\tau ,\ \forall \tau<t_b, \forall \theta\in\Theta.
\end{aligned}
\end{aligned}
&\end{flalign}
The solution to this optimization problem can be easily obtained by off-the-shelf nonlinear optimization solvers, rendering the corresponding NE readily available. Thus, by showing the potential nature of a contingency game, we introduce a novel scheme for achieving interactive contingency planning.

\begin{remark}
It is not difficult to see that Problem 4 is adapted from Problem 3 by assigning a type $t^\theta_i$ corresponding to each hypothesis $\theta$ for each agent $i$ with correlated probability distribution $p(t^{\theta_1}_i,t^{\theta_2}_j)=p(\theta)$ if $\theta_1=\theta_2=\theta$ and $p(t^{\theta_1}_i,t^{\theta_2}_j)=0$ otherwise, as well as adding additional constraints to merge the gap between trajectories corresponding to different types of the ego agent before branching time $t_b$. These facts reveal that a contingency game can be viewed as a Bayesian game with a correlated probability distribution over types and an extra set of contingency constraints.
\end{remark}

\section{Parallel Optimization for Potential Bayesian Game via Dual Consensus ADMM}
Previous discussions introduce a feasible way to solve the Bayesian game by examining its nature as a potential game. As a result, the corresponding BNE can be readily obtained by solving a unified optimal control problem. Nevertheless, complexity analysis reveals that it is particularly challenging to apply the introduced methodology to applications with strong real-time requirements like trajectory planning for autonomous agents. In fact, a centralized solving scheme can be unaffordable even for a small group of agents, as each agent can typically admit multiple types, and thus the total number of type-players will remain substantial. Meanwhile, inspections of (\ref{potential}) reveal that the coupled cost terms are pairwise, and only exist between type-players corresponding to different players. These facts favorably facilitate the development of a decentralized solving scheme, which exploits the sparsely coupled nature of the problem to enable parallel solving, and thus the computational efficiency can be greatly enhanced. 

In this section, we introduce a novel decentralized optimization scheme based on dual consensus ADMM, which enables a fully parallelizable solution to Problem 3 and is shown to be generalizable to other optimization problems with similar sparsely coupled structures like Problem 4. We first introduce the following definition of the indicator function.

\begin{definition}
The indicator function with respect to a set $\mathbb{\mathcal{X}}$ is defined as
\begin{equation}
\mathcal{I}_\mathcal{X}(x)=
\left\{
\begin{array}{ll}
0 &\textup{if}\ \ x\in\mathcal{X}, \\
\infty &\textup{otherwise}.
\end{array}
\right.
\end{equation}
\end{definition}
With such a definition, we define the individual cost term corresponding to the type-player $t_i$ as
\begin{equation}
    f_{t_i}(X^{t_i}) = p(t_i)c_{t_i}(X^{t_i}) + \mathcal{I}_{\mathcal{X}^{t_i}}(X^{t_i})
\end{equation}
which is a combination of the weighted ego cost term and the indicator function over the feasible set. Further, we assume that the coupled cost term between $t_i$ and $t_j$, $c_{t_i,t_j}(X^{t_i},X^{t_j})$, is a function over linear combination of $X^{t_i}$ and $X^{t_j}$. In other words, there exists a function $g_{(t_i,t_j)}$, such that
\begin{equation}
\label{gdefinition}
\begin{aligned}
    g_{(t_i,t_j)}(Q_{(t_i,t_j),t_i}X^{t_i}&+Q_{(t_i,t_j),t_j}X^{t_j})\\
    &=p(t_i,t_j)c_{t_i,t_j}(X^{t_i},X^{t_j})
\end{aligned}
\end{equation}
for arbitrary $t_i$ and $t_j$, where $Q_{(t_i,t_j),t_i}$ and $Q_{(t_i,t_j),t_j}$ are coefficient matrices. In this case, Problem 3 can be reformulated as the following optimization problem:
\begin{equation}
\label{prime}
\min_{\{X^v\}} \sum_{v\in\mathcal{V}} f_v(X^v)+\sum_{e\in\mathcal{E}}g_e(\sum_{v\in e}Q_{v,e}X^v)
\end{equation}
which can be understood as defined over an undirected graph $(\mathcal{V},\mathcal{E})$. In particular, $\mathcal{V}$ is the set of vertices, which in our case coincides with the set of all type-players $\cup_{i\in\mathcal{N}}\Tilde{T}_i$. $\mathcal{E}$ is the set of edges such that $\mathcal{E}=\{(t_i,t_j)|t_i\in\Tilde{T}_i,t_j\in\Tilde{T}_j,i\in\mathcal{N},j\in\mathcal{N},j>i\}$ where $(t_i,t_j)$ is an undirected edge connecting $t_i$ and $t_j$. Note that only type-players with respect to different players are connected, while two different type-players $t_i$ and $t'_i$ for the same player $i$ are not connected in the graph. $Q_{v,e}$ is the corresponding coefficient matrix. To separate the coupling cost terms from the individual cost terms, the problem can be reformulated as
\begin{equation}
\begin{aligned}
\min_{\{X^v\}}&\ \sum_{v\in\mathcal{V}} f_v(X^v)+\sum_{e\in\mathcal{E}}g_e(w_e)\\
\textup{s.t.}&\ \sum_{v\in e}Q_{v,e}X^v = w_e
\end{aligned}
\end{equation}
where $w_e$ is an auxiliary variable for decomposition. The Lagrangian is given as 
\begin{equation}
\begin{aligned}
\mathcal{L}(X,w,y) = \sum_{v\in\mathcal{V}} f_v(X^v)&+\sum_{e\in\mathcal{E}}g_e(w_e)\\
&+\sum_{e\in\mathcal{E}}y_e^\top(\sum_{v\in e}Q_{v,e}X^v-w_e)
\end{aligned}
\end{equation}
where $\{y_e\}_{e\in\mathcal{E}}$ are the Lagrange dual variables. The Lagrange dual function is thus defined as
\begin{equation}
\begin{aligned}
h(y)&=\inf_{X,w}\mathcal{L}(X,w,y)\\
&=\inf_X[\sum_{v\in\mathcal{V}} (f_v(X^v)+\sum_{e\in \textup{Adj}(v)}y_e^\top Q_{v,e}X^v)]\\
&+\inf_w[\sum_{e\in\mathcal{E}}(g_e(w_e)-y_e^\top w_e)]\\
&=-\sum_{v\in\mathcal{V}}f^*_v(-\sum_{e\in \textup{Adj}(v)}y_e^\top Q_{v,e})-\sum_{e\in\mathcal{E}}g^*_e(y_e)\\
&=-\sum_{v\in\mathcal{V}}(f^*_v(-\sum_{e\in \textup{Adj}(v)}y_e^\top Q_{v,e})+\sum_{e\in \textup{Adj}(v)}\frac{1}{N_e}g^*_e(y_e))
\end{aligned}
\end{equation}
where $f^*_v$ and $g^*_e$ are the convex conjugate functions of $f_v$ and $g_e$, respectively. $N_e$ is the number of vertices connected by edge $e$. $\textup{Adj}(v)$ denotes the set of edges that connect the vertex $v$. Maximizing the Lagrange dual function with respect to $y$, $\max_y h(y)$, results in the following Lagrange dual problem:
\begin{equation}
\label{dual}
\begin{aligned}
\min_{\{y_v\}}&\ \sum_{v\in\mathcal{V}}(f^*_v(-Q_v^\top y_v)+g^*_v(y_v))\\
\textup{s.t.}&\ y_v=[y_e]_{e\in\textup{Adj}(v)}.
\end{aligned}
\end{equation}
In particular, $Q_v$ is the concatenated coefficient matrix defined as $Q_v=[Q_{v,e}]_{e\in\textup{Adj}(v)}$. $y_v$ is the concatenated dual variables with respect to vertex $v$. $g^*_v(y_v)=\sum_{e\in \textup{Adj}(v)}\frac{1}{N_e}g^*_e(y_e)$ is the convex conjugate function of $g_v$, which is defined as $g_v(w_v)=\sum_{e\in \textup{Adj}(v)}\frac{1}{N_e}g_e(N_ew_e)$. It can be seen from (\ref{dual}) that the Lagrange dual problem of (\ref{prime}) has a decomposed consensus structure. As a result, consensus ADMM can be applied to solve this problem in a distributed manner.

To derive the corresponding consensus ADMM algorithm, we first rewrite the constraint in (\ref{dual}) as 
$E_{v,e}y_v=y_e$, where $E_{v,e}$ is a constant matrix such that by pre-multiplying $E_{v,e}$, the section of $y_v$ corresponding to $y_e$ are selected. Problem (\ref{dual}) can be reformulated as
\begin{equation}
\label{dual2}
\begin{aligned}
\min_{\{y_v\}}&\ \sum_{v\in\mathcal{V}}(\psi_v(y_v)+g^*_v(z_v))\\
\textup{s.t.}&\ E_{v,e}y_v=y_e,\ y_v=z_v
\end{aligned}
\end{equation}
where $\psi=f^*_v\circ (-Q^\top_v)$.
The augmented Lagrangian corresponding to problem (\ref{dual}) is defined as
\begin{equation}
\begin{aligned}
&\mathcal{L}_{\sigma,\rho}(y_v,z_v,y_e,s,\lambda) = \\
&\sum_{v\in\mathcal{V}}[\psi_v(y_v)+g^*_v(z_v)+s_v^\top(y_v-z_v)+\frac{\sigma}{2}||y_v-z_v||^2\\
&+\sum_{e\in\textup{Adj}(v)}(\lambda_{v,e}^\top(E_{v,e}y_v-y_e)+\frac{\rho}{2}||E_{v,e}y_v-y_e||^2)]
\end{aligned}
\end{equation}
where $s_v$ and $\lambda_{v,e}$ are dual variables. $\sigma$ and $\rho$ are positive constants. The following update steps follow from the standard ADMM algorithm
\begin{subequations}
\begin{align}
y_v^{k+1}&=\underset{y_v}{\arg\min}\ \{\psi_v(y_v)+s_v^{k\top} y_v+\frac{\sigma}{2}||y_v-z^k_v||^2\notag\\
&+\sum_{e\in\textup{Adj}(v)}(\lambda_{v,e}^{k\top}E_{v,e}y_v+\frac{\rho}{2}||E_{v,e}y_v-y^k_e||^2)\},\label{yv}\\
z_v^{k+1}&=\underset{z_v}{\arg\min}\ \{g^*_v(z_v)-s_v^{k\top} z_v+\frac{\sigma}{2}||y^{k+1}_v-z_v||^2\},\label{zv}\\
y_e^{k+1}&=\underset{y_e}{\arg\min}\ \{\sum_{v\in e}-\lambda_{v,e}^{k\top} y_e+\frac{\rho}{2}||E_{v,e}y^{k+1}_v-y_e||^2\},\label{ye}\\
s^{k+1}_v&=s^k_v+\sigma(y^{k+1}_v-z^{k+1}_v),\label{sv}\\
\lambda^{k+1}_{v,e}&=\lambda^k_{v,e}+\rho(E_{v,e}y^{k+1}_v-y^{k+1}_e).\label{lambdave}
\end{align}
\end{subequations}
To derive a fully parallel optimization scheme with respect to vertices only, we hope to eliminate $y_e$ and $\lambda_{v,e}$ from the above updates. Analytical solution to (\ref{ye}) can be obtained as
\begin{equation}
\label{ye_analytical}
    y^{k+1}_e=\frac{1}{N_e}\sum_{v\in e}E_{v,e}y^{k+1}_v+\frac{1}{\rho N_e}\sum_{v\in e}\lambda^k_{v,e}.
\end{equation}
Plugging (\ref{ye_analytical}) into (\ref{lambdave}) yields
\begin{equation}
\begin{aligned}
\lambda^{k+1}_{v,e}&=\frac{N_e-1}{N_e}\lambda^k_{v,e}-\frac{1}{N_e}\sum_{v'\in e/\{v\}}\lambda^k_{v',e}\\
&+\frac{\rho(N_e-1)}{N_e}E_{v,e}y^{k+1}_v-\frac{\rho}{N_e}\sum_{v'\in e/\{v\}}E_{v',e}y^{k+1}_{v'}
\end{aligned}
\end{equation}
which clearly implies $\sum_{v\in e}\lambda^k_{v,e}=0$ for all $k$. $e/\{v\}$ denotes the set of vertices connected by edge $e$ excluding $v$. Therefore, (\ref{ye_analytical}) can be simplified as
\begin{equation}
\label{ye_short}
    y^{k+1}_e=\frac{1}{N_e}\sum_{v\in e}E_{v,e}y^{k+1}_v
\end{equation}
and further, (\ref{lambdave}) can be simplified as
\begin{equation}
\lambda^{k+1}_{v,e}=\lambda^{k}_{v,e}+\frac{\rho}{N_e}\sum_{v'\in e}(E_{v,e}y^{k+1}_v-E_{v',e}y^{k+1}_{v'}).
\end{equation}
We further define $\lambda^k_v=[\lambda^k_{v,e}]_{e\in\textup{Adj}(v)}$ such that $\lambda^{k\top}_v=\sum_{e\in\textup{Adj}(v)}\lambda_{v,e}^{k\top}E_{v,e}$. Plug (\ref{ye_short}) into (\ref{yv}) to eliminate $y_e$ from (\ref{yv}), and we have
\begin{equation}
\begin{aligned}
y_v^{k+1}&=\underset{y_v}{\arg\min}\ \{\psi_v(y_v)+(s_v^k+\lambda_v^k)^{\top} y_v+\frac{\sigma}{2}||y_v-z^k_v||^2\notag\\
&+\sum_{e\in\textup{Adj}(v)}\frac{\rho}{2}||E_{v,e}y_v-\frac{1}{N_e}\sum_{v'\in e}E_{v',e}y^{k}_{v'}||^2\}.
\end{aligned}
\end{equation}

\begin{algorithm}[t]
\caption{Dual Consensus ADMM for Solving (\ref{prime})
}\label{alg:alg1}
\begin{algorithmic}[1]
\State \textbf{choose} $\sigma,\rho >0$
\State \textbf{initialize} $y^{0}_v,z^{0}_v,s^{0}_v,\lambda^{0}_v,X^{0,v}\ \forall v \in \mathcal{V}$
\State \textbf{repeat}: 
\State \hspace{0.2cm} \textbf{For} $v \in \mathcal{V}$ \textbf{in parallel}:
\State \hspace{0.5cm} Send $E_{v,e}y^k_v$ to $v'$ for all $e\in\textup{Adj}(v),v'\in e/\{v\}$
\State \hspace{0.5cm} $r^{k+1}_v=\sigma z^k_v-\lambda^k_v-s^k_v$
\Statex \hspace{1.8cm} $+\sum_{e\in\textup{Adj}(v)}\sum_{v'\in e}\frac{\rho}{N_e}E^\top_{v,e}E_{v',e}y^k_{v'}$
\State \hspace{0.5cm} $X^{k+1,v}=\underset{X^v}{\arg\min}\ \{f_v(X^v)$
\State \hspace{1.8cm} $+\frac{1}{2(\sigma+\rho)}||Q_vX^v+r^{k+1}_v||^2\}$
\State \hspace{0.5cm} $y^{k+1}_v=\frac{1}{\sigma+\rho}(Q_vX^{k+1,v}+r^{k+1}_v)$
\State \hspace{0.5cm} $z^{k+1}_v=\frac{s_v^k}{\sigma}+y^{k+1}_v-\frac{1}{\sigma}\textup{Prox}^{\frac{1}{\sigma}}_{g_v}(s^k_v+\sigma y^{k+1}_v).$
\Statex \hspace{0.5cm} $s^{k+1}_v=s^k_v+\sigma(y^{k+1}_v-z^{k+1}_v)$
\State \hspace{0.5cm} \textbf{For} $e\in\textup{Adj}(v)$:
\State \hspace{1.0cm} $\lambda^{k+1}_{v,e}=\lambda^{k}_{v,e}+\frac{\rho}{N_e}\sum_{v'\in e}(E_{v,e}y^{k+1}_v-E_{v',e}y^{k+1}_{v'})$
\State \hspace{0.5cm} $\lambda^{k+1}_{v}=[\lambda^{k+1}_{v,e}]_{e\in\textup{Adj}(v)}$
\State \textbf{until} termination criterion is satisfied
\end{algorithmic}
\label{alg1}
\end{algorithm}

By completing the square, this optimization problem is equivalent to
\begin{equation}
y_v^{k+1}=\underset{y_v}{\arg\min}\ \{f^*_v(-Q_v^\top y_v)+\frac{\sigma+\rho}{2}||y^\top_v-\frac{1}{\sigma+\rho}r^{k+1}_v||^2\}
\end{equation}
where
\begin{equation}
\label{rupdate}
    r^{k+1}_v=\sigma z^k_v-\lambda^k_v-s^k_v+\sum_{e\in\textup{Adj}(v)}\sum_{v'\in e}\frac{\rho}{N_e}E^\top_{v,e}E_{v',e}y^k_{v'}.
\end{equation}
Following \cite[Lemma~B.1]{banjac2019decentralized}, we conclude that this optimization problem admits the solution
\begin{equation}
\label{yupdate}
    y^{k+1}_v=\frac{1}{\sigma+\rho}(Q_vX^{k+1,v}+r^{k+1}_v)
\end{equation}
with
\begin{equation}
\label{Xupdate}
X^{k+1,v}=\underset{X^v}{\arg\min}\ \{f_v(X^v)+\frac{1}{2(\sigma+\rho)}||Q_vX^v+r^{k+1}_v||^2\}.
\end{equation}
It will then be shown that the sequence $\{X^{k,v}\}$ will converge to a minimizer of (\ref{prime}). Meanwhile, by completing the square of (\ref{zv}), we conclude that
\begin{equation}
\label{zupdate}
\begin{aligned}
z^{k+1}_v&=\underset{y_v}{\arg\min}\ \{g^*_v(z_v)+\frac{\sigma}{2}||z_v-\frac{s^k_v+\sigma y^{k+1}_v}{\sigma}||^2\}\\
&=\textup{Prox}^\sigma_{g^*_v}(\frac{s_v^k}{\sigma}+y^{k+1}_v)\\
&=\frac{s_v^k}{\sigma}+y^{k+1}_v-\frac{1}{\sigma}\textup{Prox}^{\frac{1}{\sigma}}_{g_v}(s^k_v+\sigma y^{k+1}_v)
\end{aligned}
\end{equation}
where the proximal operator is defined as $\textup{Prox}^\rho_f(x)={\arg\min}_y\ \{f(y)+\frac{\rho}{2}||x-y||^2\}$ and the last equation follows from \cite[Theorem~14.3(ii)]{bauschke2017correction}.

Based on the previous discussion, we introduce Algorithm 1, which is a fully distributed algorithm for solving optimization problems in the form of (\ref{prime}). Formally, we introduce the following assumption and theorem.
\begin{assumption}
    Assume that $f_v$ and $g_e$ are convex, proper, and closed for all $v\in\mathcal{V}$ and $e\in\mathcal{E}$. Also, a prime-dual solution exists for problem (\ref{prime}) and strong duality holds. Further, (\ref{Xupdate}) admits a unique and bounded solution.
\end{assumption}
\begin{theorem}

    The iterates $\{z^{k}_v\}_{v\in\mathcal{V}}$ and $\{y^{k}_v\}_{v\in\mathcal{V}}$ of Algorithm 1 converge to a minimizer of the dual problem (\ref{dual2}), and the iterates $\{X^{k,v}\}_{v\in\mathcal{V}}$ of Algorithm 1 converges to a minimizer of the primal problem (\ref{prime}).
\end{theorem}
\begin{proof}
See Appendix C.
\end{proof}

Particularly, similar dual consensus ADMM methods are also introduced in~\cite{banjac2019decentralized,grontas2022distributed}. Nevertheless, we hope to emphasize that comparing to the algorithms proposed in \cite{banjac2019decentralized,grontas2022distributed}, Algorithm 1 is a novel distributed algorithm for solving problems in the shape of (\ref{prime}), which contains individual cost terms and coupled cost terms defined in a pair-wise or group-wise manner. In Algorithm 1, each node $v$ only computes the dual variables and the auxiliary variables with respect to the coupled cost terms in which $v$ is involved. Therefore, it is particularly suitable for solving the potential Bayesian game defined in Problem 3, where pair-wise coupled cost terms exist only between pairs of type-players $(t_i,t_j)$ corresponding to different players $i$ and $j$. It is also readily applicable to other similar optimization problems that are sparsely coupled. In the simulations, we will also show that the proposed method is superior to the decentralized solving scheme based on methods proposed in~\cite{banjac2019decentralized,grontas2022distributed} in terms of computational efficiency.

\section{Implementation Details}
With the above derivations, a parallel solving scheme for the potential Bayesian game can be developed by applying Algorithm 1 to Problem 3. Similar to \cite{huang2023decentralized}, the individual cost term for $t_i$, $c_{t_i}$, is defined as
\begin{equation}
    c_{t_i}(X^{t_i})=\sum_{\tau\in\mathcal{T}}||x^{t_i}_\tau-x^{t_i}_{\tau,ref}||^2_Q +||u^{t_i}_\tau||^2_R
\end{equation}
where $x^{t_i}_{\tau,ref}$ is the reference state vector for type-player $t_i$ at $\tau$. $Q$ and $R$ are positive semi-definite coefficient matrices. Meanwhile, $||\cdot||^2_W$ is the square of weighted $L_2$-norm $||a||^2_W=a^\top Wa$. As an example for illustration, we assume that all agents possess the same dynamics as a single-track bicycle, which is commonly used for vehicle dynamic constraints. We define $\mathcal{X}^{t_i}$ as the feasible set of the following dynamic constraints
\begin{equation}
    x^{t_i}_{\tau+1}=f(x^{t_i}_\tau,u^{t_i}_\tau)
\end{equation}
which is characterized by the following equations
\begin{equation}
\label{dynamics}
\left\{\begin{aligned}
p_{x,\tau+1} &= p_{x,\tau}+f_r(v_\tau,\delta_\tau)\cos(\theta_\tau),\\
p_{y,\tau+1} &= p_{y,\tau}+f_r(v_\tau,\delta_\tau)\sin(\theta_\tau),\\
\theta_{\tau+1} &= \theta_{\tau}+\arcsin(\frac{\tau_s v_\tau\sin(\delta_\tau)}{b}),\\
v_{\tau+1} &= v_\tau+\tau_sa_\tau,
\end{aligned}
\right.
\end{equation}
where $p_x$ and $p_y$ denote the X and Y coordinates of mid-point of the rear axle of the vehicle, $\theta$ is the heading angle, $v$ is the longitudinal velocity, $\delta$ is the front wheel steering angle, and $a$ is the acceleration. $\tau_s$ is the time interval and $b$ is the wheelbase. $f_r$ is defined as
\begin{equation}
f_r(v,\delta) = b+\tau_sv\cos(\delta)-\sqrt{b^2-(\tau_sv\sin(\delta))^2}.
\end{equation}
Naturally, the state vector and the control input vector are defined as $x^{t_i}_\tau=[p^{t_i}_{x,\tau},p^{t_i}_{y,\tau},\theta^{t_i}_\tau,v^{t_i}_\tau]$ and $u^{t_i}_\tau=[\theta^{t_i}_\tau,a^{t_i}_\tau]$. Meanwhile, the coupled cost terms are corresponding to collision avoidance. Similar to the case of a vehicle, we represent each agent with two identical circles aligned along the longitudinal axis, with the one near the front end denoted as $f$ and the one near the rear end denoted as $r$. Penalties are imposed on distances between each pair of circles corresponding to different agents, namely
\begin{equation}
\begin{aligned}
&c_{t_it_j}(X^{t_i},X^{t_j})=\sum_{\tau\in\mathcal{T}}\sum_{\eta,\gamma\atop\in\{f,r\}} (l^{t_i\eta,t_j\gamma}_\tau)^2,\\
&l^{t_i\eta,t_j\gamma}_\tau=\left\{
\begin{array}{ll}
\sqrt{\beta}(d^{t_i\eta,t_j\gamma}_\tau-d_\textup{safe}) & \textup{if}\ d^{t_i\eta,t_j\gamma}_\tau < d_\textup{safe},\\
0 & \textup{else}
\end{array}
\right.
\end{aligned}
\end{equation}
where $d^{t_i\eta,t_j\gamma}$ is the distance between center of circle $\eta$ of $t_i$ and circle $\gamma$ of $t_j$, $d_\textup{safe}$ is the predefined safety distance, and $\beta$ is a weight parameter.

Nonetheless, it is not appropriate to directly apply Algorithm 1 with these definitions, due to two existing problems: 1) the dynamic model is nonlinear, and therefore its corresponding feasible set $\mathcal{X}^{t_i}$ is non-convex; 2) the definition of the coupled cost term $c_{t_i,t_j}$ does not meet the requirement featured by ($\ref{gdefinition}$). To resolve these problems, we adopt the convexification strategy similar to \cite{huang2023decentralized}. Namely, we perform convexification around current nominal trajectories $X^{t_i}$ and optimize over $\delta X^{t_i}=\{\delta x^{t_i}_\tau,\delta u^{t_i}_\tau\}_{\tau\in\mathcal{T}}$. Linearization of dynamic constraints around current nominal trajectories is performed to ensure that the feasible set is linear and thus convex. Also, Gauss-Newton approximation is applied on $c_{t_it_j}$ to ensure that $g$ is convex and ($\ref{gdefinition}$) is satisfied. Specifically, the ego cost is reformulated as
\begin{equation}
\label{convex1}
    \hat{c}_{t_i}(\delta X^{t_i}) = \sum_{\tau\in\mathcal{T}}||x^{t_i}_\tau-x^{t_i}_{\tau,ref}+\delta x^{t_i}_\tau||^2_Q +||u^{t_i}_\tau+\delta u^{t_i}_\tau||^2_R,
\end{equation}
and the collision avoidance cost is approximated as
\begin{equation}
\label{hat_c_couple}
\begin{aligned}
&\hat{c}_{t_it_j}(\delta X^{t_i},\delta X^{t_j})=\\
&\sum_{\tau\in\mathcal{T}}\sum_{\eta,\gamma\atop\in\{f,r\}}(Q^{t_i\eta,t_j\gamma}_\tau\delta x^{t_i}_\tau+Q^{t_j\gamma,t_i\eta}_\tau\delta x^{t_j}_\tau+l^{t_i\eta,t_j\gamma}_\tau)^2,
\end{aligned}
\end{equation}
where the coefficient matrix $Q^{t_i\eta,t_j\gamma}_\tau=\partial l^{t_i\eta,t_j\gamma}_\tau/\partial x^{t_i}$ and $Q^{t_j\gamma,t_i\eta}_\tau=\partial l^{t_i\eta,t_j\gamma}_\tau/\partial x^{t_j}$. 
We plug the definition of $\hat{c}_{t_it_j}$ into (\ref{gdefinition}), and conclude that the coefficient matrices corresponding to $t_i$ at time stamp $\tau$ is given as
$Q_{{t_i},\tau}=[[\sqrt{p(t_i,t_j)}Q^{t_i\eta,t_j\gamma}_\tau]_{\eta,\gamma\in\{f,r\}}]_{t_j\in \textup{adj}(t_i)}$, where $\textup{adj}(t_i)$ is the set of neighbor vertices of $t_i$ such that $t_j\in\textup{adj}(t_i)$ if and only if $(t_i,t_j)\in\mathcal{E}$. Updates of $r,s,y,\lambda$ are performed in a piecewise manner regarding time stamp $\tau$, namely
\begin{equation}
\label{update1}
\begin{aligned}
r^{k+1}_{t_i,\tau}&=\sigma z^k_{t_i,\tau}-\lambda^k_{t_i,\tau}-s^k_{t_i,\tau}\\
&+\sum_{e\in\textup{Adj}(t_i)}\sum_{t_j\in e}\frac{\rho}{N_e}E^\top_{t_i,e,\tau}E_{t_j,e,\tau}y^k_{t_j,\tau},\\
y^{k+1}_{t_i,\tau}&=\frac{1}{\sigma+\rho}(Q_{t_i,\tau}\delta x^{k+1,t_i}_\tau+r^{k+1}_{t_i,\tau}),\\
s^{k+1}_{t_i,\tau}&=s^k_{t_i,\tau}+\sigma(y^{k+1}_{t_i,\tau}-z^{k+1}_{t_i,\tau}),\\
\lambda^{k+1}_{t_i,e,\tau}&=\lambda^{k}_{t_i,e,\tau}+\frac{\rho}{N_e}\sum_{t_j\in e}(E_{t_i,e,\tau}y^{k+1}_{t_i,\tau}-E_{t_j,e,\tau}y^{k+1}_{t_j,\tau}).\\
\end{aligned}
\end{equation}
Likewise, $E_{t_i,e,\tau}$ is a constant matrix such that $E_{t_i,e,\tau}y^k_{t_i,\tau}$ returns the part of $y^k_{t_i,\tau}$ that correspond to coupled constraints represented by edge $e$. Plug the definition of $g_v$ in (\ref{gdefinition}) into (\ref{zupdate}) yields the following $z$-update scheme
\begin{equation}
\label{update2}
    z^{k+1}_{t_i,\tau} = \frac{1}{4\sigma+1}(4s^k_{t_i,\tau}+4\sigma y^{k+1}_{t_i,\tau}+2l_{t_i,\tau})
\end{equation}
where $l_{{t_i},\tau}=[[\sqrt{q(t_i,t_j)}l^{t_i\eta,t_j\gamma}_\tau]_{\eta,\gamma\in\{f,r\}}]_{t_j\in \textup{adj}(t_i)}$. These piecewise variables aggregate to the original $r,s,y,z,\lambda$ through concatenation as
\begin{equation}
\label{update3}
\begin{aligned}
r^{k}_{t_i}&=[r^{k}_{t_i,\tau}]_{\tau\in\mathcal{T}},\ z^{k}_{t_i}=[z^{k}_{t_i,\tau}]_{\tau\in\mathcal{T}},\\
s^{k}_{t_i}&=[s^{k}_{t_i,\tau}]_{\tau\in\mathcal{T}},\ y^{k}_{t_i}=[y^{k}_{t_i,\tau}]_{\tau\in\mathcal{T}},\\
\lambda^{k}_{t_i,\tau}&=[\lambda^{k}_{t_i,e,\tau}]_{e\in\textup{Adj}(t_i)},\ \lambda^{k}_{t_i}=[\lambda^{k}_{t_i,\tau}]_{\tau\in\mathcal{T}}.
\end{aligned}
\end{equation}
Meanwhile, $\delta X^{k+1,t_i}$ is the optimizer of the optimization problem in Step 7 of Algorithm 1, namely
\begin{equation}
\label{LQR}
\begin{aligned}
&\begin{aligned}
\min_{\delta X^{t_i}} &\ 
\sum_{\tau\in\mathcal{T}} ||\delta x^{t_i}_\tau||^2_Q\prime+q^\top\delta x^{i}_\tau+||\delta u^{i}_\tau||^2_R+2u^{t_i\top}_\tau R\delta u^{i}_\tau
\\
\textup{s.t.}&\ \delta X^{t_i}\in\delta \mathcal{X}^{t_i},\\
\end{aligned}\\
&q=2p(t_i)(x^{t_i}_\tau-x^{t_i}_{\tau,ref})^\top Q+\frac{r^{k+1\top}_{t_i,\tau}Q_{t_i,\tau}}{\sigma+\rho},\\
&Q\prime=p(t_i)Q+\frac{Q_{t_i,\tau}^\top Q_{t_i,\tau}}{2(\sigma+\rho)}.
\end{aligned}
\end{equation}
In particular, $\delta\mathcal{X}^{t_i}$ is the feasible set of the linearized kinematics. It is clear that (\ref{LQR}) is a standard LQR problem, which is readily solvable by applying dynamic programming. On solving (\ref{LQR}), a series of feedback control strategies $\{k^{t_i}_\tau,K^{t_i}_\tau\}_{\tau\in\mathcal{T}}$ are obtained, while update of trajectories is performed by the standard iLQR update scheme together with a linear search scheme, similar to Algorithm 3 in \cite{huang2023decentralized}, which is also included here for completeness:
\begin{equation}
\label{update}
\begin{aligned}
u^{t_i}_\tau&\leftarrow u^{t_i}_\tau+\alpha k^{t_i}_\tau+K^{t_i}_\tau(x^{t_i}_\tau-\hat{x}^{t_i}_\tau),\\
x^{t_i}_{\tau+1}&=f(x^{t_i}_\tau,u^{t_i}_\tau).
\end{aligned}
\end{equation}
$\alpha$ is the line search parameter, and $\hat{x}^{t_i}_\tau$ is the state before update. Convexification, parallel solving, and trajectory updating are performed alternatively until convergence. This process is concluded as Algorithm 2. The convergence criterion in Step 2 of Algorithm 2 is that the change in value of the potential function between two consecutive iterations is smaller than 0.1.
\textup{ADMM\_Max\_Iter} is set to be 3.

Particularly for the potential contingency game defined in Problem 4, it shares the same structure as the potential Bayesian game defined in Problem 3, with extra constraints to enforce the consistency between different contingency plans. Instead of enforcing the hard constraints, we add the following soft penalties
\begin{equation}
    c_\textup{contingency}(X^{t^\Theta_{EA}})=\sum_{\theta'\neq\theta}\sum_{\tau<t_b}||x^{t^\theta_{EA}}_\tau-x^{t^{\theta'}_{EA}}_\tau||^2_{Q_\textup{contingency}},
\end{equation}
where $Q_\textup{contingency}$ is a diagonal weight matrix that penalizes inconsistency between contingency plans. Obviously, when elements of $Q_\textup{contingency}$ go to infinity, the soft penalties become hard constraints. Due to the similarity between problems 3 and 4, Algorithm 2 can also be applied to solve Problem 4 by adding extra edges that connect $t^\theta_{EA}$ and $t^{\theta'}_{EA}$. The decoupling of $ c_\textup{contingency}$ is similar and therefore is omitted for simplicity.

\begin{algorithm}[t]
\caption{Parallel Optimization for Potential Bayesian Game
}\label{alg:alg2}
\begin{algorithmic}[1]
\Statex \textbf{Inputs:}
\Statex \hspace{0.1cm} Connectivity graph $(\mathcal{V},\mathcal{E})$;
\Statex \hspace{0.1cm} Reference trajectories $\{X^{t_i}_{ref}\}_{t_i\in\mathcal{V}}$;
\Statex \hspace{0.1cm} Probabilities $\{p(t_i)\}_{t_i\in\mathcal{V}}$ and $\{p(t_i,t_j)\}_{(t_i,t_j)\in\mathcal{E}}$
\Statex \hspace{0.1cm} Parameters $Q,R,\beta,d_\textup{safe},\sigma,\rho$.
\Statex \textbf{Outputs:}
\Statex \hspace{0.1cm} BNE trajectories $X^*$.
\Statex \textbf{Procedures:}
\State \hspace{0.2cm}\textbf{initialize} all trajectories $\{X^{t_i}\}$.
\State \hspace{0.2cm}\textbf{repeat} until convergence:
\State \hspace{0.7cm} \textbf{For} $t^i$ \textbf{in parallel}:
\State \hspace{1.2cm} Perform convexification by (\ref{convex1}) and (\ref{hat_c_couple}).
\State \hspace{1.2cm} \textbf{For} $k<\textup{ADMM\_Max\_Iter}$:
\State \hspace{1.7cm}Perform updates by (\ref{update1}), (\ref{update2}), (\ref{update3}), and (\ref{LQR}).
\State \hspace{1.2cm} Perform trajectory updates by (\ref{update}).
\end{algorithmic}
\label{alg1}
\end{algorithm}


To enable closed-loop simulations and experiments, we adopt a similar strategy as in \cite{huang2024integrated}. After obtaining the BNE trajectories $X^{*}$, if the ego agent possesses more than one type, the best type $t_i^*$ can be selected as
\begin{equation}
    t_i^* = \underset{t_i\in\Tilde{T}_i}{\arg\min}\ C_{t_i}(X^*).
\end{equation}
With the best type $t_i^*$, the first step of the trajectory $X^{t^*_i*}$ is performed, which resembles the receding horizon control. The common prior over the intentions of all agents is then updated by applying the Bayesian filtering on the observations of all agents' movements.

\section{Simulation Results}

\subsection{Open-Loop Simulations}

In this section, we first examine the potential Bayesian game in Problem 3 by performing open-loop simulations in different traffic scenarios. In each scenario, we assume that the ego agent admits a unique type, while it holds a belief over the underlying types of all other agents. Through open-loop simulations, we show that a spectrum of sets of interactive trajectories can be generated by varying the belief, such that each set of interactive trajectories is a BNE corresponding to the Bayesian game under that belief. These results demonstrate the ability of the proposed method to generate optimal trajectories conditioned on the probabilities of the underlying intentions of the participating agents.

\begin{figure*}[t]
\centering
\subfigure[$\tau=33$]{\includegraphics[scale=0.38]{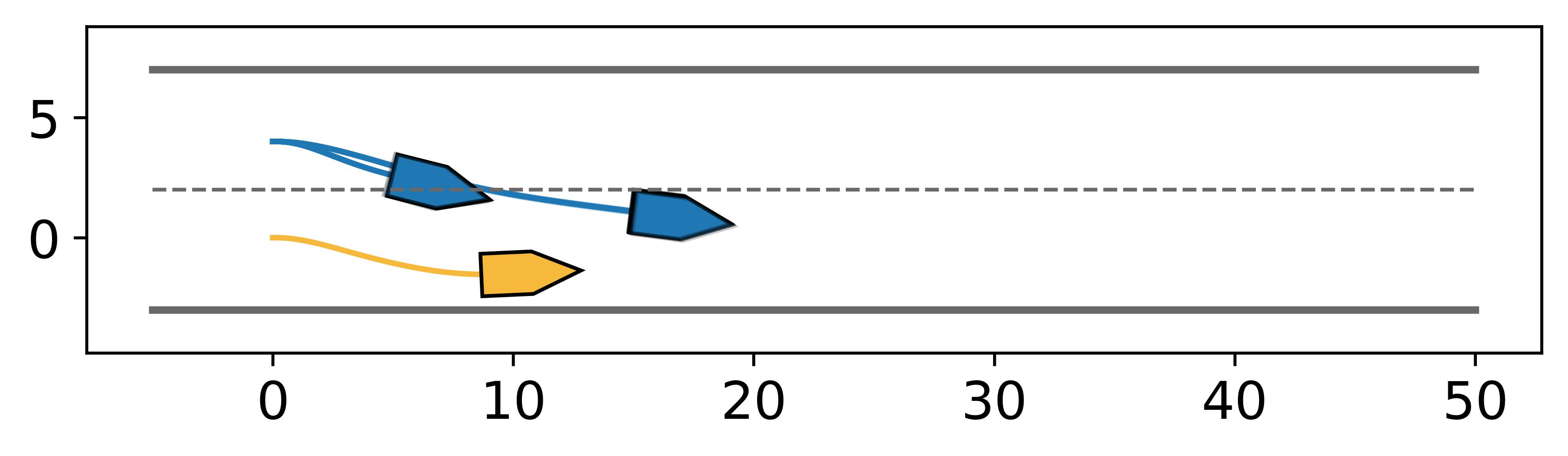}}
\subfigure[$\tau=66$]{\includegraphics[scale=0.38]{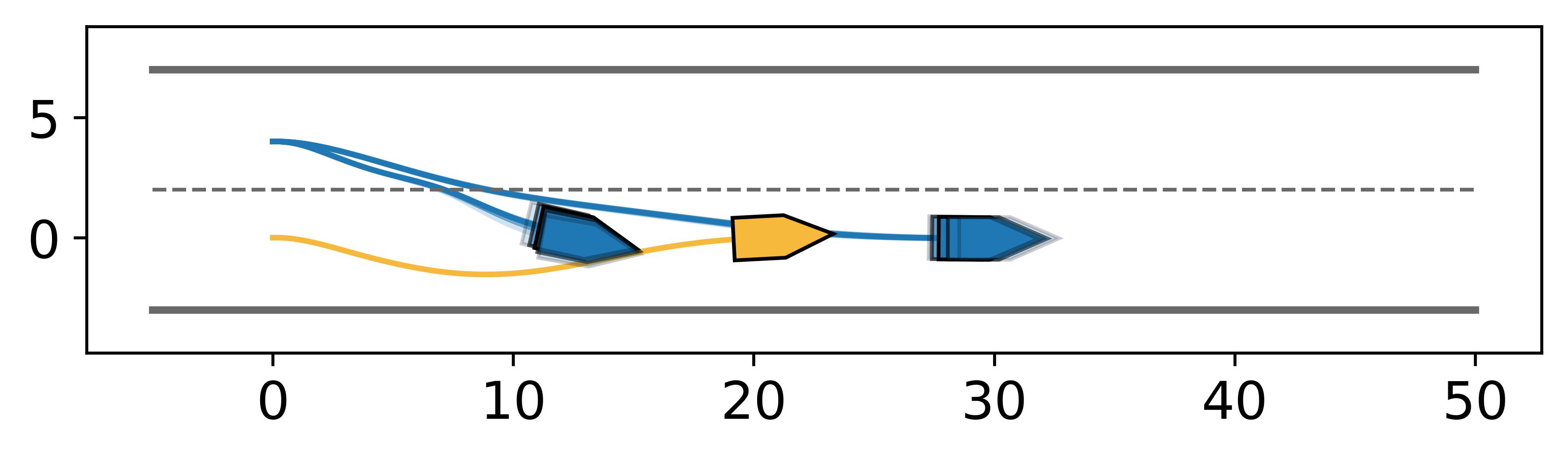}}
\subfigure[$\tau=99$]{\includegraphics[scale=0.38]{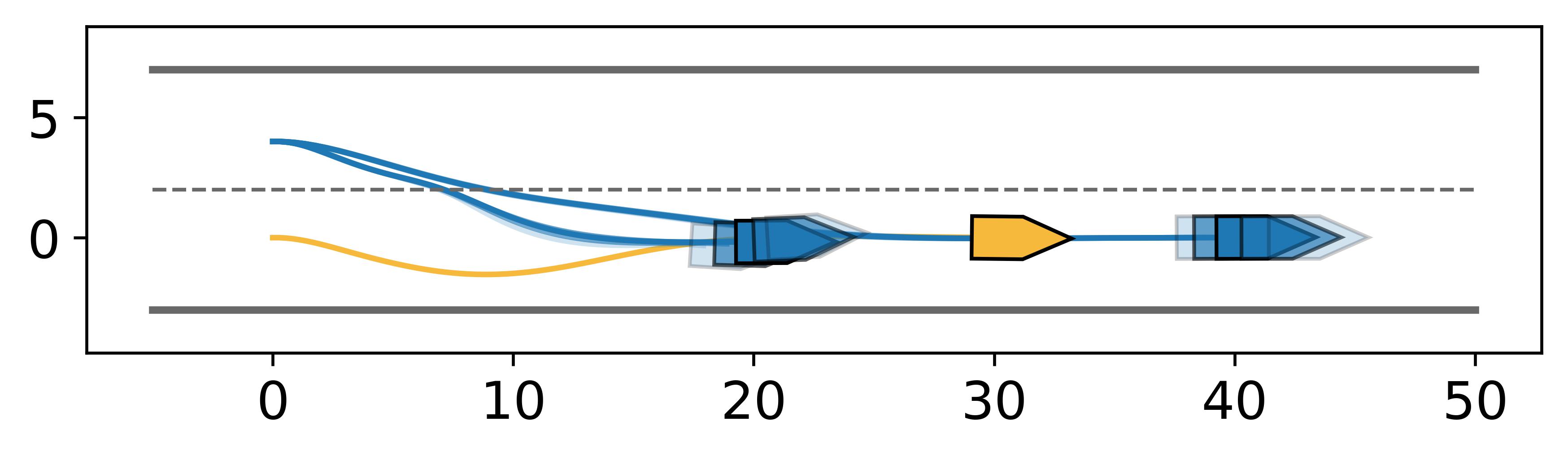}}

\caption{Simulation results in the merging scenario with $[w_1,w_2]=[0.5,0.5]$. The average longitudinal velocity of EA is roughly equal to 3\,m/s, which is the reference velocity.}

\label{fig:M1}
\end{figure*}

\begin{figure*}[t]
\centering
\subfigure[$\tau=33$]{\includegraphics[scale=0.38]{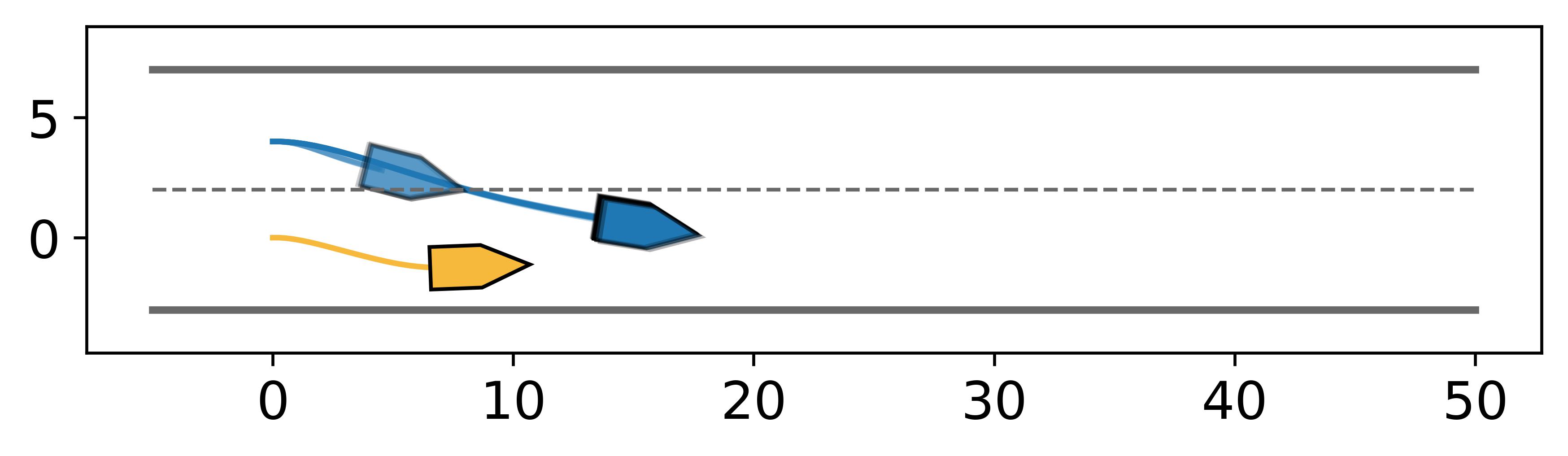}}
\subfigure[$\tau=66$]{\includegraphics[scale=0.38]{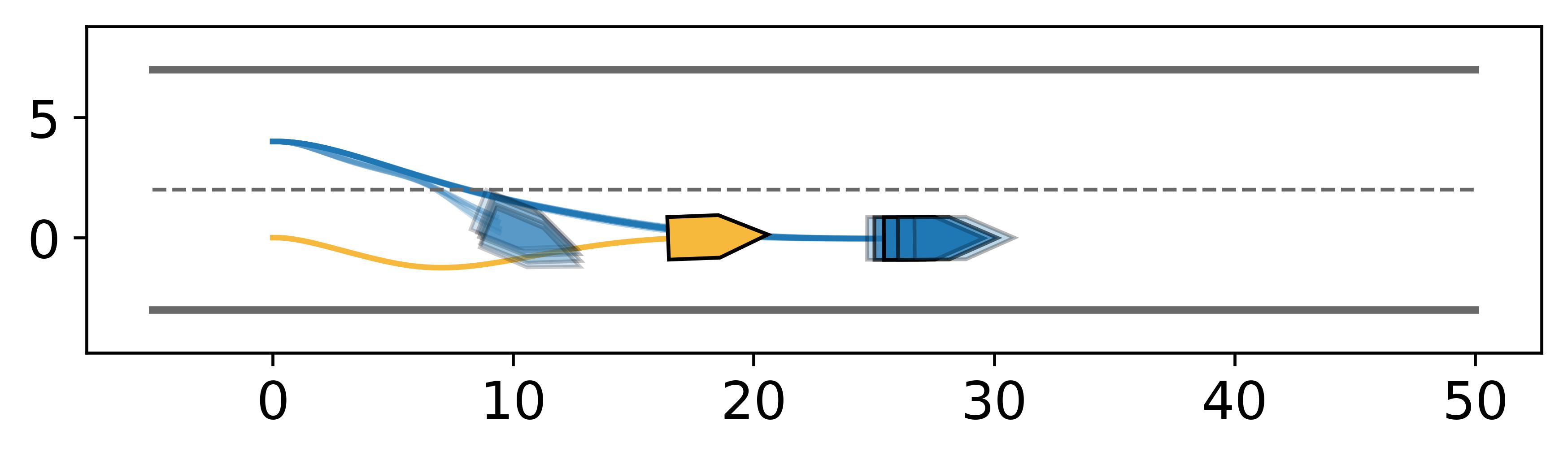}}
\subfigure[$\tau=99$]{\includegraphics[scale=0.38]{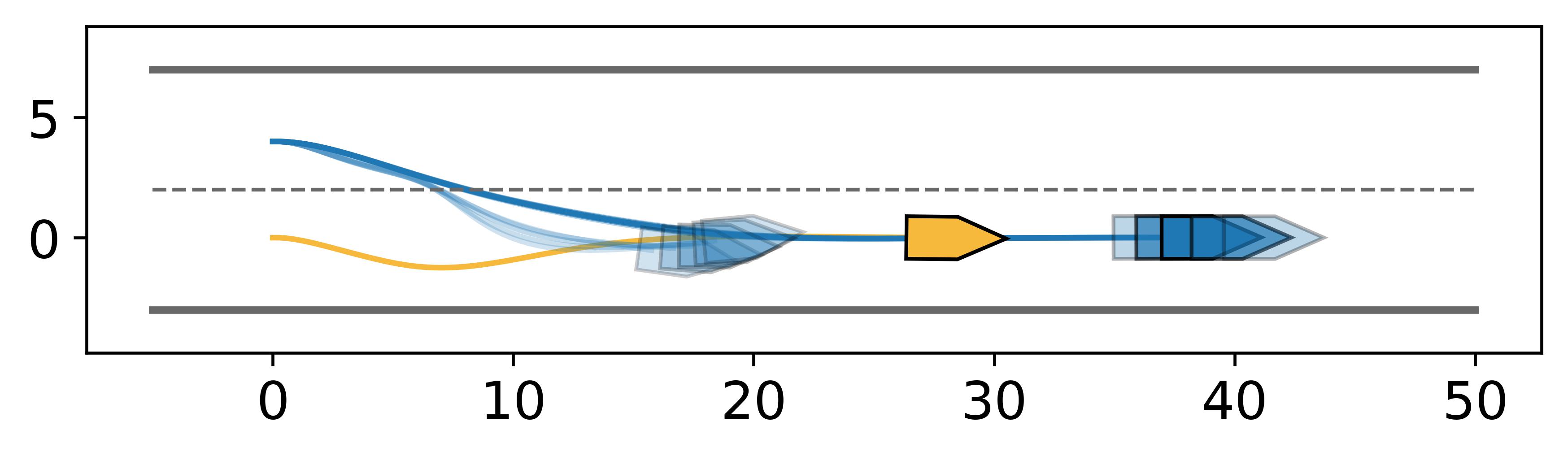}}

\caption{Simulation results in the merging scenario with $[w_1,w_2]=[0.9,0.1]$. The average longitudinal velocity of EA is smaller than 3\,m/s, indicating that it slows down to yield to OA.}

\label{fig:M2}
\end{figure*}

\begin{figure*}[t]
\centering
\subfigure[$\tau=33$]{\includegraphics[scale=0.38]{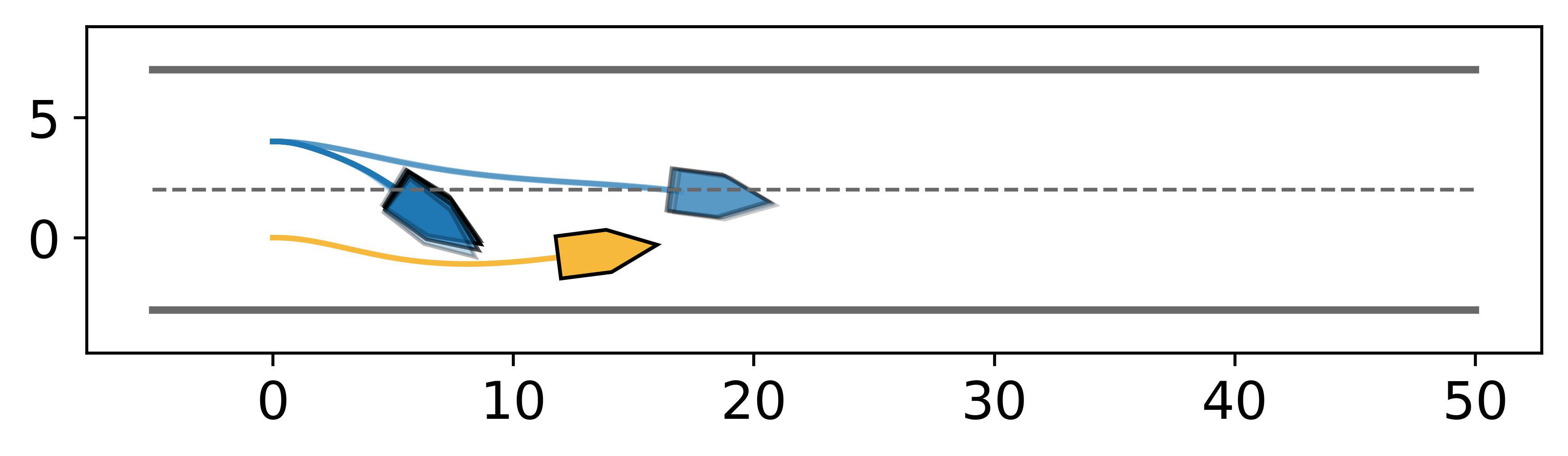}}
\subfigure[$\tau=66$]{\includegraphics[scale=0.38]{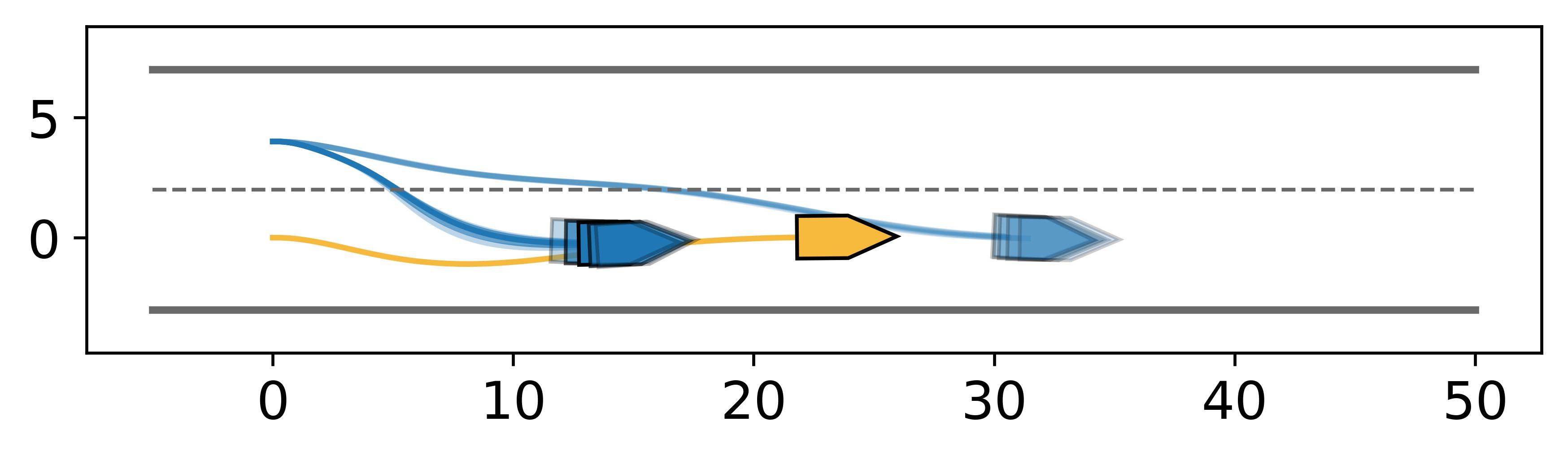}}
\subfigure[$\tau=99$]{\includegraphics[scale=0.38]{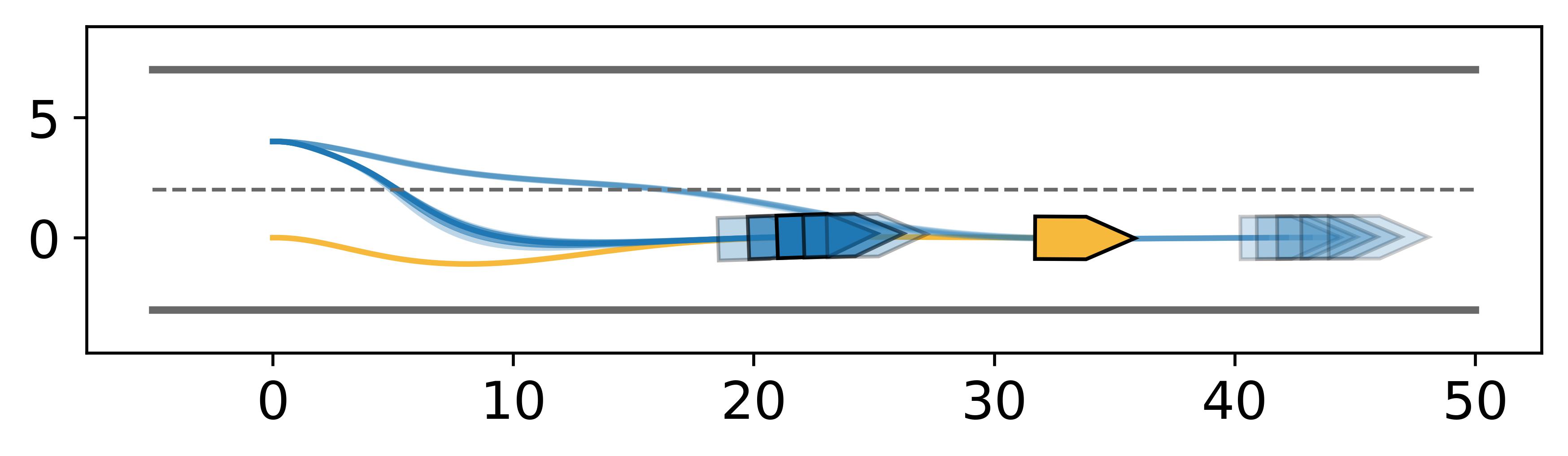}}

\caption{Simulation results in the merging scenario with $[w_1,w_2]=[0.1,0.9]$. The average longitudinal velocity of EA is higher than 3\,m/s, indicating that it speeds up for OA to merge from behind.}

\label{fig:M3}
\end{figure*}

\begin{figure}[t]
\centering
\subfigure[$\tau=50$]{\includegraphics[scale=0.18]{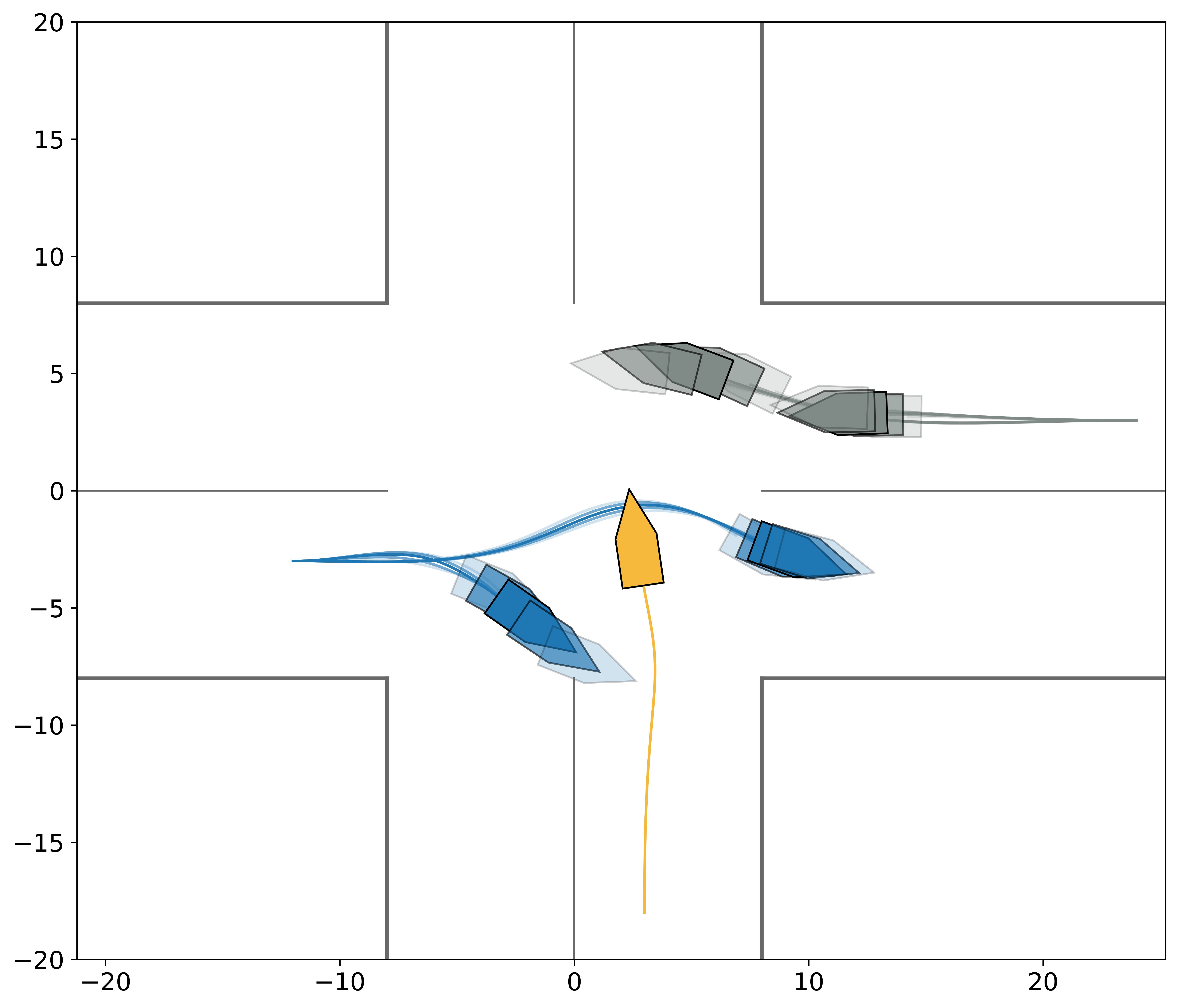}}
\subfigure[$\tau=75$]{\includegraphics[scale=0.18]{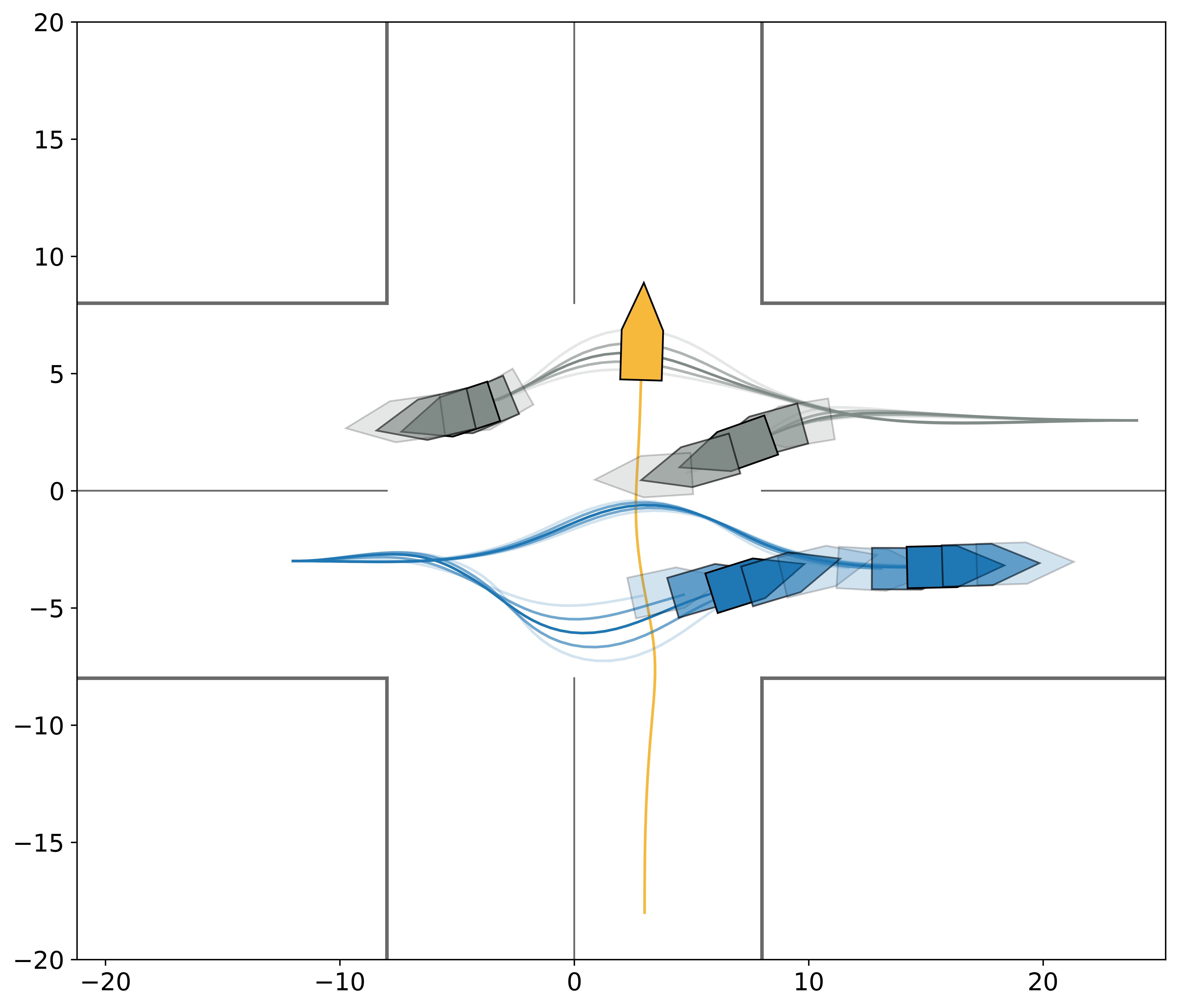}}

\caption{Simulation results in the intersection scenario with $[w^{OA1}_1,w^{OA1}_2,w^{OA2}_1,w^{OA2}_2]=[0.5,0.5,0.5,0.5]$. In this baseline case, equal probabilities are assumed for all behavioral modes of OAs. EA adopts an almost straight trajectory to cross the intersection scenario.}

\label{fig:C1}
\end{figure}

\begin{figure}[t]
\centering
\subfigure[$\tau=50$]{\includegraphics[scale=0.18]{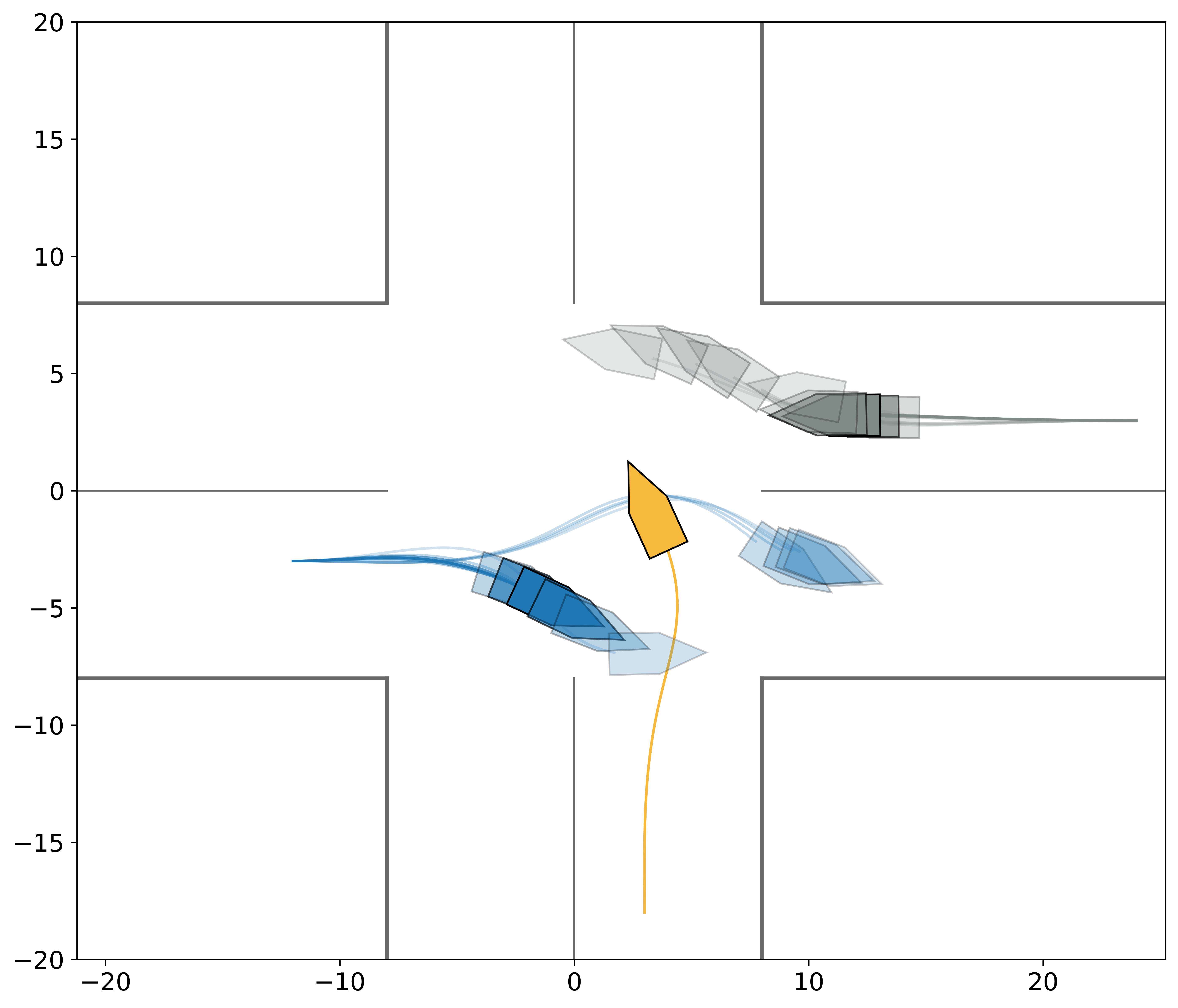}}
\subfigure[$\tau=75$]{\includegraphics[scale=0.18]{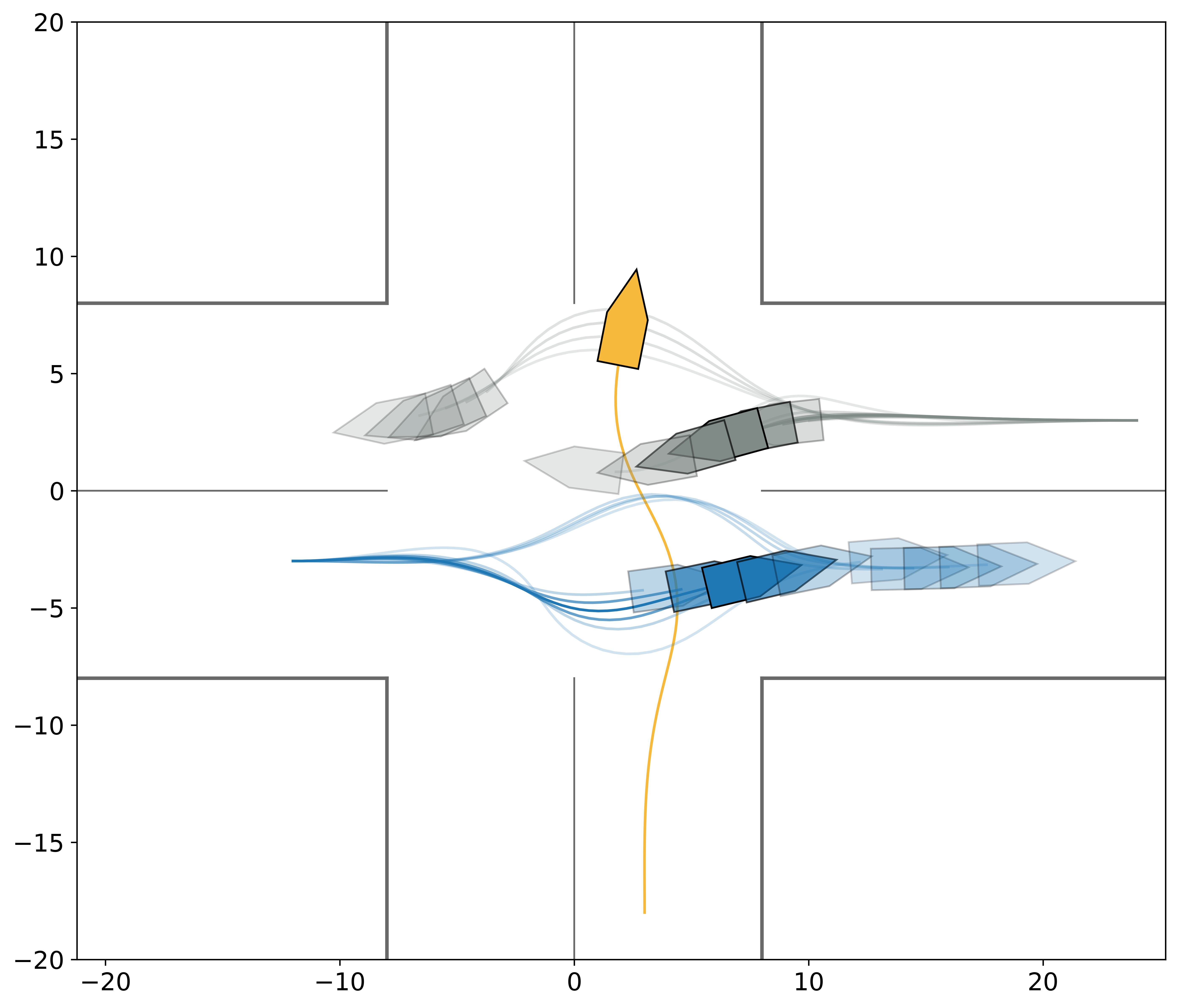}}

\caption{Simulation results in the intersection scenario with $[w^{OA1}_1,w^{OA1}_2,w^{OA2}_1,w^{OA2}_2]=[0.1,0.9,0.1,0.9]$. In this case, both OAs are likely to slow down and yield. EA first swerves to the right to cross in front of OA1 and then swerves to the left to cross in front of OA2.}

\label{fig:C2}
\end{figure}

\begin{figure}[t]
\centering
\subfigure[$\tau=50$]{\includegraphics[scale=0.18]{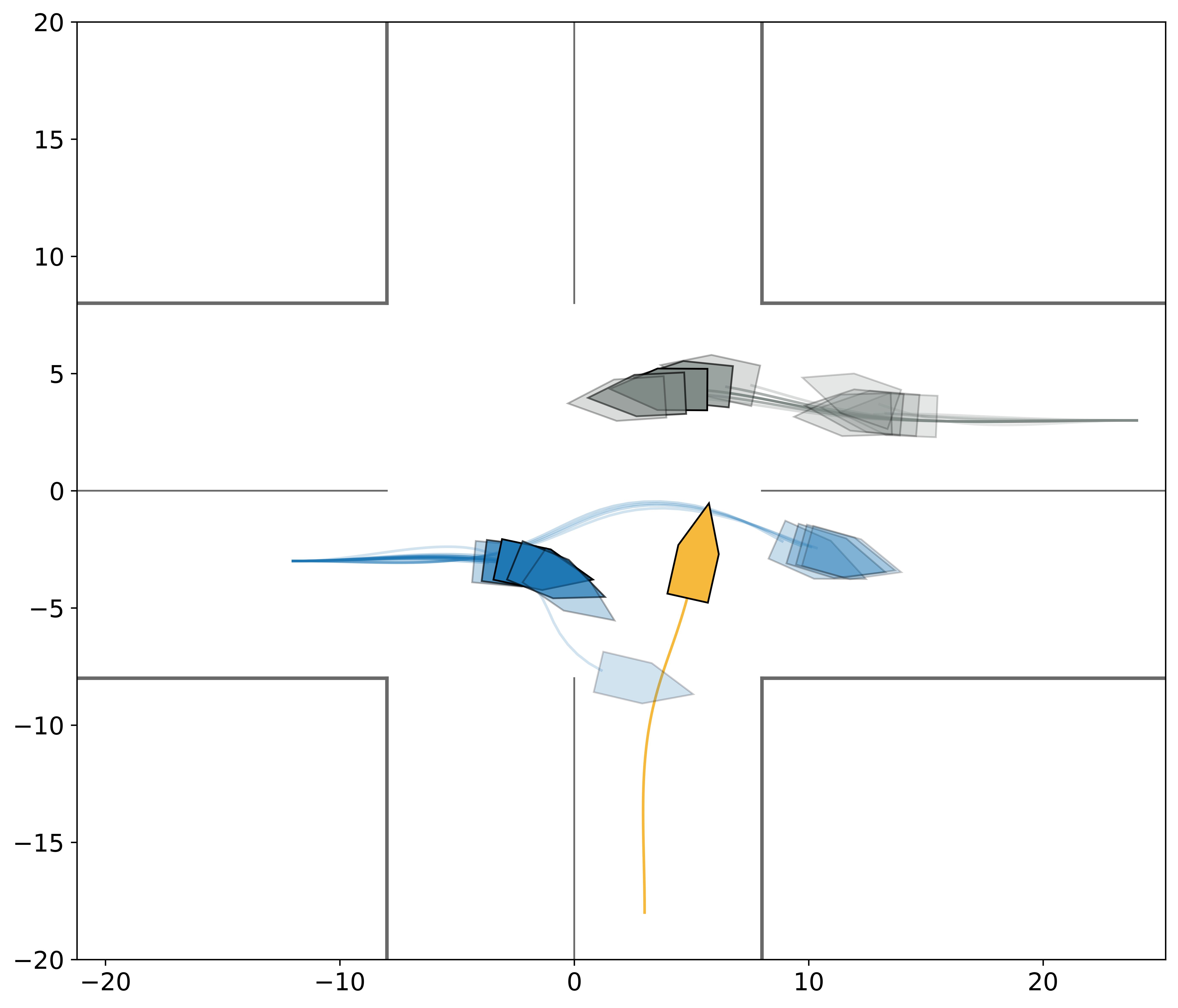}}
\subfigure[$\tau=75$]{\includegraphics[scale=0.18]{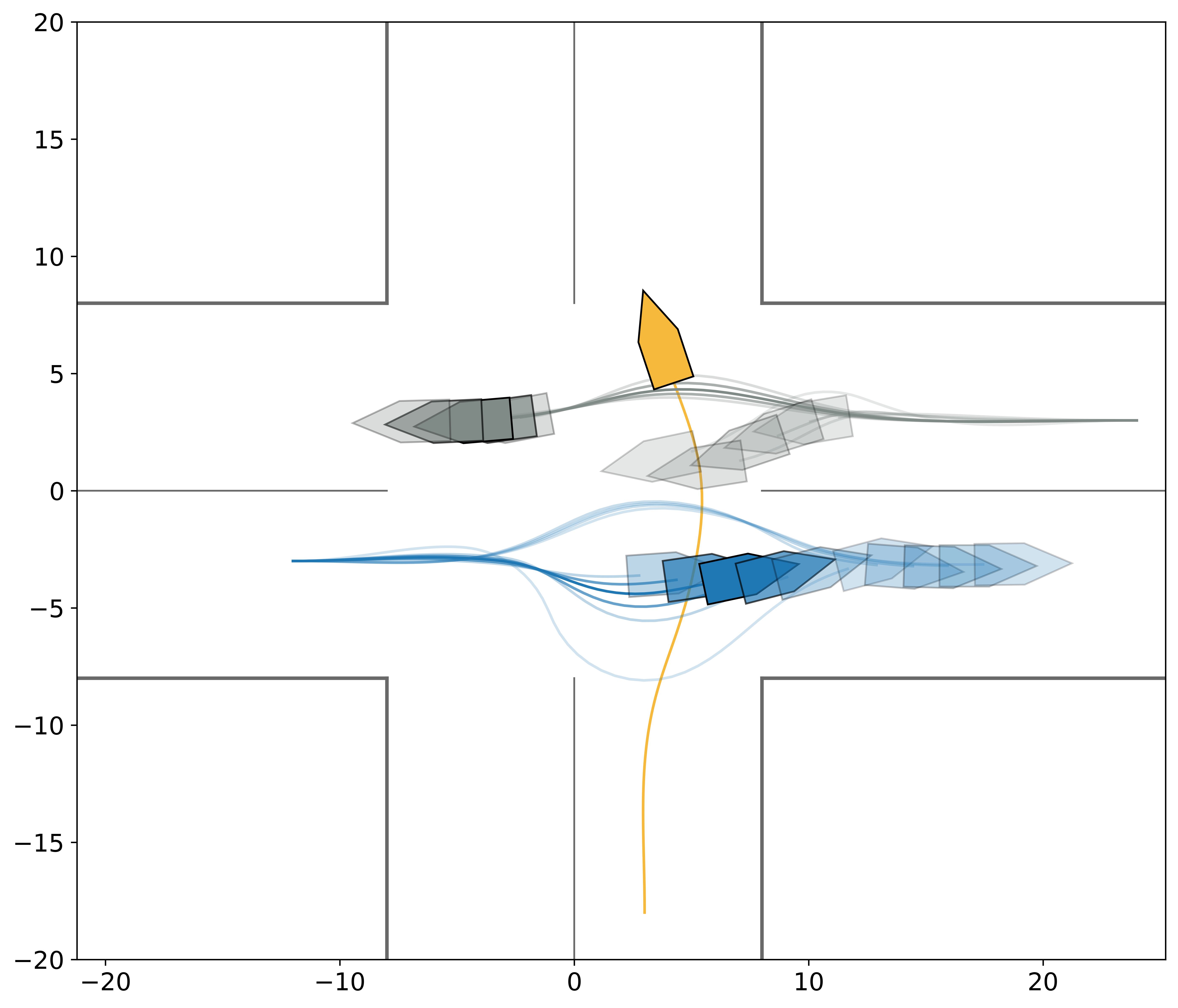}}

\caption{Simulation results in the intersection scenario with $[w^{OA1}_1,w^{OA1}_2,w^{OA2}_1,w^{OA2}_2]=[0.1,0.9,0.9,0.1]$. In this case, OA1 is likely to yield, and OA2 is likely to rush. EA swerves to the right to cross in front of OA1 and from behind OA2.}

\label{fig:C3}
\end{figure}

\begin{figure}[t]
\centering
\subfigure[$\tau=50$]{\includegraphics[scale=0.18]{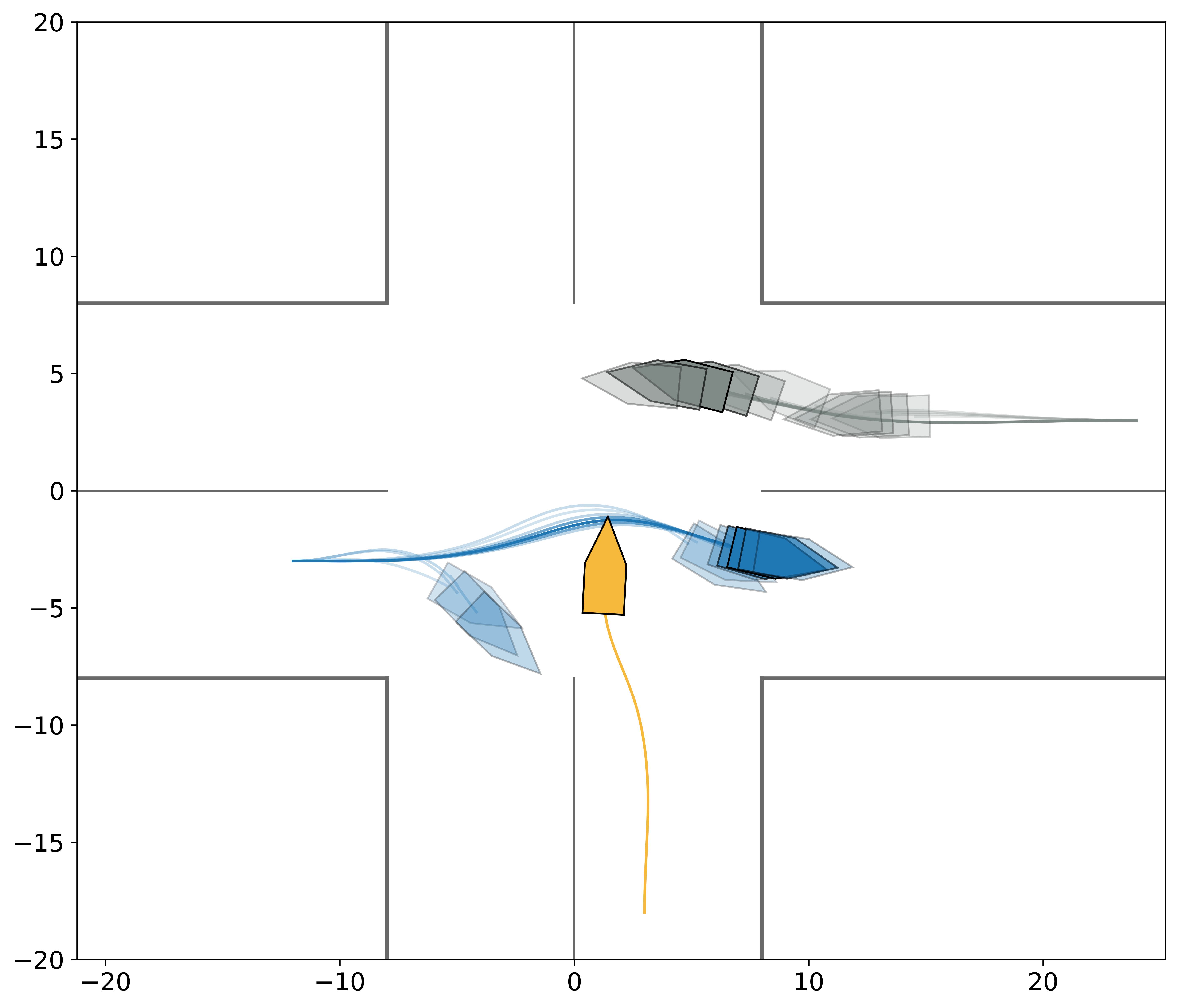}}
\subfigure[$\tau=75$]{\includegraphics[scale=0.18]{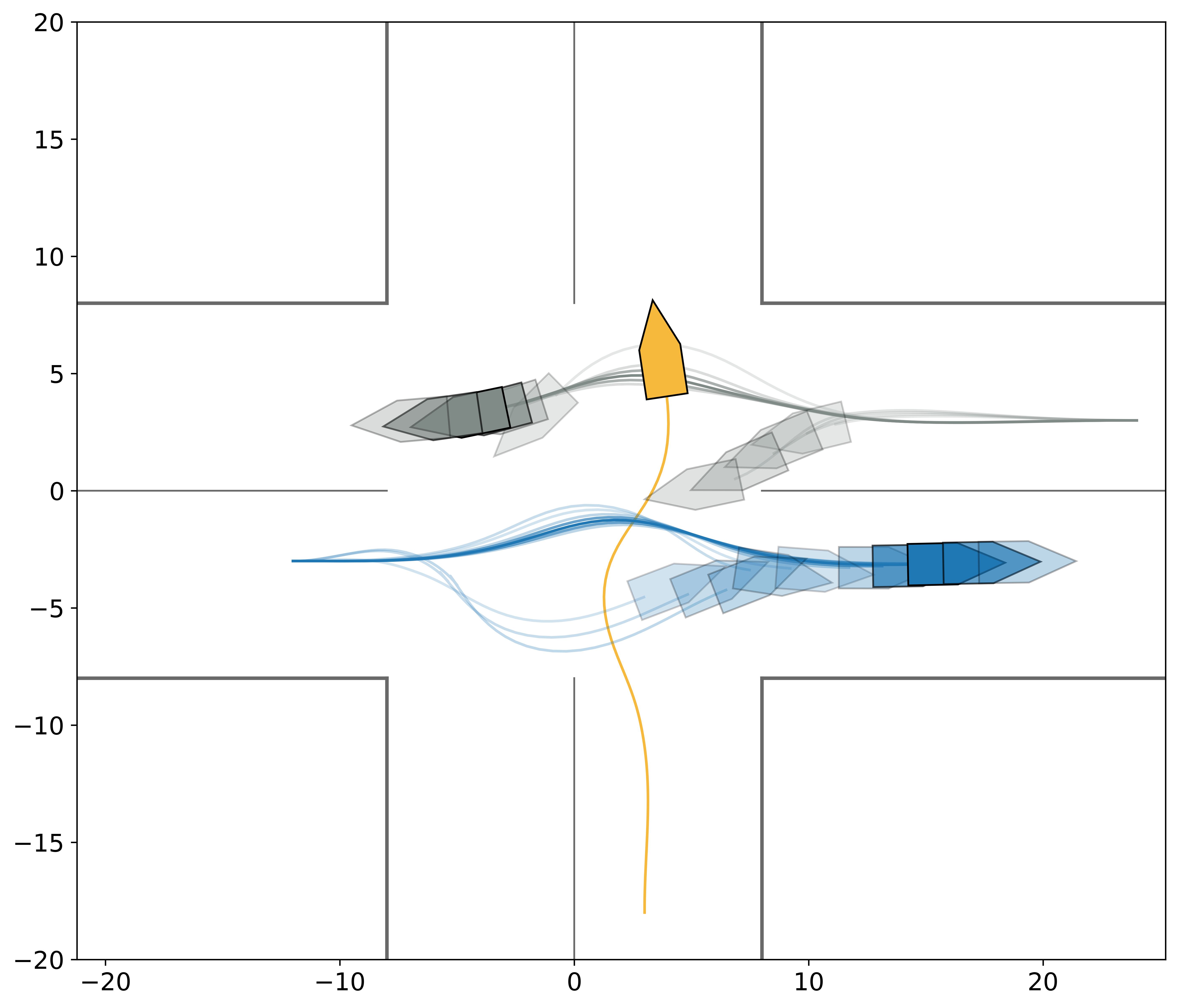}}

\caption{Simulation results in the intersection scenario with $[w^{OA1}_1,w^{OA1}_2,w^{OA2}_1,w^{OA2}_2]=[0.9,0.1,0.9,0.1]$. In this case, both OAs are likely to rush. EA first swerves to the left to cross from behind OA1 and then swerves to the right to cross from behind OA2.}

\label{fig:C4}
\end{figure}

\begin{figure}[t]
\centering
\subfigure[$\tau=50$]{\includegraphics[scale=0.18]{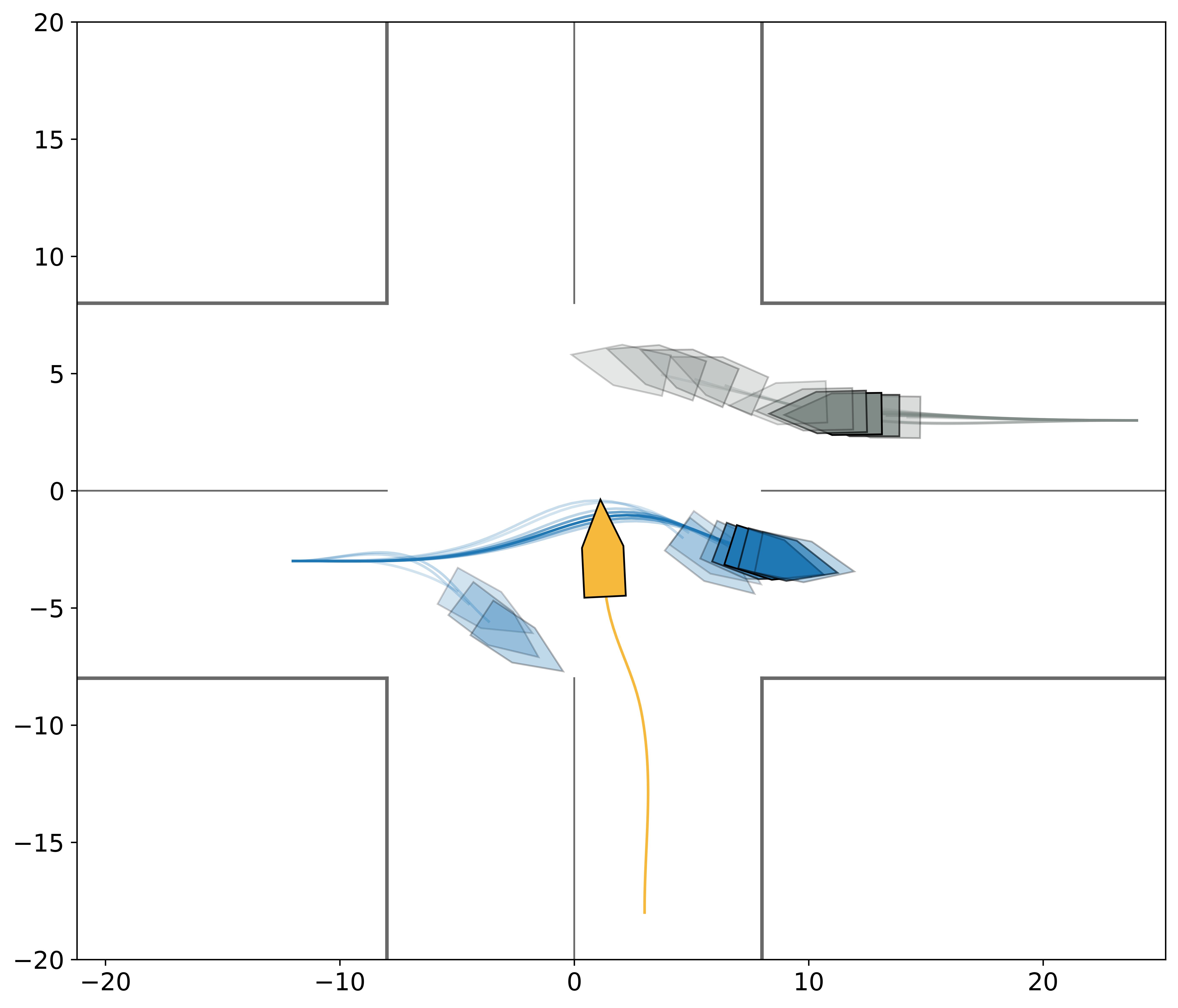}}
\subfigure[$\tau=75$]{\includegraphics[scale=0.18]{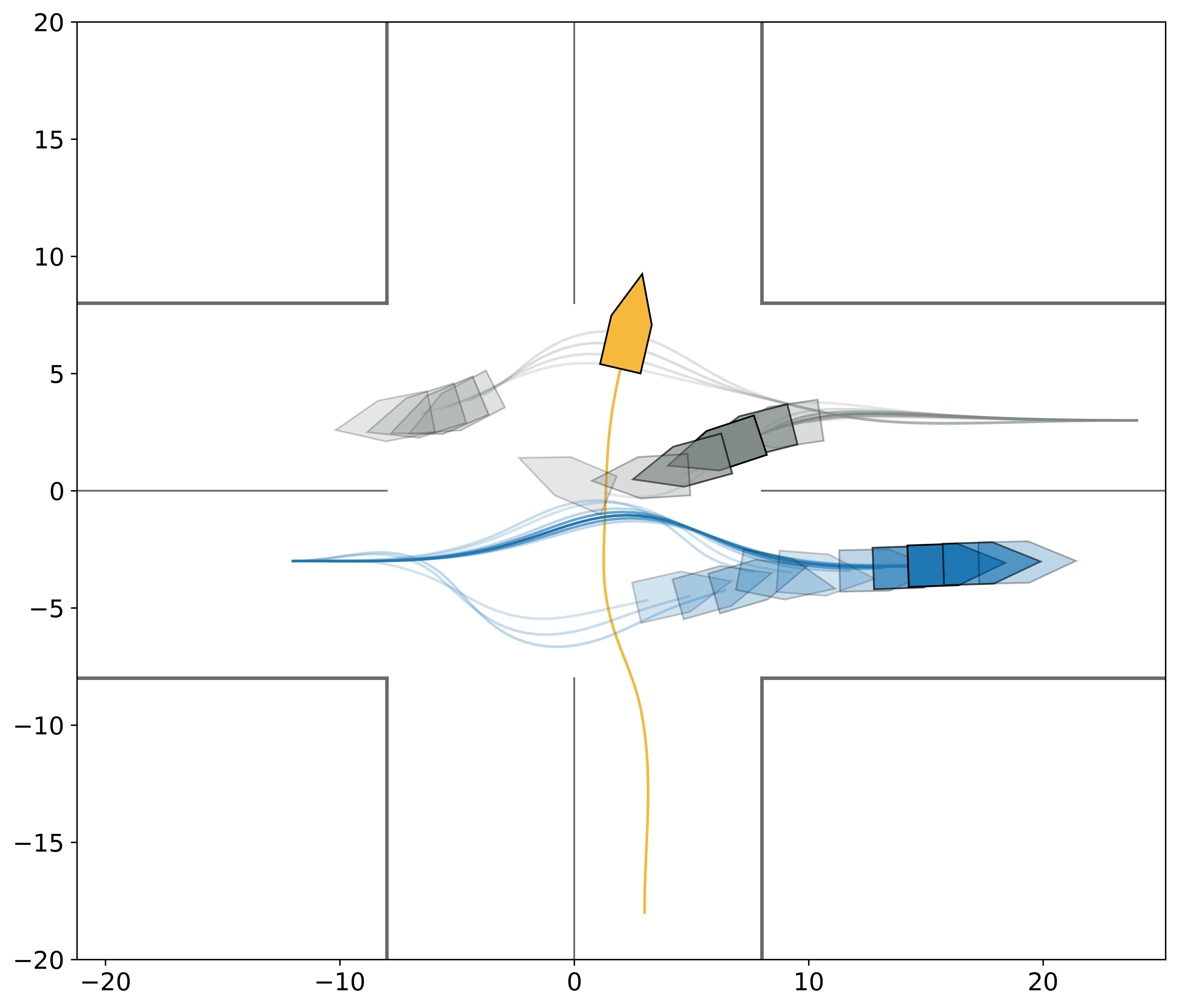}}

\caption{Simulation results in the intersection scenario with $[w^{OA1}_1,w^{OA1}_2,w^{OA2}_1,w^{OA2}_2]=[0.9,0.1,0.1,0.9]$. In this case, OA1 is likely to rush, and OA2 is likely to yield. EA swerves to the left to cross from behind OA1 and in front of OA2.}

\label{fig:C5}
\end{figure}

We first consider a merging scenario with two agents, where the ego agent and the other agent are referred to as EA and OA. The OA intends to merge into the lane kept by the EA. The initial states of EA and OA are $[0,0,0,3]$ and $[0,4,0,3]$, respectively. The target of both agents is to follow the lane with $p_y=0$. The weight matrices for both agents are defined as $Q=\textup{diag}(0,1,0,2)$ and $R=\textup{diag}(10,0.1)$. For collision avoidance, we set $d_\textup{safe}=4.5\,\textup{m}$ and $\beta=1.4$. We set the reference velocity of EA as 3\,m/s, and assume uncertainties over the reference longitudinal velocity of OA, which is subject to the following bimodal Gaussian mixture:
\begin{equation}
\begin{aligned}
    p(v_{ref}) &= w_1\mathcal{N}(v_{ref}|v_{ref,1},\sigma_1^2)+w_2\mathcal{N}(v_{ref}|v_{ref,2},\sigma_2^2),\\
    v_{ref,1}&=3.5\,\textup{m}/\textup{s},\ v_{ref,2}=2.5\,\textup{m}/\textup{s},\sigma_1=\sigma_2=0.2\,\textup{m}/\textup{s}.
\end{aligned}
\end{equation}
$\mathcal{N}$ denotes the Gaussian distribution, $w_1$ and $w_2$ are weight parameters. Note that the proposed method does not preclude other forms of distribution. To obtain the optimal strategy in mathematical expectation under this uncertainty, we sample 10 samples from the Gaussian mixture as $\{v_{ref,i}-2\sigma_i,v_{ref,i}-\sigma_i,v_{ref,i},v_{ref,i}+\sigma_i,v_{ref,i}+2\sigma_i\}$ with $i\in\{1,2\}$ to generate the types of OA. By varying combinations of $w_1$ and $w_2$, different BNEs are obtained. The simulation results are shown in Figs. \ref{fig:M1}, \ref{fig:M2}, and \ref{fig:M3}, where EA is yellow, and transparency is used to denote the probability of the type players of OA. The baseline results are shown in Fig. \ref{fig:M1}, where EA holds equal beliefs over two behavioral modes of OA. In this case, the EA manages to track the preassigned reference velocity. Meanwhile, when EA has a higher belief in the higher intended velocities of OA and therefore it is more likely to merge in the front, the best strategy is to slow down to yield, resulting in slightly slower longitudinal velocities of EA compared to baseline results (see Fig. \ref{fig:M2}). In contrast, if EA believes that OA is more likely to slow down, the best strategy is to accelerate to allow OA to merge from behind, resulting in higher longitudinal velocities (see Fig. \ref{fig:M3}). These simulation results indicate that our method manages to obtain different optimal strategies conditioned on different beliefs.

We also examine an intersection scenario with three agents crossing simultaneously, where the ego agent is referred to as EA and two other agents are referred to as OA1 and OA2. The initial states of EA, OA1, and OA2 are $[3,-18,1.57,3],[-12,-3,0,3],[24,3,3.14,3]$, respectively. EA intends to cross the intersection from the bottom to the top. OA1 intends to cross the intersection from left to right, and OA2 intends to cross the intersection from right to left. Weight matrices are given as $Q=\textup{diag}(1,1,0,0)$ and $R=\textup{diag}(10,0.1)$. For collision avoidance, we set $d_\textup{safe}=6.5\,\textup{m}$ and $\beta=1.4$. The reference velocity of EA is given as 3\,m/s, and we assume uncertainties over the reference longitudinal velocities of both OA1 and OA2 with the following bimodal Gaussian mixture:

\begin{equation}
\begin{aligned}
    p(v^{OA1}_{ref}) &= w^{OA1}_1\mathcal{N}(v^{OA1}_{ref}|v^{OA1}_{ref,1},\sigma_{OA1,1}^2)\\
    &+w^{OA1}_2\mathcal{N}(v^{OA1}_{ref}|v^{OA1}_{ref,2},\sigma_{OA1,2}^2),\\
        p(v^{OA2}_{ref}) &= w^{OA2}_1\mathcal{N}(v^{OA2}_{ref}|v^{OA2}_{ref,1},\sigma_{OA2,1}^2)\\
        &+w^{OA2}_2\mathcal{N}(v^{OA2}_{ref}|v^{OA2}_{ref,2},\sigma_{OA2,2}^2),\\
    v^{OA1}_{ref,1}&=v^{OA2}_{ref,1}=3.6\,\textup{m}/\textup{s},\\ v^{OA1}_{ref,2}&=v^{OA2}_{ref,2}=2.4\,\textup{m}/\textup{s},\\
    \sigma_{OA1,1}&=\sigma_{OA1,2}=\sigma_{OA2,1}=\sigma_{OA2,2}=0.2\,\textup{m}/\textup{s}.
\end{aligned}
\end{equation}
We also assume that EA admits a unique type and sample 10 samples from each Gaussian mixture to obtain the types for each OA. For the EA and each type of OAs, we generate the reference trajectory by assuming uniform motion with the corresponding reference velocity. Simulation results are shown in Figs. \ref{fig:C1}-\ref{fig:C5}. In particular, baseline results are shown in Fig. \ref{fig:C1}, such that equal probabilities are assumed over all behavioral modes of OAs. Compared with the baseline results, results in Figs. \ref{fig:C2}-\ref{fig:C5} show that EA either decides to rush and cross in front of the OA, or it decides to yield and cross from behind, conditioned on what it believes is the most likely behavioral mode of the OA, resulting in different optimal trajectories. These results verify the effectiveness of the proposed method in the intersection scenario by showing different sets of BNE trajectories conditioned on beliefs of the driving intentions of the OAs.

Meanwhile, to demonstrate the computational efficiency and solution accuracy of the proposed method, we perform quantitative comparisons between the proposed method and the following baseline methods.

\textit{Baseline 1. ILQGames}: The iLQGames~\cite{fridovich2020efficient} is a general-purpose game solver for various types of game settings. We utilize this game solver to obtain the NE corresponding to the agent-form game featured by Problem 1.

\textit{Baseline 2. IPOPT}: For this baseline, we use an interior point optimizer (IPOPT) to obtain the solution for Problem 3. The IPOPT solver is implemented by CasADi, which is an open-source software tool for numerical optimization in general and optimal control.

\textit{Baseline 3. P-iLQR}: For this baseline, we solve Problem 3 by potential iLQR~\cite{kavuncu2021potential}, which is a centralized solving method that utilizes the iLQR algorithm to obtain multi-agent trajectories.

\textit{Baseline 4. D-iLQR}: We solve Problem 3 by the decentralized iLQR~\cite{huang2023decentralized} algorithm, which is developed directly based on the method proposed in~\cite{banjac2019decentralized,grontas2022distributed}.

\begin{figure}[t]
\centering
\subfigure[Intersection]{\includegraphics[scale=0.45]{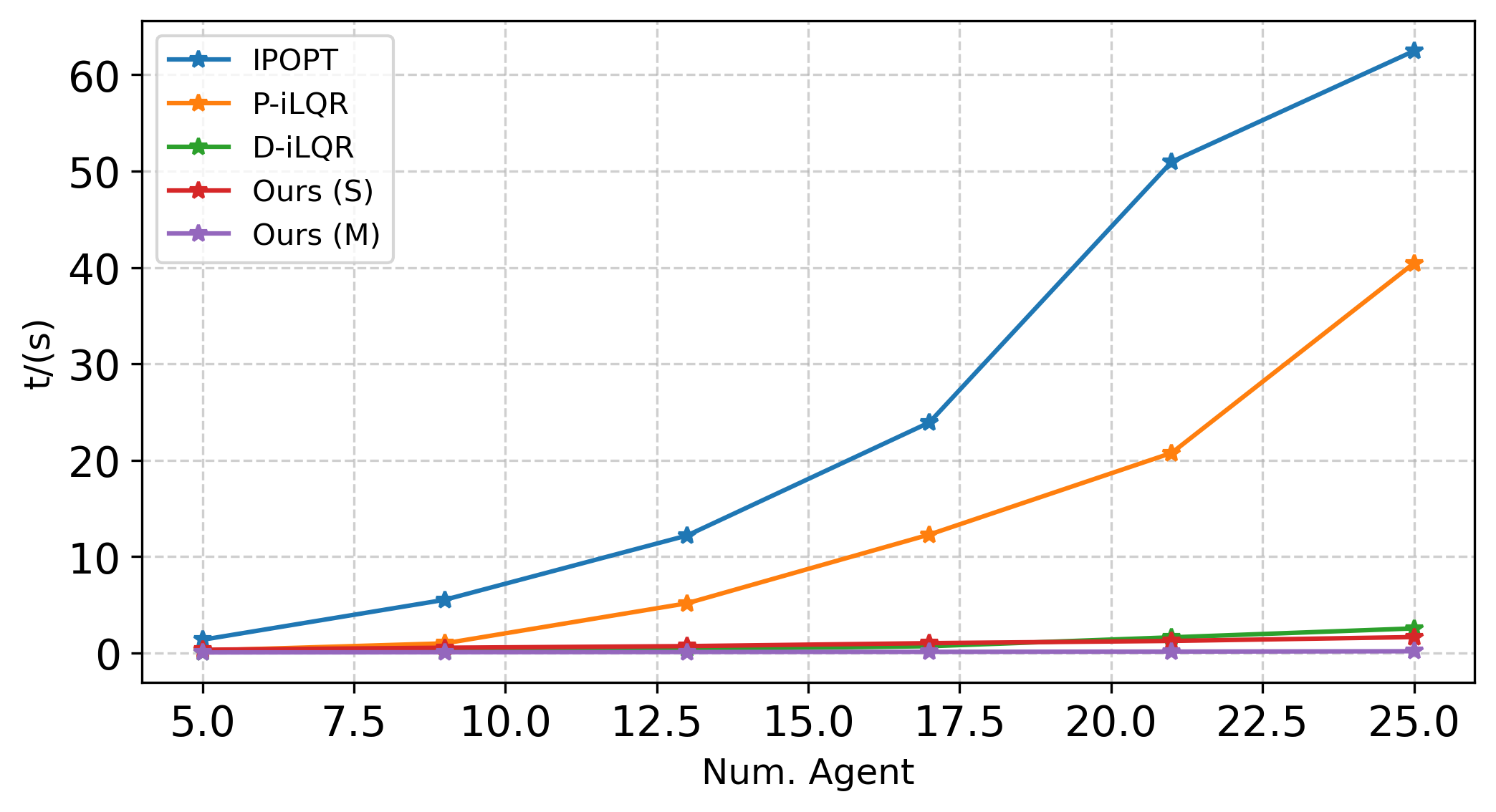}}
\subfigure[Merging]{\includegraphics[scale=0.45]{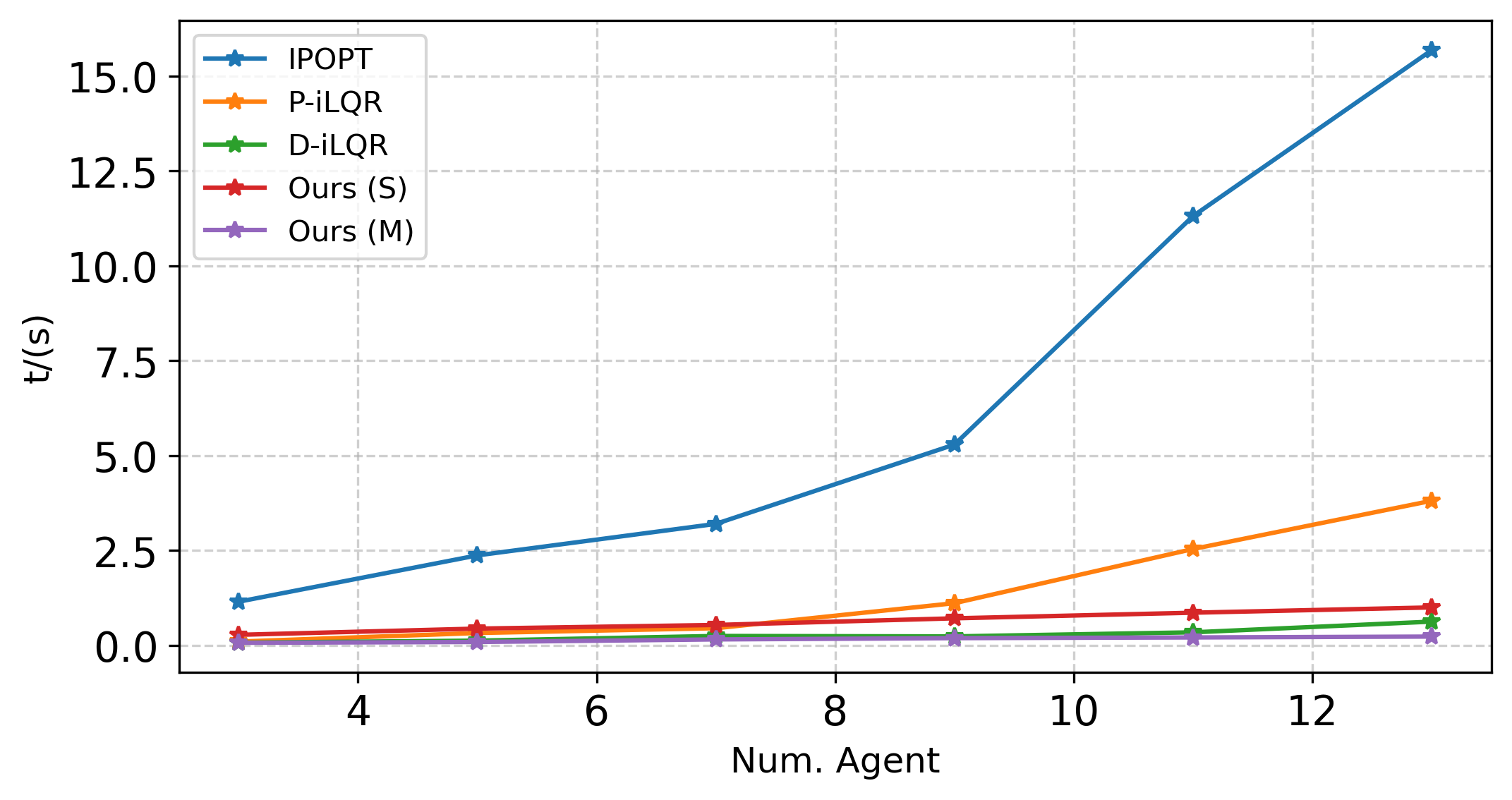}}

\caption{Computation time with respect to the number of type-players for the baselines and the proposed method. It can be shown from the curves that the proposed method with multi-process implementation has much better scalability than other methods in both scenarios.}

\label{fig:solutionComparison}
\end{figure}

For the proposed method, we provide both a single-process implementation \textit{Ours\,(S)} and a multi-process implementation \textit{Ours\,(M)} for comparison. In the multi-process implementation, the number of processes is set to be the number of type-players. All methods are
running on a server with 2\,× Intel(R) Xeon(R) Gold 6348
CPU @ 2.60GHz. The quantitative results are shown in Table \ref{Tab:time}, with the fastest solving time marked in \textbf{bold} and the best solution results marked with \underline{underline}. Note that in the table, T. and C. denote the computation time and the total cost at termination, respectively. While iLQGames is the only method that solves the original agent-form game instead of the corresponding potential game, results show that it takes significantly longer compared to all other methods. In the limited cases that it manages to finish within meaningful computation time, the costs at termination show that the quality of the planning results obtained is poor. From these outcomes, we conclude that solving the Bayesian game introduced in this paper is extremely hard without the merit of potential game formulation, even with a state-of-the-art general-purpose game solver like iLQGames. This finding supports our claim that the potential formulation of the Bayesian game greatly enhances the accessibility of the corresponding BNE. Meanwhile, among all methods that adopt the potential game formulation, the multi-process implementation of the proposed method achieves the lowest computational time in all cases, which clearly demonstrates its superiority in computational efficiency. Termination costs show that the proposed method achieves optimal/near-optimal solutions in all cases. We also plot the curves of solution time with respect to the number of agents to show the scalability. Quantitatively, in the intersection scenario, when the number of agent types increases by 5$\times$, the computation time increases by 45.6$\times$, 163.8$\times$, 28.0$\times$, 4.9$\times$, and 3.3$\times$ for IPOPT, P-iLQR, D-iLQR, Ours\,(S), and Ours\,(M), respectively. In the merging scenario, when the number of agent types increases by 4.3$\times$, the computation time increases by 13.6$\times$, 38.5$\times$, 8.8$\times$, 3.6$\times$, and 3.4$\times$, for IPOPT, P-iLQR, D-iLQR, Ours\,(S), and Ours\,(M), respectively. These qualitative and quantitative results show the superiority of the proposed parallel solving scheme in scalability compared with all the centralized/decentralized baselines in the simulation.

\begin{table*}[]
\centering
\caption{Comparison of computation time and solution accuracy between iLQGames, IPOPT, P-iLQR, D-iLQR, and the proposed method. }
\renewcommand\arraystretch{1.5} 
\begin{tabular}{cc|cccccc|cccccc}
\toprule[1.5pt]
\multicolumn{2}{c|}{Scenarios}     & \multicolumn{6}{c|}{Intersection}                    & \multicolumn{6}{c}{Merging}                           \\ \toprule[1.0pt]
\multicolumn{2}{c|}{Num. Types}   & 5       & 9      & 13     & 17     & 21     & 25     & 3       & 5       & 7      & 9      & 11     & 13     \\ \toprule[1.5pt]
\multirow{2}{*}{iLQGames} & T. (s) & 253.870 & $>$600 & $>$600 & $>$600 & $>$600 & $>$600 & 137.240 & 517.269 & $>$600 & $>$600 & $>$600 & $>$600 \\
                          & C.     & 35017.1 & ---    & ---    & ---    & ---    & ---    & 2551.46 & 1016.87 & ---    & ---    & ---    & ---    \\ \toprule[1pt]
\multirow{2}{*}{IPOPT}    & T. (s) & 1.370   & 5.520  & 12.190 & 23.920 & 50.970 & 62.470 & 1.150   & 2.370   & 3.200  & 5.290  & 11.310 & 15.680 \\
                          & C.     & \underline{918.27}  & \underline{918.38} & \underline{927.78} & \underline{929.85} & 938.79 & \underline{935.77} & \underline{645.87}  & \underline{650.21}  & \underline{648.22} & 670.64 & 669.7  & 673.68 \\ \toprule[1pt]
\multirow{2}{*}{P-iLQR}   & T. (s) & 0.247   & 0.999  & 5.157  & 12.288 & 20.758 & 40.450 & 0.099   & 0.325   & 0.451  & 1.111  & 2.538  & 3.809  \\
                          & C.     & 919.52  & 920.66 & 981.57 & 983.22 & 988.13 & 997.63 & 646.99  & 781.25  & 801.58 & 764.33 & 772.62 & 769.74 \\ \toprule[1pt]
\multirow{2}{*}{D-iLQR}   & T. (s) & 0.091   & 0.155  & 0.405  & 0.699  & 1.624  & 2.550  & 0.071   & 0.124   & 0.247  & 0.236  & 0.343  & 0.625  \\
                          & C.     & 921.58  & 928.55 & 943.3  & 951.11 & 960.27 & 967.76 & 649.69  & 659.83  & 728.68 & 832.22 & 831.68 & 835.84 \\ \toprule[1pt]
\multirow{2}{*}{Ours\,(S)} & T. (s) & 0.336   & 0.545  & 0.712  & 1.020  & 1.250  & 1.650  & 0.277   & 0.441   & 0.537  & 0.712  & 0.860  & 0.999  \\
                          & C.     & 919.93  & 920.01 & 929.99 & 932.63 & \underline{937.26} & 939.64 & 646.94  & 650.9   & 649.12 & \underline{650.75} & \underline{651.04} & \underline{655.86} \\ \toprule[1pt]
\multirow{2}{*}{Ours\,(M)} & T. (s) & \textbf{0.055}   & \textbf{0.077}  & \textbf{0.102}  & \textbf{0.119}  & \textbf{0.143}  & \textbf{0.182}  & \textbf{0.067}   & \textbf{0.088}   & \textbf{0.156}  & \textbf{0.190}  & \textbf{0.208}  & \textbf{0.231}  \\
                          & C.     & 919.93  & 920.01 & 929.99 & 932.63 & \underline{937.26} & 939.64 & 646.94  & 650.9   & 649.12 & \underline{650.75} & \underline{651.04} & \underline{655.86} \\ \toprule[1.5pt]
\end{tabular}
\label{Tab:time}
\end{table*}
 
\begin{figure}[t]
\centering
\subfigure[Ours, $p(\textup{up})=0.9$.]{\includegraphics[scale=0.33]{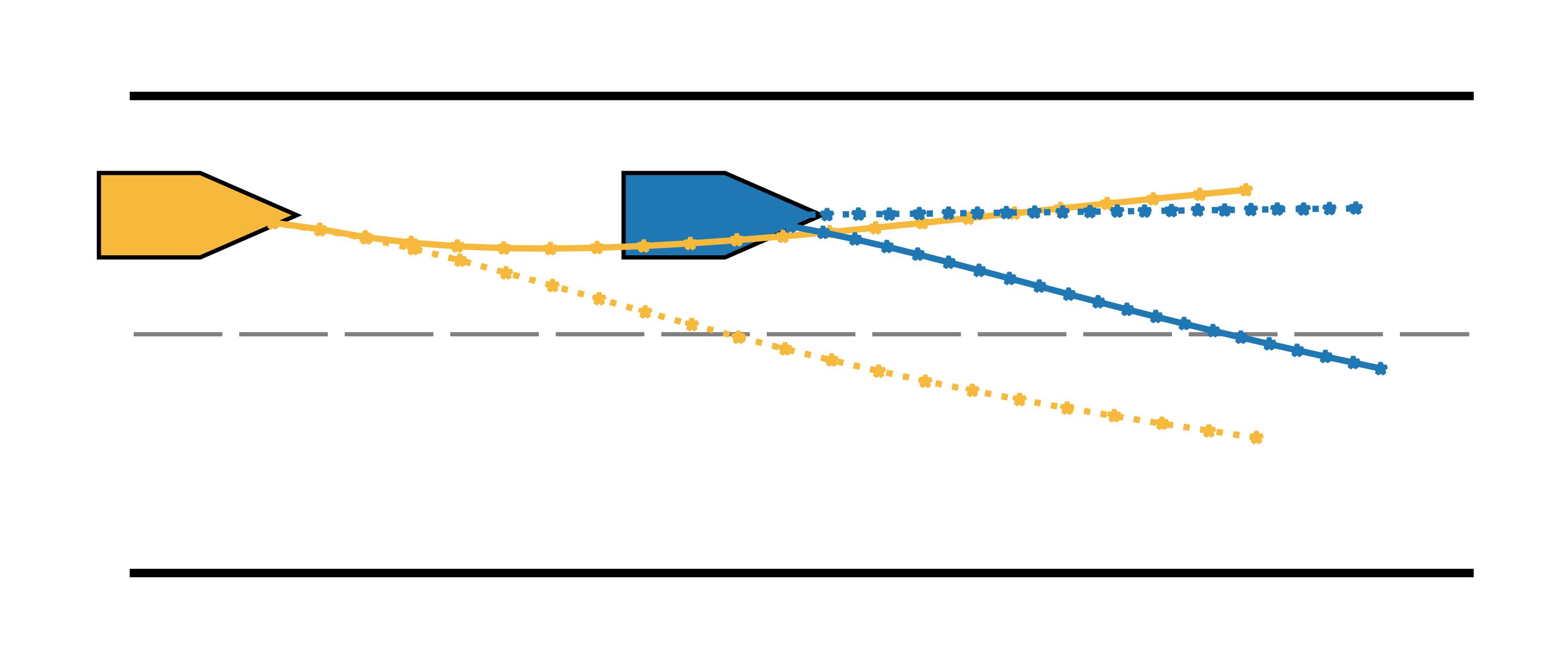}}
\subfigure[Ours, $p(\textup{up})=0.5$.]{\includegraphics[scale=0.33]{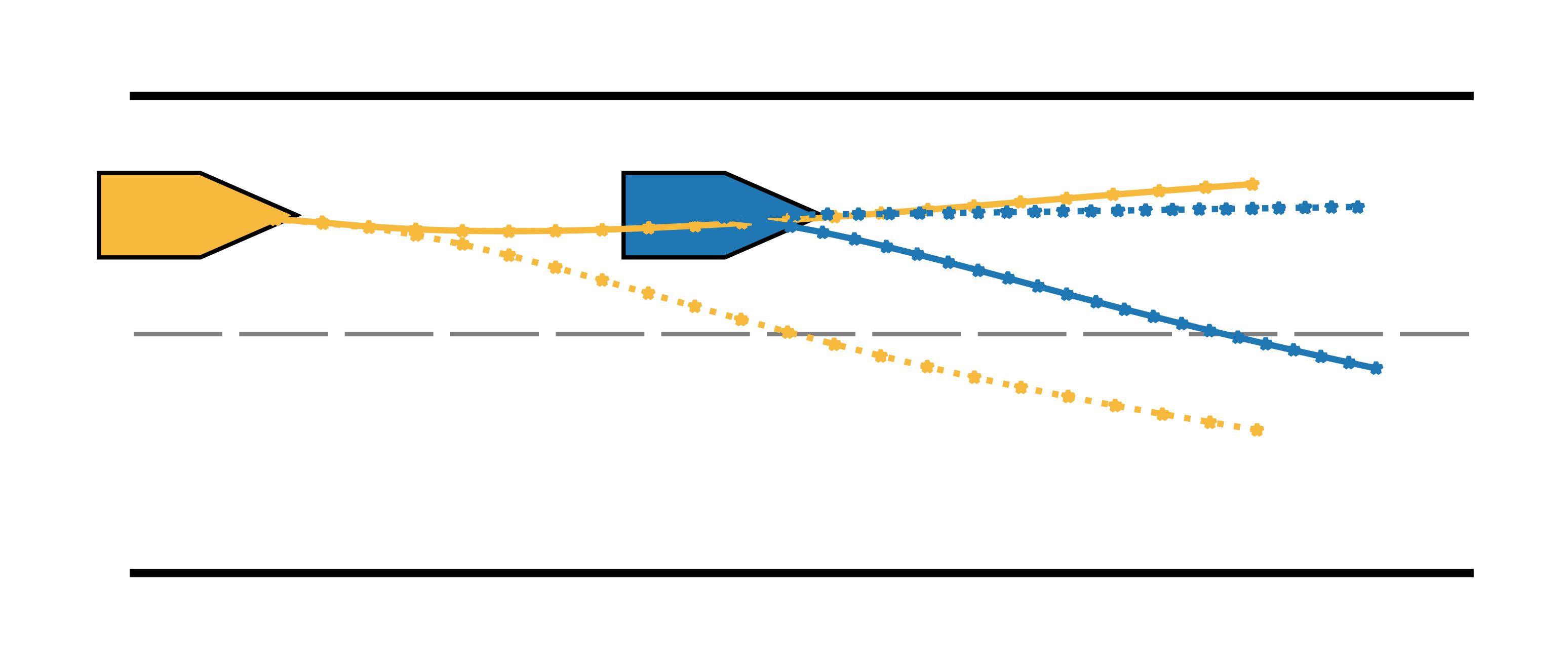}}
\subfigure[Ours, $p(\textup{up})=0.1$.]{\includegraphics[scale=0.33]{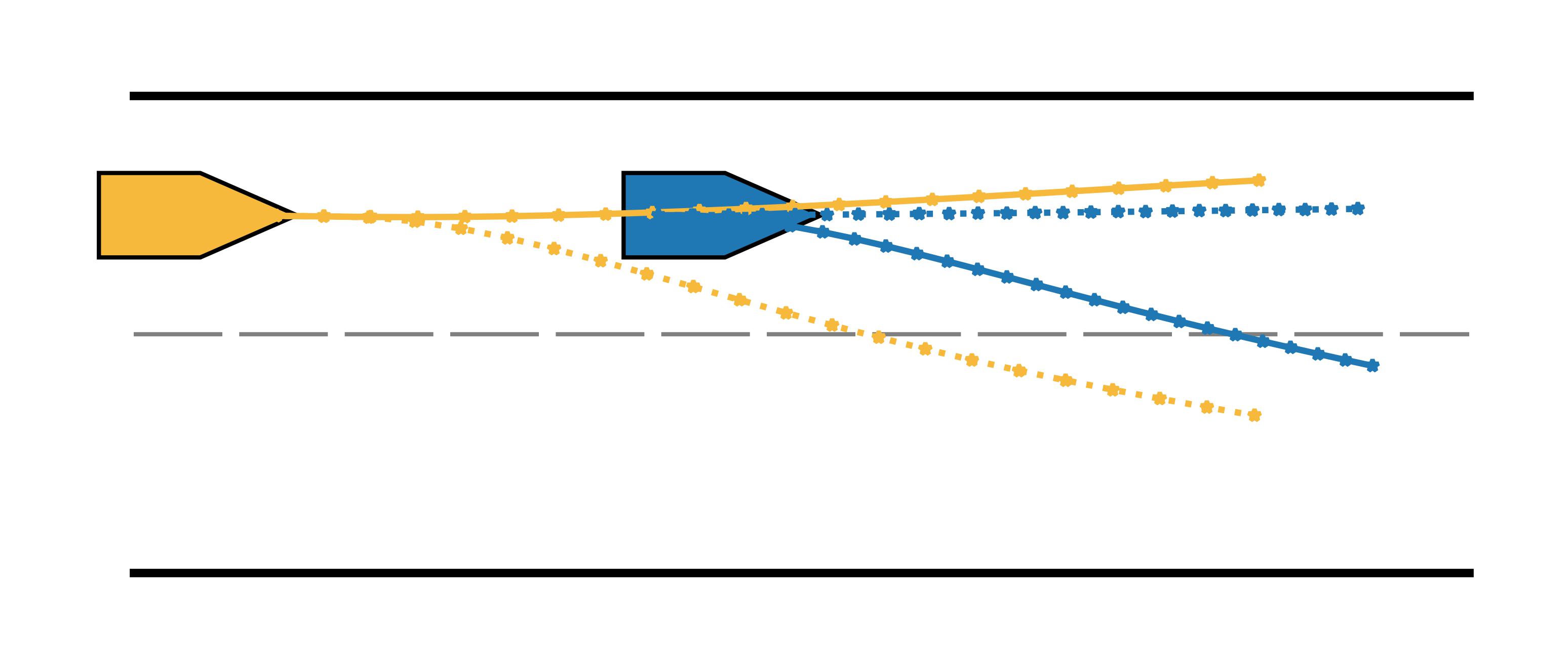}}
\subfigure[Baseline, $p(\textup{up})=0.9$.]{\includegraphics[scale=0.33]{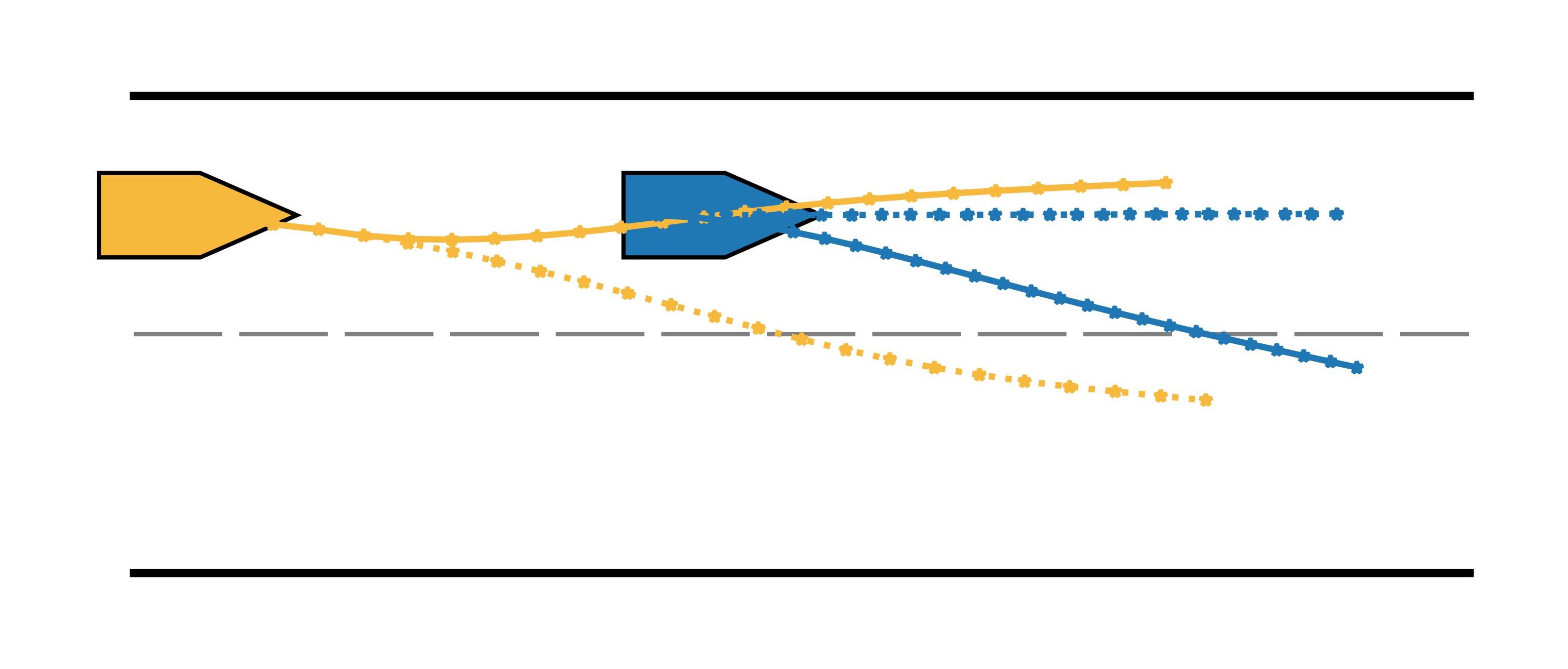}}
\subfigure[Baseline, $p(\textup{up})=0.5$.]{\includegraphics[scale=0.33]{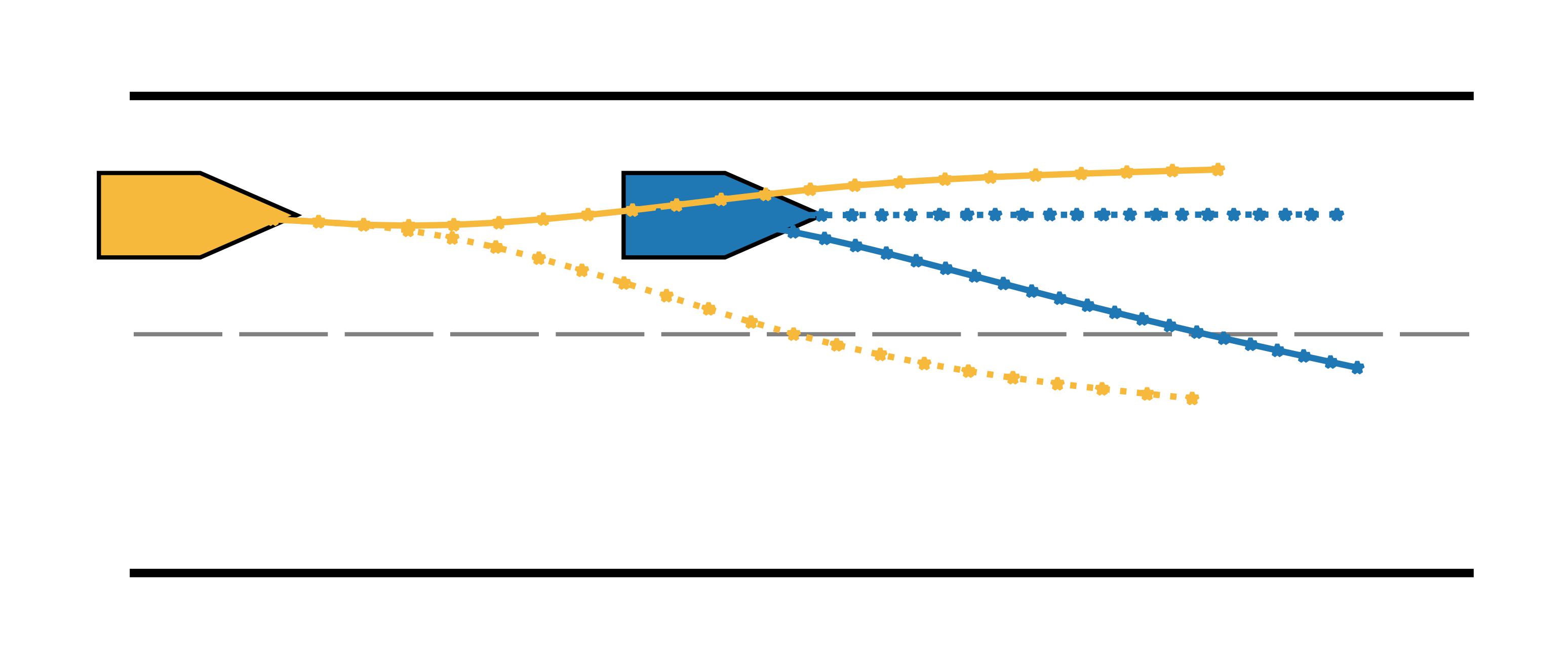}}
\subfigure[Baseline, $p(\textup{up})=0.1$.]{\includegraphics[scale=0.33]{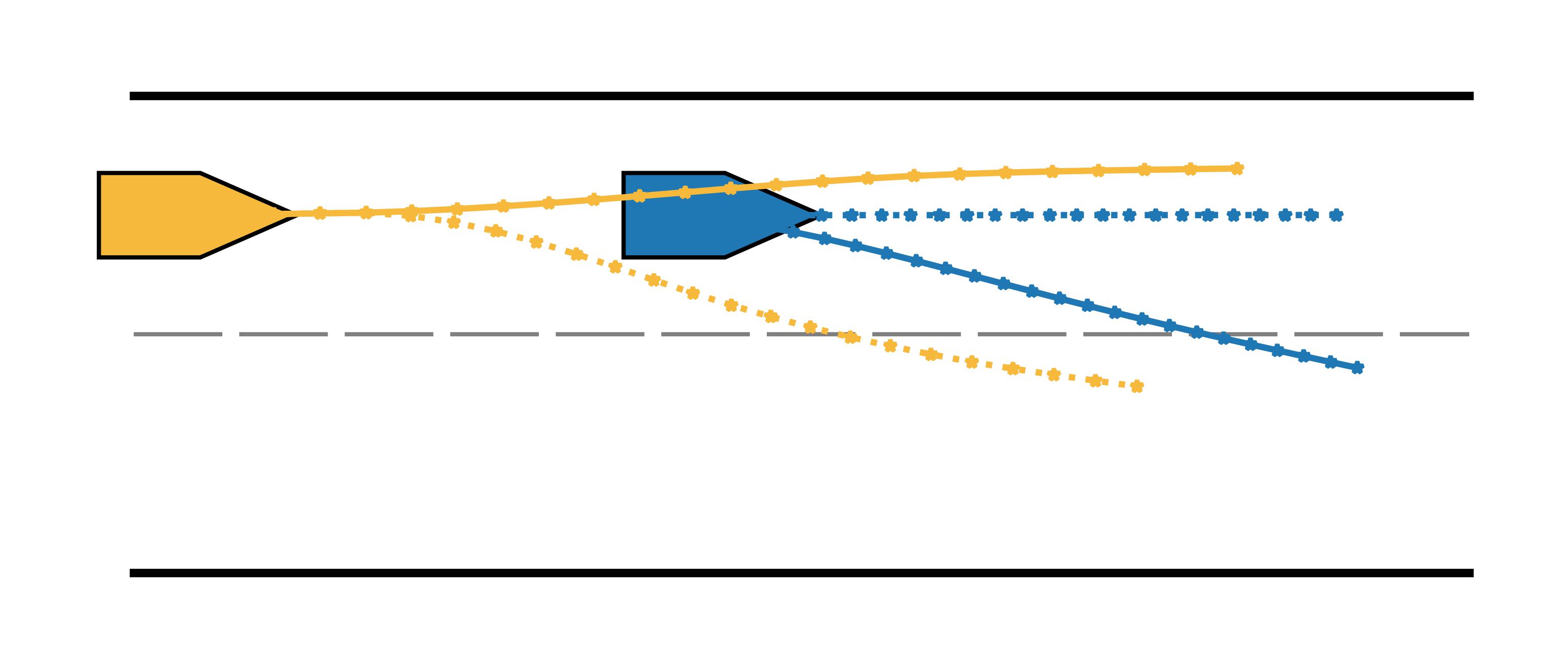}}

\caption{Simulation results of the contingency game by the proposed method and the baseline method under different beliefs. OA is in blue, and trajectories of the same linestyle correspond to the same hypothesis. $p(\textup{up})$ denotes the probability that the OA aims to follow the upper lane $p(p^{OA}_{y,ref}=0.5\,\textup{m})$. The results of the proposed method are similar to those of the baseline. When $p(\textup{up})$ is high, namely that the OA is likely to follow the upper lane, the contingency plan is biased towards a lane changing decision so as to avoid the OA and perform the overtaking. As $p(\textup{up})$ reduces, the contingency plan of the EA gradually shifts to favor lane keeping. These qualitative results demonstrate the effectiveness of the proposed method in interactive contingency planning.}

\label{fig:contingency}
\end{figure}


In addition to the potential Bayesian game in Problem 3, we also examine the potential contingency game in Problem 4. We consider a two-agent overtaking scenario, in which the ego agent (referred to as EA) is trying to overtake the slower agent (referred to as OA) in front. In particular, the EA is uncertain about the target lane and target velocity of the OA, and therefore, the EA maintains a belief over the hypothesis space $\Theta=\{(p^{OA}_{y,ref},v^{OA}_{ref})|p^{OA}_{y,ref}\in \mathcal{R}^{OA}_{y,ref},v^{OA}_{ref}\in\mathcal{R}^{OA}_{v,ref}\}$, where $\mathcal{R}^{OA}_{y,ref}$ and $\mathcal{R}^{OA}_{v,ref}$ are the space of possible target lanes and possible target velocities of OA, respectively. The initial states of EA and OA are given as $[-4.0,0.5,0,1.0]$ and $[-2.9,0.5,0,0.75]$.

We first perform a qualitative analysis to verify the effectiveness of the proposed method. We let $\mathcal{R}^{OA}_{y,ref}=\{0,0.5\}\,\textup{m}$ and $\mathcal{R}^{OA}_{v,ref}=\{0.5\}\,\textup{m}/\textup{s}$. For each hypothesis, we set the reference lane of the EA to be different from that of the OA, namely ${p}^{EA}_{y,ref}=0\,\textup{m}$ if ${p}^{OA}_{y,ref}=0.5\,\textup{m}$ and vice versa, so as to perform overtaking. The reference velocity of the EA is set to be $1.0\,\textup{m}/\textup{s}$ for both hypotheses. We use the prediction horizon $T=25$ and the discretization $\tau_s=0.1\,\textup{s}$. Weight matrices are set to be $Q=\textup{diag}(0,0.5,0.25,1.0)$, $R=\textup{diag}(0.5,1.0)$, and $Q_\textup{contingency}=\textup{diag}(50,50,100,10)$. For collision avoidance, we let $d_\textup{safe}=0.5\,\textup{m}$ and $\beta=1.4$. The branching time is set to be $t_b=5$. Results are shown in Fig. \ref{fig:contingency}, where the contingency plans obtained by the proposed method and the baseline method under different beliefs are demonstrated. Similar to the baseline method, the contingency plan of the proposed method is also a function of the OA's intent probability. To facilitate overtaking, it is biased away from the lane that the OA is most likely to follow. Meanwhile, a contingency trajectory is maintained in case the anticipation of the OA's intention is wrong. These qualitative results substantiate the correctness of the potential formulation of the contingency game.

To demonstrate the efficiency and scalability of the proposed potential contingency game, we compare the solution time of the proposed method with that of the original contingency game~\cite{peters2024contingency} with an increasing number of hypotheses. To generate up to 10 hypotheses, we consider the space of possible target velocities of OA as $\mathcal{R}^{OA}_{v,ref}=\{0,0.25,0.5,0.75,1.0\}\,\textup{m}/\textup{s}$ and $\mathcal{R}^{OA}_{y,ref}=\{0,0.5\}\,\textup{m}$. In each case, we assume a uniform probability distribution for all hypotheses. All other parameters remain unchanged. The simulation results are shown in Table II. In our simulations, it generally takes hundreds of milliseconds to several seconds for the baseline method to converge. Meanwhile, owing to the merit of the potential formulation and the parallel solving scheme introduced, the proposed method improves the computational efficiency by orders of magnitude compared with the baseline, enabling the application of the contingency game to complicated traffic scenarios with a large number of hypotheses. Also, when the number of hypotheses increases from 2 to 10, the computation time of the baseline method increases by 82.8$\times$, while the proposed method increases by 9.8$\times$, showing better scalability.

\begin{table}[]
\centering
\caption{Comparison of computation time between contingency game~\cite{peters2024contingency} the proposed method with an increasing number of hypotheses. }
\renewcommand\arraystretch{1.5}
\begin{tabular}{cccccc}
\toprule[1.2pt]
Num. Hypo. & 2        & 4        & 6        & 8        & 10 \\ \toprule[1.2pt]

Baseline~\cite{peters2024contingency}    & 0.275\,s & 0.286\,s         &  0.491\,s        & 1.199\,s  & 22.771\,s\\
Ours (M)   & 0.006\,s & 0.009\,s & 0.014\,s & 0.042\,s & 0.059\,s\\ \toprule[1.2pt]
\end{tabular}
\label{Tab:time2}
\end{table}

\begin{table*}[]
\centering
\caption{Comparison of performance between baseline methods and the proposed method}
\renewcommand\arraystretch{1.6}
\begin{tabulary}{\textwidth}{C|CCCCC|CCCCC}
\toprule[1.5pt]
\multirow{2}{*}{} & \multicolumn{5}{c|}{Intersection}                                                   & \multicolumn{5}{c}{Merging}                                                        \\
          &  $\Tilde{||\Delta V||}\downarrow$        
          &  $\Tilde{||\Delta X||}\downarrow$              
          &  $\Tilde{||\delta||}\downarrow$              
          &  $\Tilde{||a||}\downarrow$           
          &  $\Tilde{||d||}\uparrow$             
          &  $\Tilde{||\Delta V||}\downarrow$        
          &  $\Tilde{||\Delta X||}\downarrow$              
          &  $\Tilde{||\delta||}\downarrow$              
          &  $\Tilde{||a||}\downarrow$           
          &  $\Tilde{||d||}\uparrow$  
          \\
            &  $(\textup{m}/\textup{s})$     
            &  $(\textup{m})$              
          &  $(\textup{rad})$              
          &  $(\textup{m}/\textup{s}^2)$           
          &  $(\textup{m})$             
          &  $(\textup{m}/\textup{s})$        
          &  $(\textup{m})$              
          &  $(\textup{rad})$              
          &  $(\textup{m}/\textup{s}^2)$           
          &  $(\textup{m})$  
          \\ \toprule[1.0pt]
MLE               & 0.568          & 0.810          & 0.202          & 0.991          & 2.389          & 0.486          & \textbf{3.204} & 0.084          & 0.343          & 2.115          \\ 
BNE               & 0.522          & \textbf{0.713} & 0.196          & 0.866          & 2.828          & 0.458          & 3.307          & 0.074          & 0.291          & \textbf{2.347} \\ 
MLE-Update        & 0.504          & 0.812          & 0.194          & 0.839          & 2.873          & 0.472          & 3.263          & 0.077          & 0.327          & 2.202          \\ 
BNE-Update        & \textbf{0.482} & 0.719          & \textbf{0.187} & \textbf{0.816} & \textbf{2.953} & \textbf{0.418} & 3.298          & \textbf{0.073} & \textbf{0.280} & 2.318          \\ \toprule[1.5pt]
\end{tabulary}
\end{table*}

\subsection{Closed-Loop Simulations}
We also perform closed-loop simulations to further demonstrate the effectiveness of the proposed method in addressing uncertainties. To generate the trajectories of the OAs, we assume that they know the underlying intentions of all agents. Therefore, the OAs are playing a game without uncertainties, and they follow the NE trajectories obtained through this game. Meanwhile, the EA does not know the intentions of the OAs for certain but has access to the distribution of the underlying intentions of the OAs. We compare the performance of the EA under the following four game settings.

\textit{MLE}: In this setting, the EA plays a game without uncertainties by assuming the intentions of the OAs through maximum-likelihood estimation (MLE), and follows the NE trajectory obtained through this game.

\textit{BNE}: In this setting, the EA plays the Bayesian game in Problem 3 and follows the BNE trajectory.

\textit{MLE-Update}: Similar to the MLE setting, the EA also plays a game without uncertainties by assuming the intentions of the OAs through MLE, but performs Bayesian filtering to update the distribution in each planning cycle. Bayesian filtering is performed based on the BNE trajectories obtained in the BNE setting.

\textit{BNE-Update}: Similar to the BNE setting, the EA plays the Bayesian game and follows the BNE trajectory with Bayesian filtering to update the distribution in each planning cycle.

\begin{figure}[t]
\centering

\subfigure[MLE setting]{\includegraphics[scale=0.45]{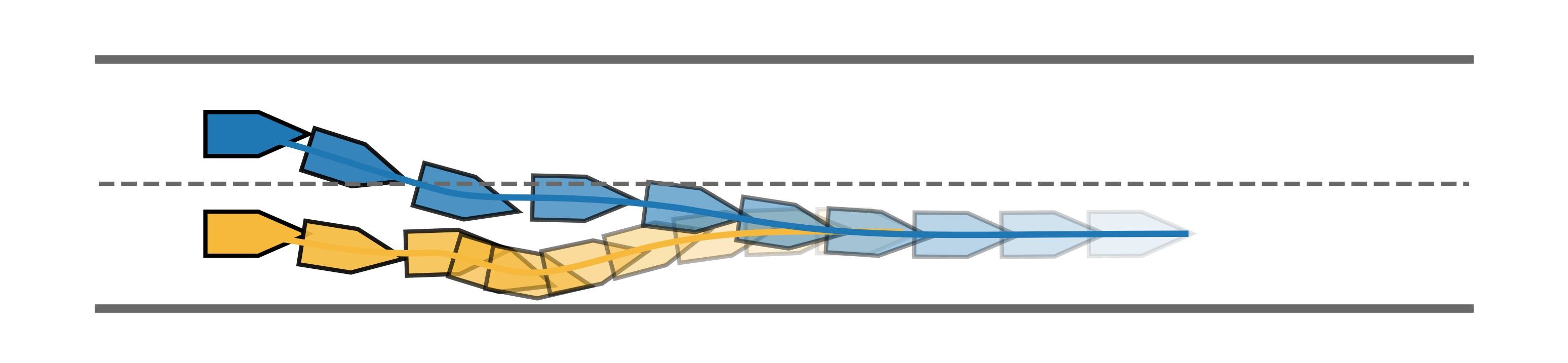}}
\subfigure[BNE setting]{\includegraphics[scale=0.45]{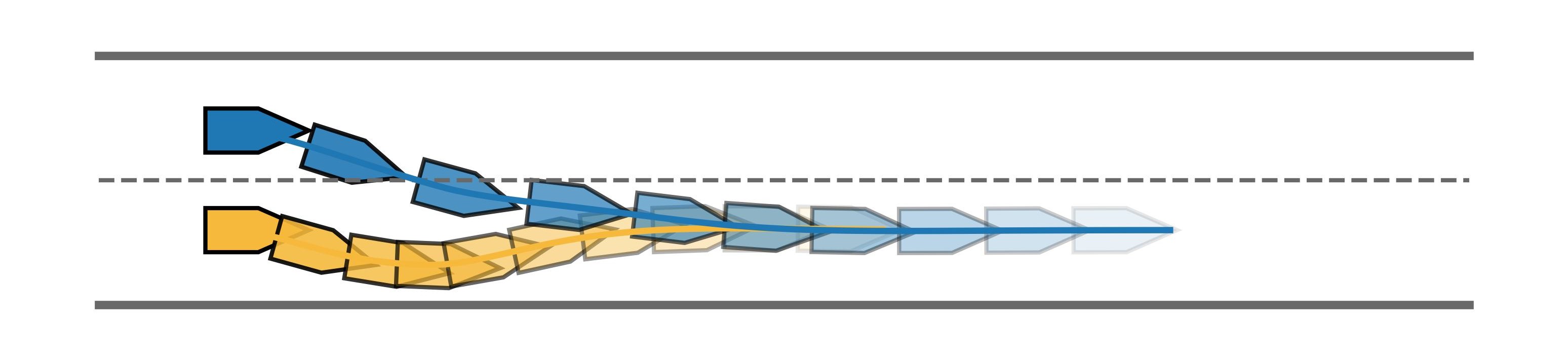}}
\caption{Closed-loop simulation results in the merging scenario under MLE and BNE game settings with $[w_1,w_2]=[0.51,0.49]$, and the real type of OA is that $v_{ref}=v_{ref,2}$. Under the MLE game setting, the EA believes that the OA is going to slow down while in fact the OA does the opposite. Due to the misjudgment of the OA's intention, the EA first speeds up and then takes an emergency brake to avoid a collision. Meanwhile, the BNE game setting generates a trajectory that is smooth and secure for both cases.}

\label{fig:MM}
\end{figure}

\begin{figure}[t]
\centering

\subfigure[MLE setting]{\includegraphics[scale=0.225]{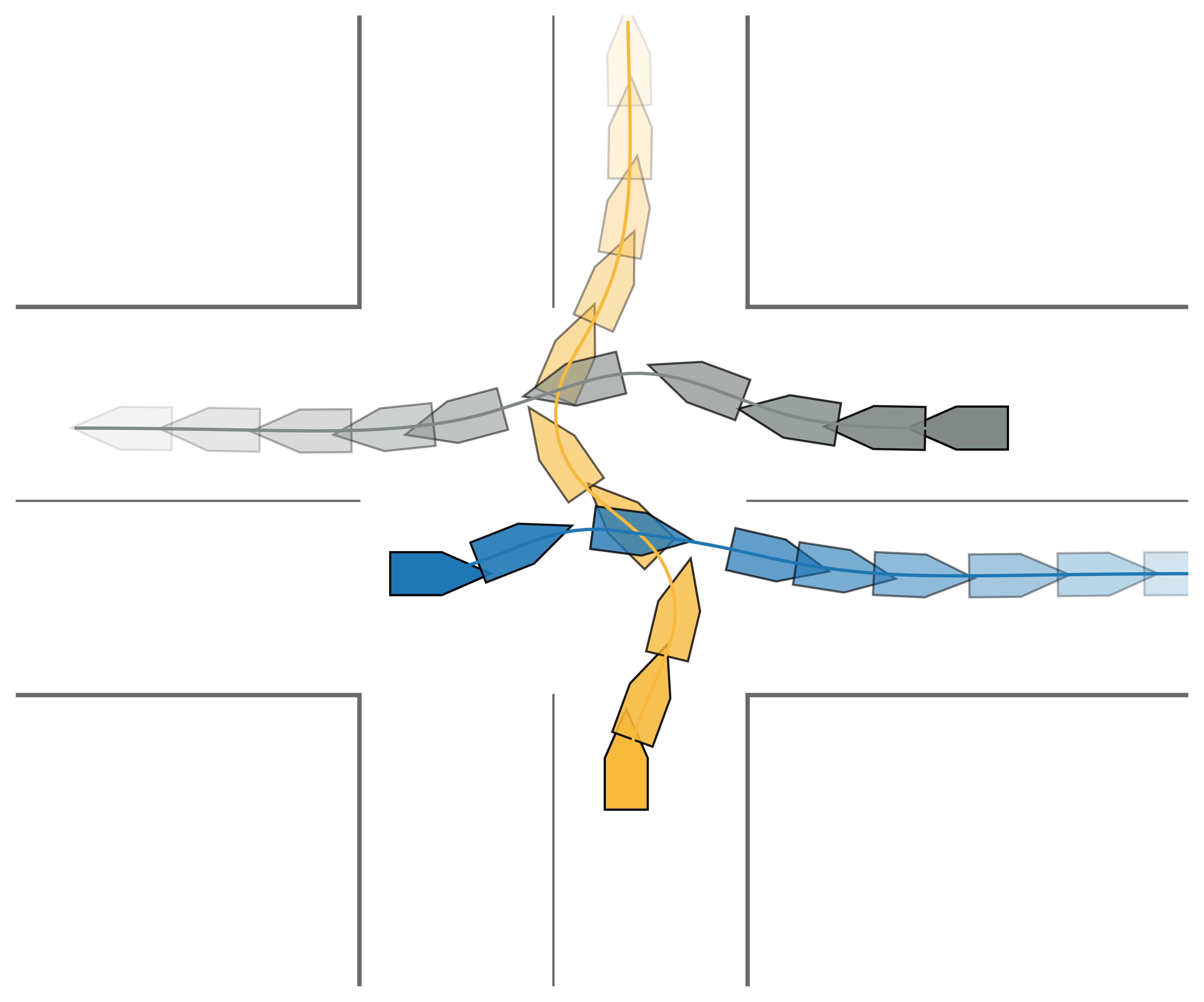}}
\subfigure[BNE setting]{\includegraphics[scale=0.225]{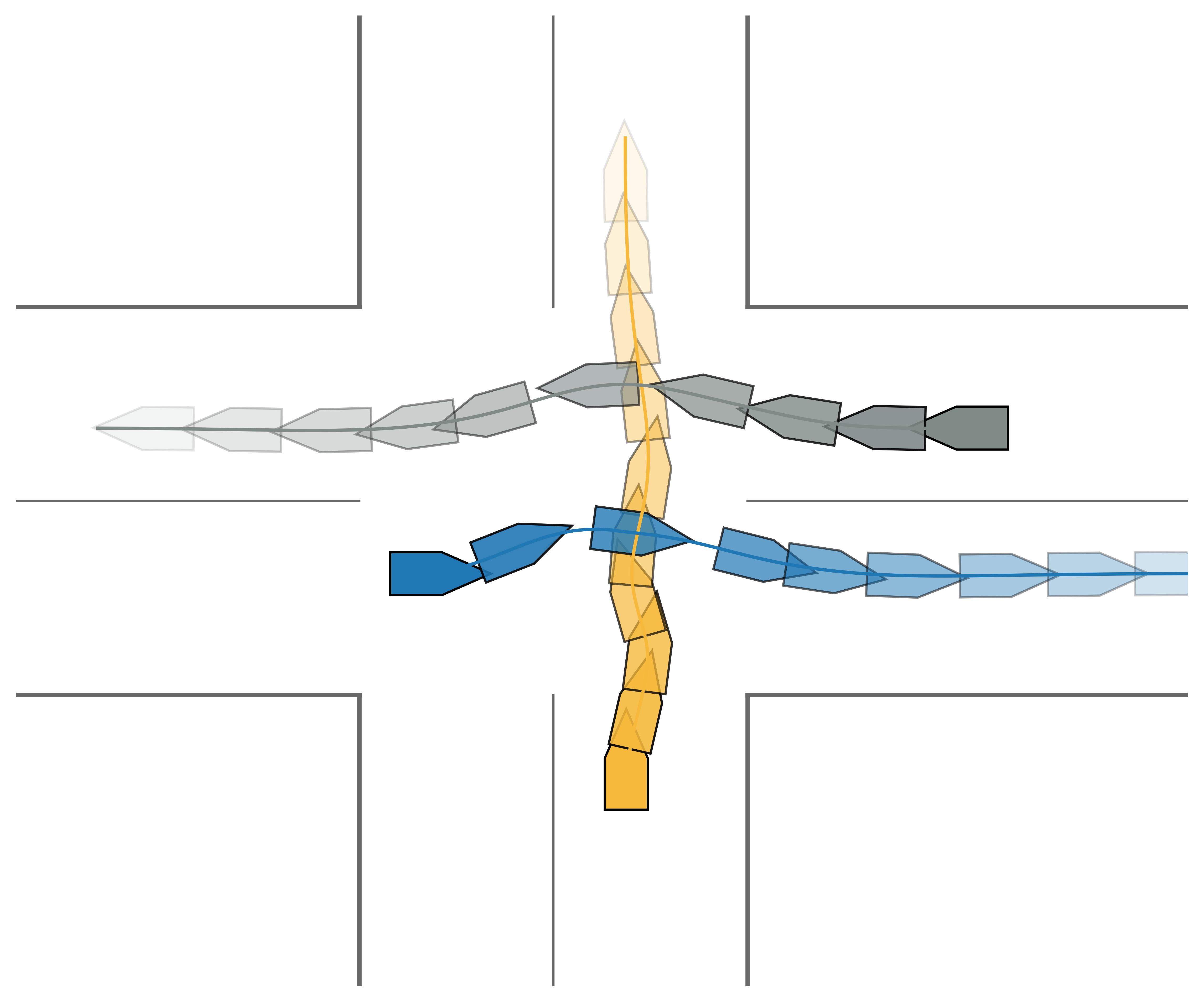}}
\caption{Closed-loop simulation results in the intersection scenario under MLE and BNE game settings with $[w^{OA1}_1,w^{OA1}_2,w^{OA2}_1,w^{OA2}_2]=[0.51,0.49,0.51,0.49]$, and the real types of OA1 and OA2 are that $v^{OA1}_{ref}=v^{OA1}_{ref,2}$ and $v^{OA2}_{ref}=v^{OA2}_{ref,2}$. Under the MLE game setting, the EA mistakenly believes that both OAs are going to slow down and yield. Therefore, it attempts to pass through the intersection in front of both OAs. This misjudgment leads to a collision as both OAs do the opposite. Meanwhile, the game with the BNE setting manages to generate a smooth and safe trajectory.}

\label{fig:CC}
\end{figure}

Qualitative results comparing between MLE and BNE game settings are shown in Fig. \ref{fig:MM} and Fig. \ref{fig:CC}. These results clearly demonstrate the limitation of the game method that only considers the most likely hypothesis. When there is a discrepancy between the hypothesis considered by the EA and the real intention of the OAs, the trajectory generated by the game method can lead to potential danger and even collisions. On the contrary, the game method based on BNE considers different hypotheses and is well-prepared for different intentions of the OAs, leading to smooth and safe trajectories. For quantitative analysis, performances of game methods under different settings are evaluated through a Monte Carlo study. For both the intersection scenario and the merging scenario, we randomly sample the initial states of both the OAs and the EA from uniform distributions centered around the initial conditions in Section VII-A. In particular, for the OAs, we also sample the distribution of their intended reference velocity by sampling the reference velocity $v_{ref}$ and the weight parameters $w_1$ and $w_2$. For each scenario, 100 random initial conditions are drawn. Meanwhile, for each initial condition, we draw 5 random samples of the actual intended reference velocities of the OAs from the distribution. Altogether 500 simulations are performed for each scenario. In each simulation, we simulate the receding horizon trajectories for 10\,s with the length of the planning cycle to be 2\,s. We compare the average deviation from the reference speed $\Tilde{||\Delta V||}$, the average deviation from the reference trajectory $\Tilde{||\Delta X||}$, the average steering angle $\Tilde{||\delta||}$, and the average acceleration $\Tilde{||a||}$ of the EA, as well as the average minimal distance between the EA and the OAs $\Tilde{||d||}$. Simulation results are shown in Table III. Comparing with the uni-hypothesis game solution that only considers the most likely hypothesis (the MLE method), the game solution based on BNE significantly enhances security by increasing the minimal distance between the EA and the OAs through interaction, and other performance indices are also improved. This is because the BNE trajectory of the EA is optimal in terms of mathematical expectation over the distribution of the OAs' intentions. Meanwhile, the results also verify the effectiveness of the Bayesian update in estimating the underlying intentions of the OAs; as in both the MLE and the BNE settings, the performance improves by applying the Bayesian filtering.

\begin{figure}[!htb]
\centering
\includegraphics[scale=0.23]{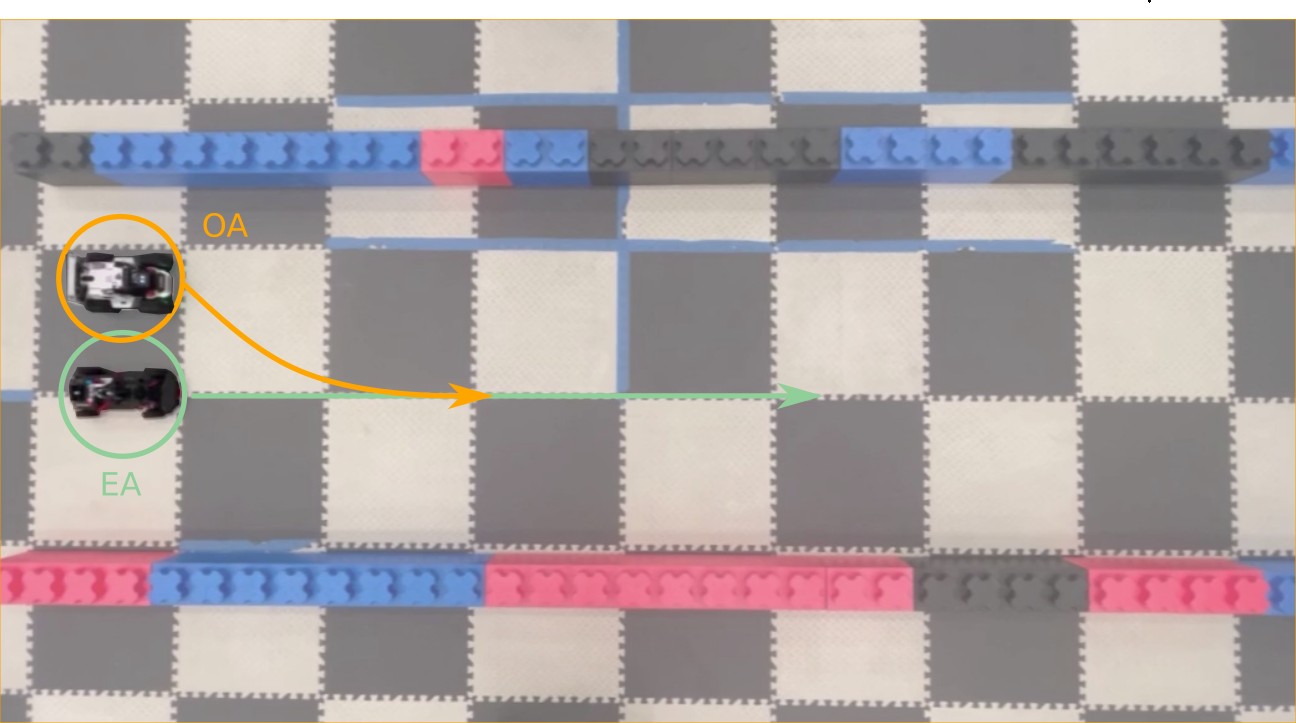}
\caption{Settings of the real-world experiments for the merging scenario. The EA is moving from left to right, and the OA aims to merge into the lane currently occupied by the EA. The reference speed of the EA is 0.15\,m/s. The OA is either a fast-type player with a reference longitudinal speed of 0.175\,m/s or a slow-type player with a reference longitudinal speed of 0.125\,m/s. Equal prior probabilities are assumed for both types.}

\label{fig:settings2}
\end{figure}

\begin{figure}[!htb]
\centering
\includegraphics[scale=0.37]{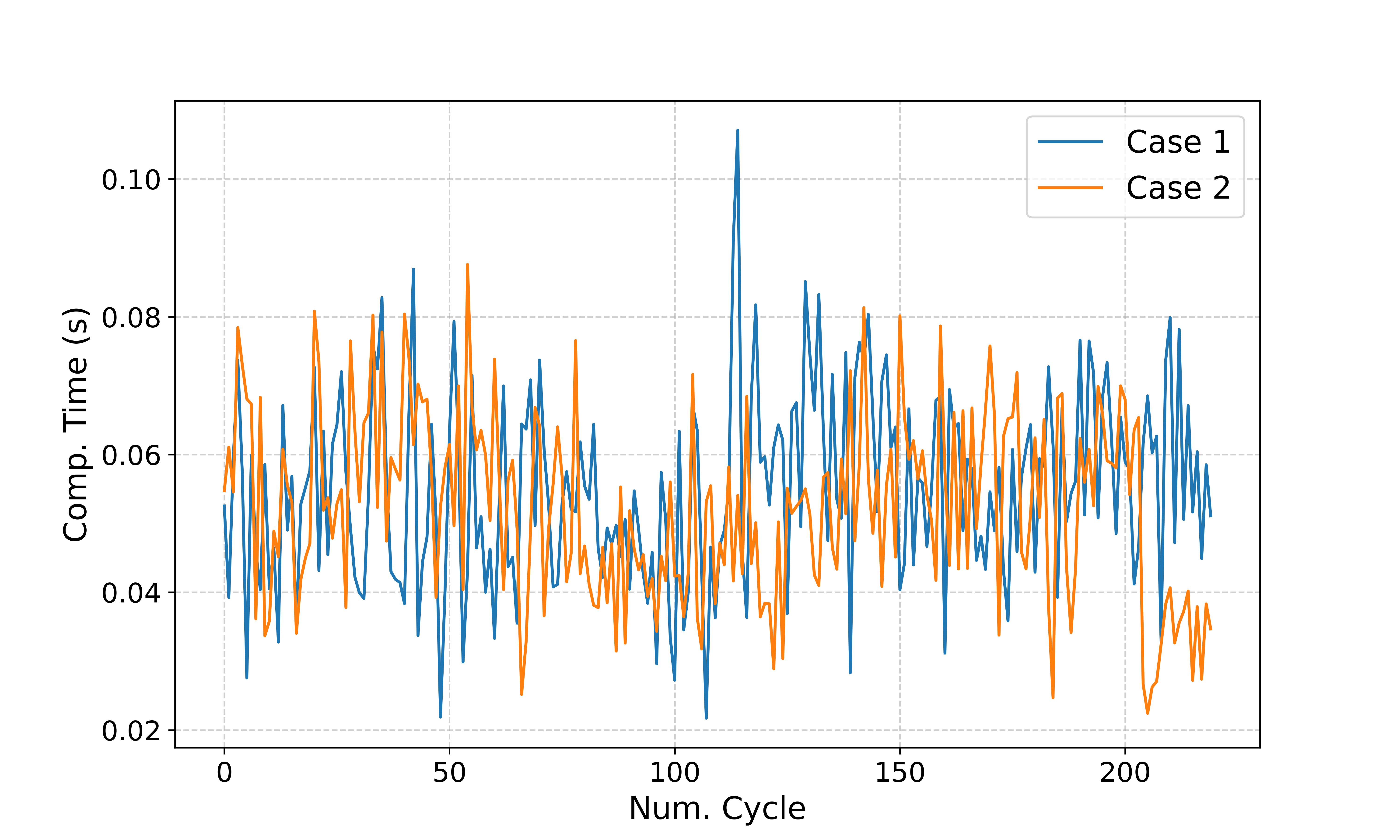}

\caption{Computation time in real-world experiments for the merging scenario. In particular, Case 1 refers to the case where OA is a slow agent; case 2 refers to the case where OA1 is a fast agent.}

\label{fig:realtimes2}
\end{figure}

\section{Real-World Experiments}

We also conduct real-world experiments to verify the effectiveness of the proposed method. In the experiments, a Tianbot T100 equipped with NVIDIA Jetson Xavier NX is used as the EA, and Agilex LIMO vehicles are used as the OAs. The experiments are set up with the Robot Operating System (ROS)~\cite{quigley2009ros} running on a laptop equipped with an 11th Gen Intel (R) Core (TM) i7-1165G7 CPU with a clock speed of 2.80\,GHz. To provide real-time feedback of the poses of all agents, we use the OptiTrack system, which operates at a frame rate of 200 FPS. The trajectories of all agents are generated following the same scheme as in Section VII-C with the BNE game settings. Meanwhile, a pure pursuit algorithm is utilized to perform the low-level control in order to track the generated trajectories. The entire system is running at a frequency of roughly 10\,Hz.

Similar to the closed-loop simulations, we also consider two scenarios, including the merging scenario and the intersection scenario. Particularly for the merging scenario, the settings of the experiments are shown in Fig. \ref{fig:settings2}. In the experiments for this scenario, we examine both cases where the OA is a fast and a slow agent.  Computation time against the number of planning cycles is plotted in Fig. \ref{fig:settings}, and the average planning time is roughly 0.05\,s, which satisfies the real-time requirement. Snapshots of experiments are shown in Fig. 
\ref{fig:realm1} and Fig. \ref{fig:realm2}. By performing the BNE strategy, the EA properly reacts to different intentions of the OA to facilitate smooth merging. Particularly, when the OA is a fast agent, the EA slows down to yield to the OA, allowing the OA to merge in the front; instead, when the OA is a slow agent, the EA accelerates to allow the OA to merge from behind. While comparing the first snapshot of both cases, a common behavior of the EA is to steer away from the OA, which is part of the optimal maneuver obtained by computing the BNE strategy. This enhances the safety as it increases the distance between the EA and the OA, creating enough space for the OA to merge in both cases.

\begin{figure}[!htb]
\centering
\begin{minipage}{0.33\linewidth}
    \includegraphics[scale=0.104]{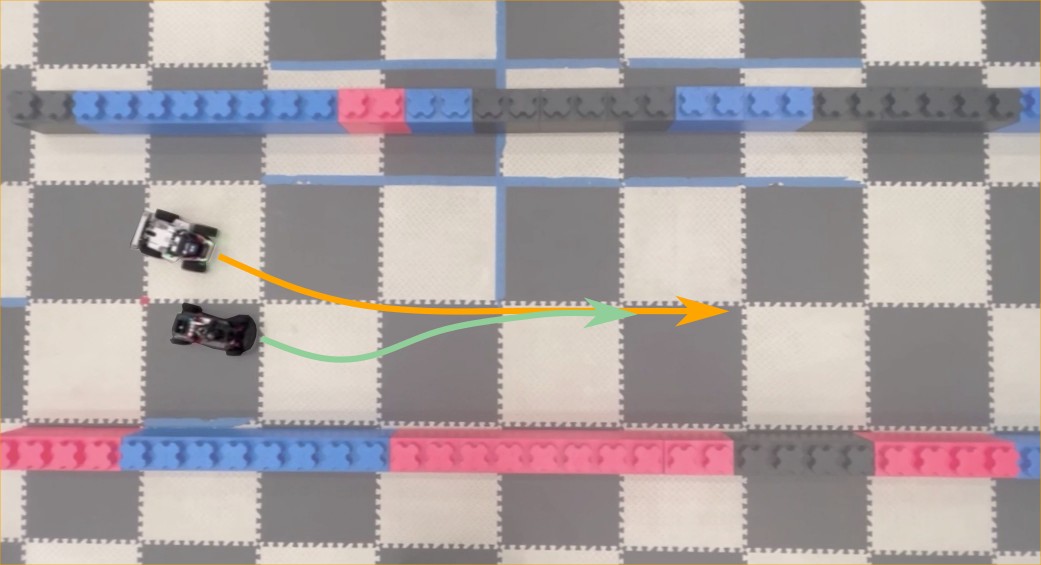}
\end{minipage}\hfill
\begin{minipage}{0.33\linewidth}
    \includegraphics[scale=0.104]{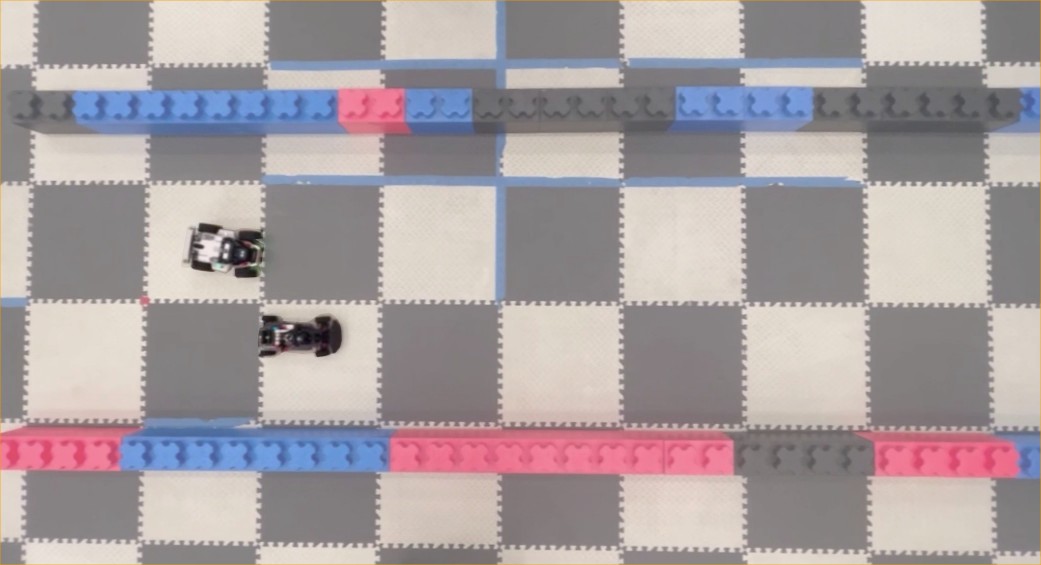}
\end{minipage}\hfill
\begin{minipage}{0.33\linewidth}
    \includegraphics[scale=0.104]{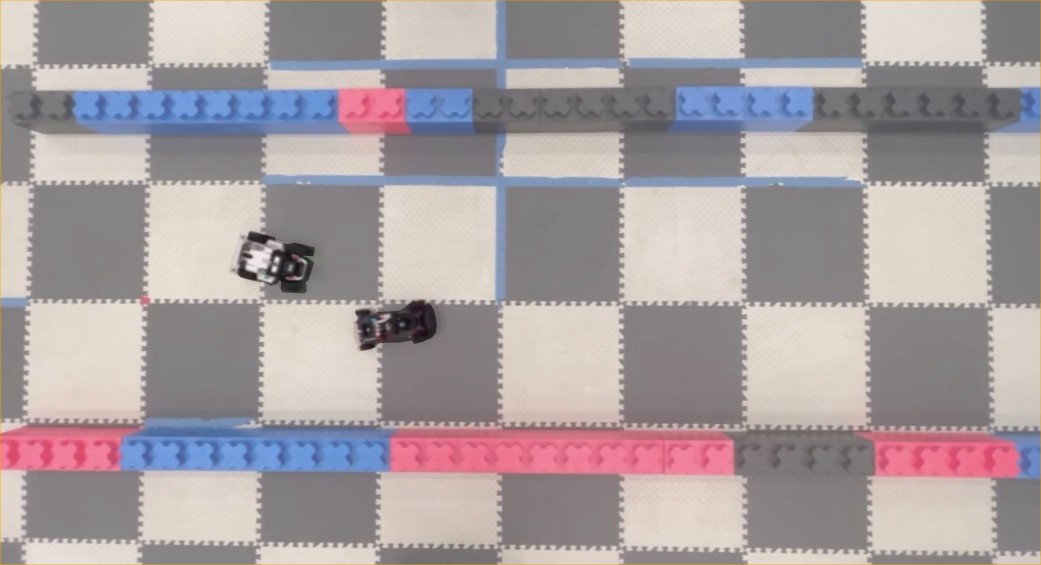}
\end{minipage}

\vspace{0.1cm}

\begin{minipage}{0.33\linewidth}
    \includegraphics[scale=0.104]{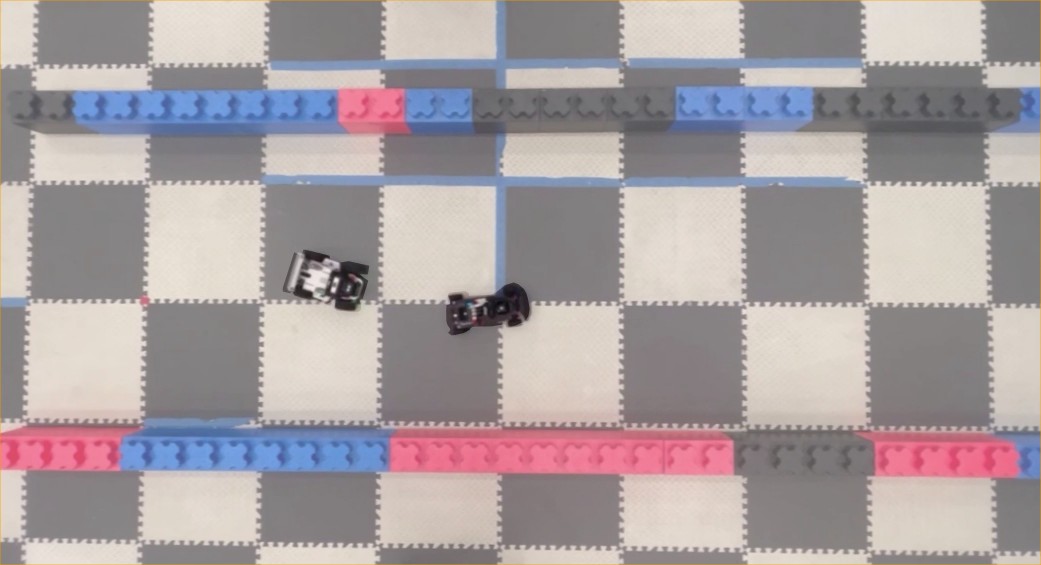}
\end{minipage}\hfill
\begin{minipage}{0.33\linewidth}
    \includegraphics[scale=0.104]{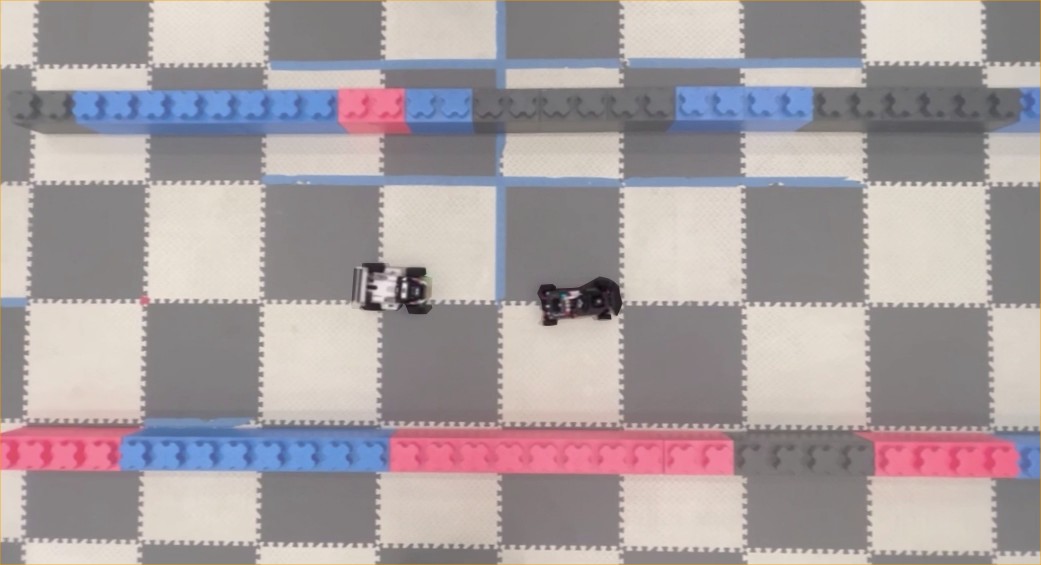}
\end{minipage}\hfill
\begin{minipage}{0.33\linewidth}
    \includegraphics[scale=0.104]{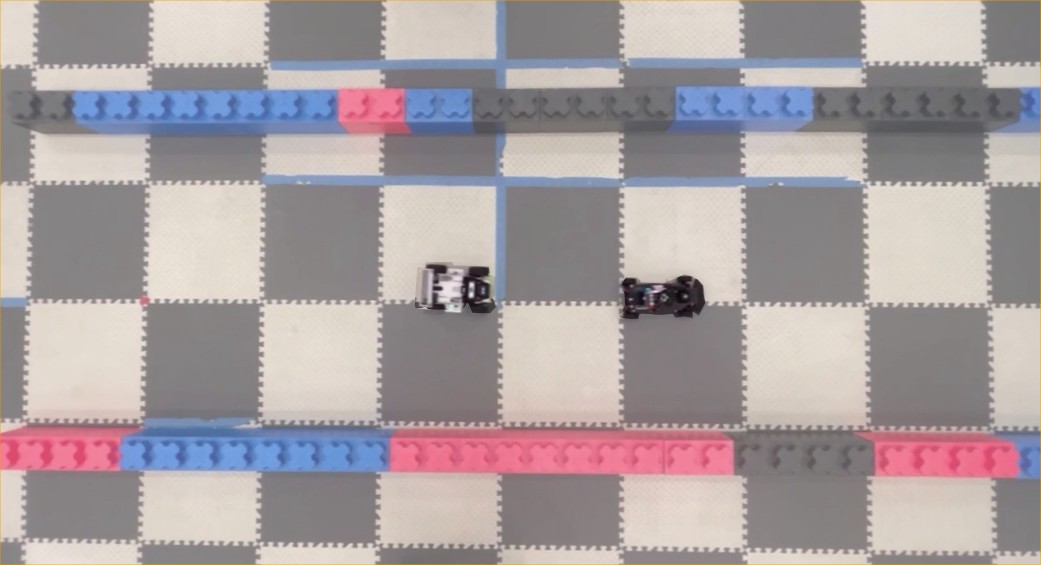}
\end{minipage}
\caption{Snapshots of real-world experiments for the merging scenario under the case when the OA is a slow agent. In this case, the OA merges into the lower lane from behind the EA.}
\label{fig:realm1}
\end{figure}

\begin{figure}[!htb]
\centering
\begin{minipage}{0.33\linewidth}
    \includegraphics[scale=0.104]{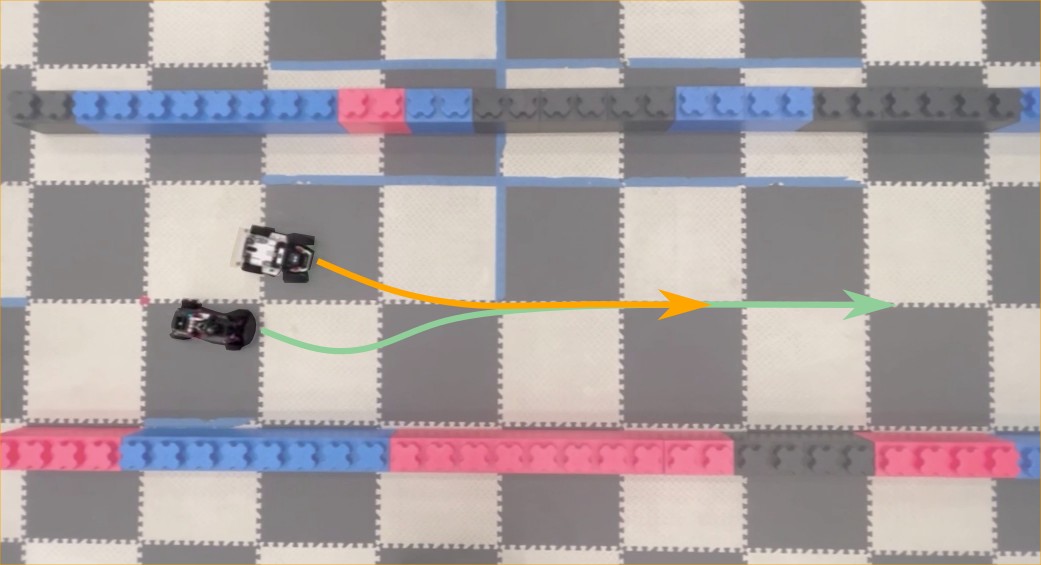}
\end{minipage}\hfill
\begin{minipage}{0.33\linewidth}
    \includegraphics[scale=0.104]{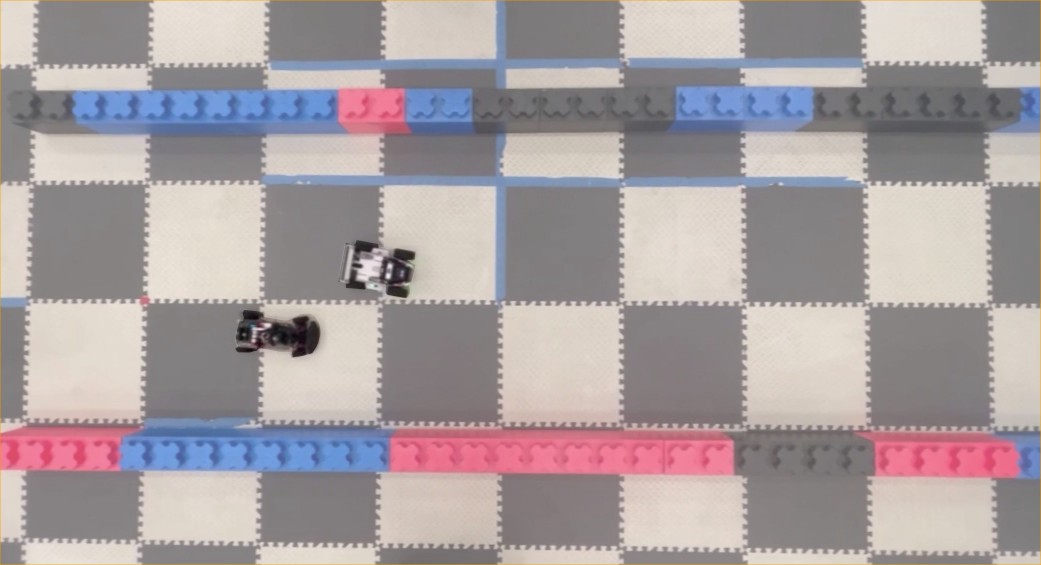}
\end{minipage}\hfill
\begin{minipage}{0.33\linewidth}
    \includegraphics[scale=0.104]{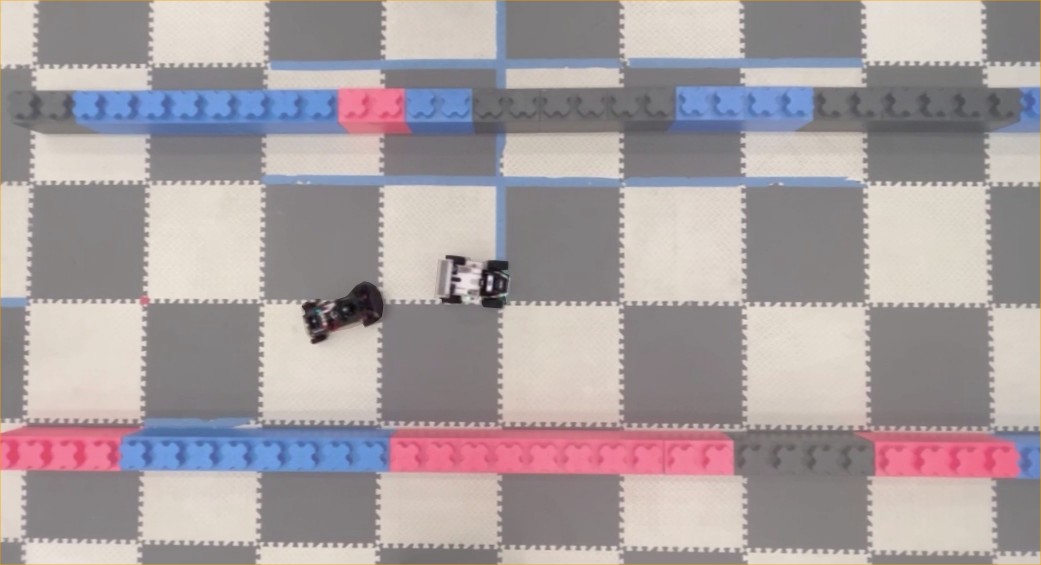}
\end{minipage}

\vspace{0.1cm}

\begin{minipage}{0.33\linewidth}
    \includegraphics[scale=0.104]{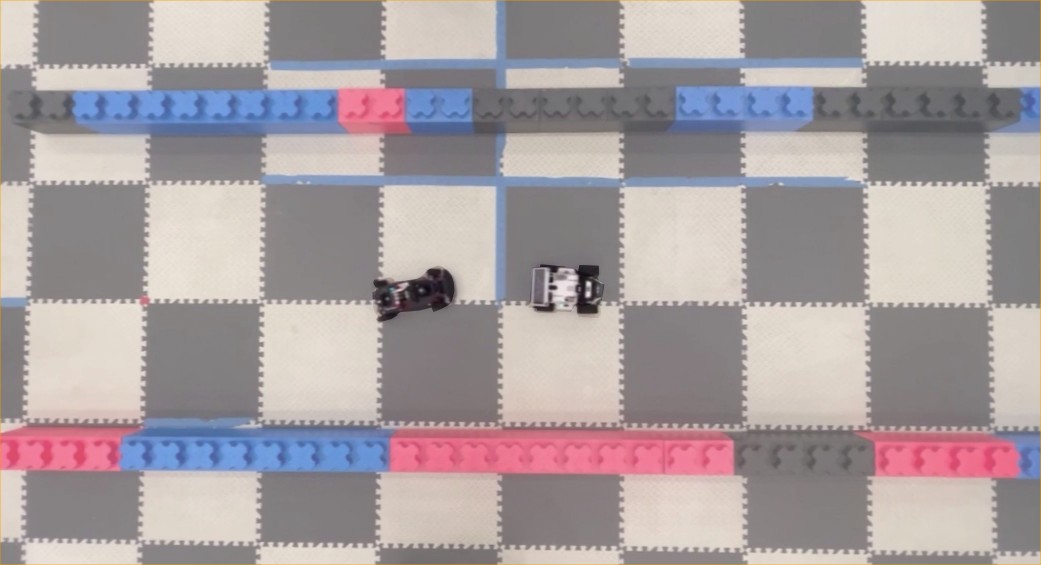}
\end{minipage}\hfill
\begin{minipage}{0.33\linewidth}
    \includegraphics[scale=0.104]{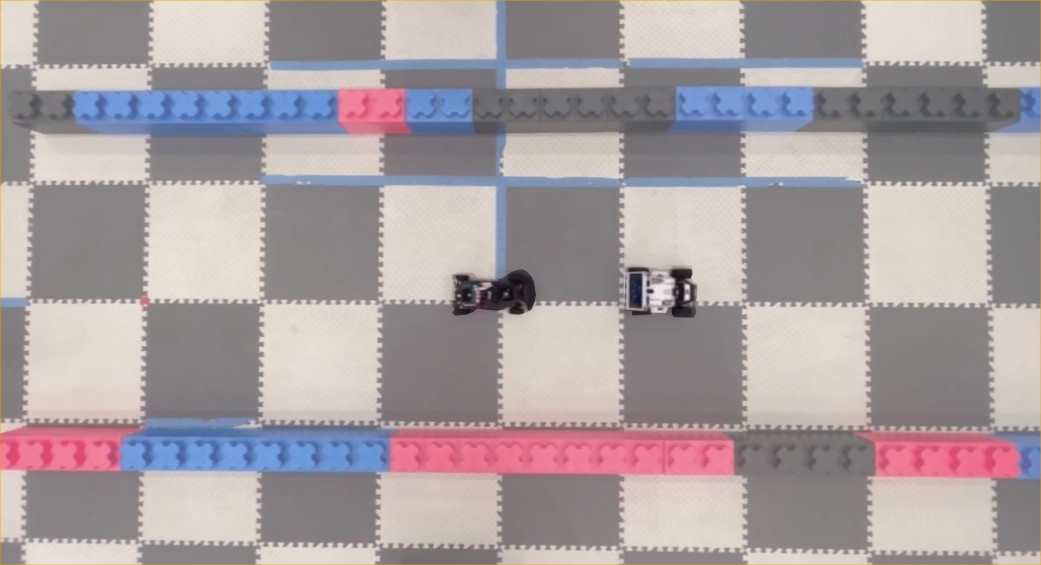}
\end{minipage}\hfill
\begin{minipage}{0.33\linewidth}
    \includegraphics[scale=0.104]{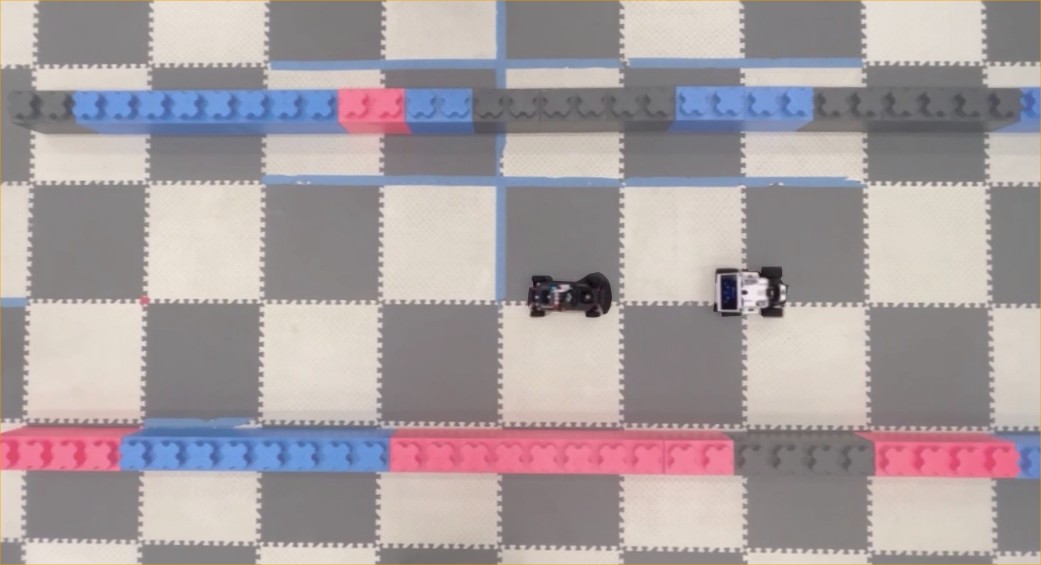}
\end{minipage}
\caption{Snapshots of real-world experiments for the merging scenario under the case when the OA is a fast agent. In this case, the OA merges into the lower lane in front of the EA.}
\label{fig:realm2}
\end{figure}

For the intersection scenario, the settings are shown in Fig. \ref{fig:settings}. We enumerate all possible combinations of the types of OAs. The computation time against the number of planning cycles is plotted in Fig. \ref{fig:realtimes}. From these curves, we conclude that the computation time of the BNE strategy is roughly 0.05\,s in average, featuring real-time performance. Meanwhile, the experimental results are shown in the snapshots in Figs. \ref{fig:real1}-\ref{fig:real4}. It can be seen from these snapshots that by performing the BNE strategy, the EA manages to handle different situations. Specifically, when the OA is a fast-type agent, the EA passes through the intersection from behind it; on the other hand, when the OA is a slow-type agent, the EA passes through the intersection from the front of it. A notable fact observed by comparing the first snapshot of the four different cases is that in all cases, the EA swerves to the right at the beginning, which is part of the BNE strategy given the common initial conditions and the initial beliefs over the intention of OAs in all cases. This maneuver is optimal in the sense that it steers the EA away from OA1, which is closer to the EA, creating space for collision avoidance and safety margins to properly react to different intentions of OAs. As a result, EA manages to pass the intersection without collisions in all cases. 

In conclusion, the above experiments show that the proposed method is deployable to real-world systems, and its computational efficiency meets the real-time requirement of the trajectory planning task with uncertainties in highly interactive scenarios.

\begin{figure}[!htb]
\centering
\includegraphics[scale=0.22]{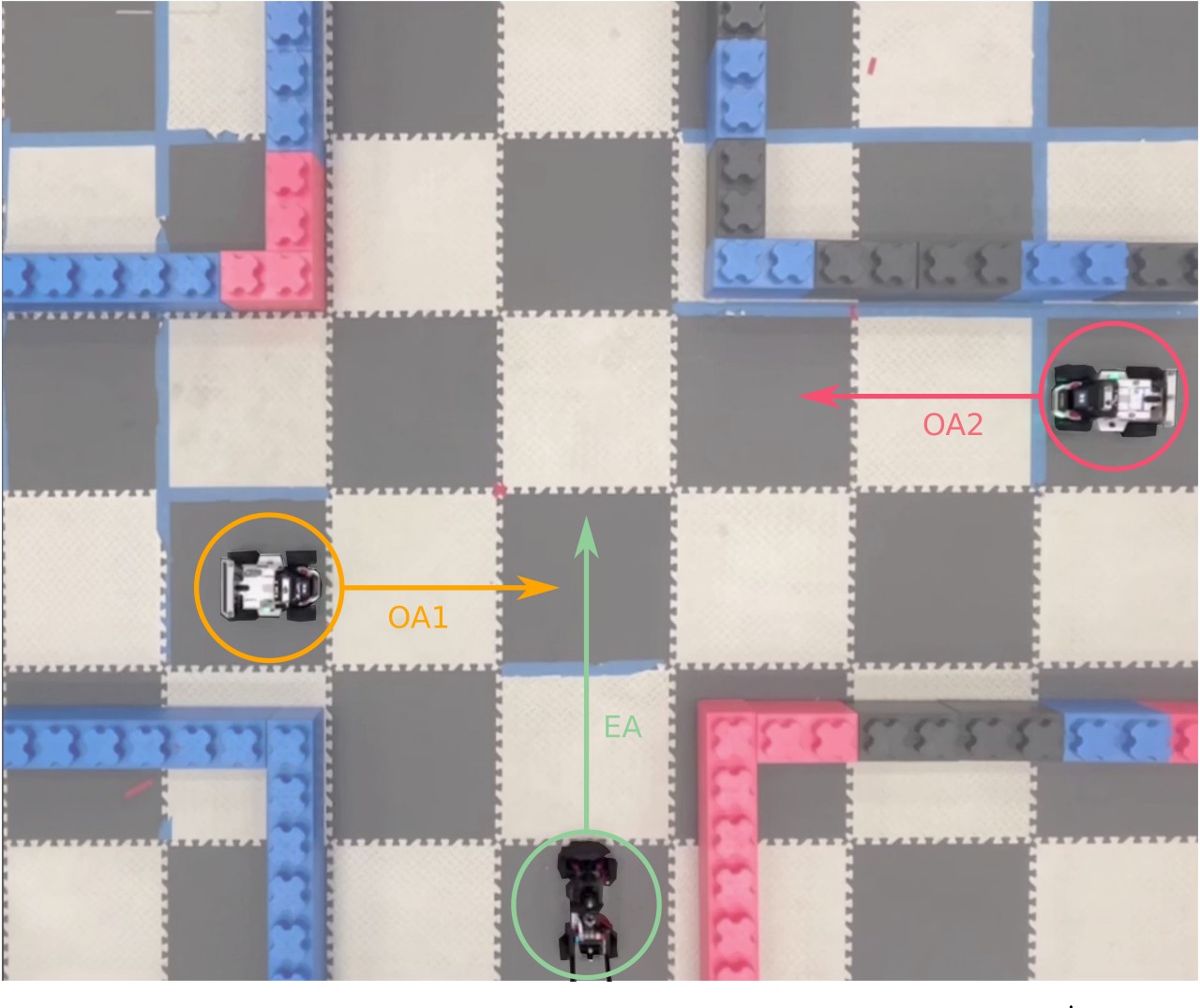}
\caption{Settings of the real-world experiments for the intersection scenario. The EA, the OA1, and the OA2 are crossing the intersection from bottom to top, from left to right, and from right to left, respectively. The reference speed of the EA is 0.15\,m/s. Each OA is set to be either a fast-type player or a slow-type player. The fast-type OA aims to track a longitudinal velocity of 0.18\,m/s, while the slow-type OA aims to track a longitudinal velocity of 0.12\,m/s. Equal prior probabilities are assumed for all types.}

\label{fig:settings}
\end{figure}

\begin{figure}[t]
\centering
\includegraphics[scale=0.37]{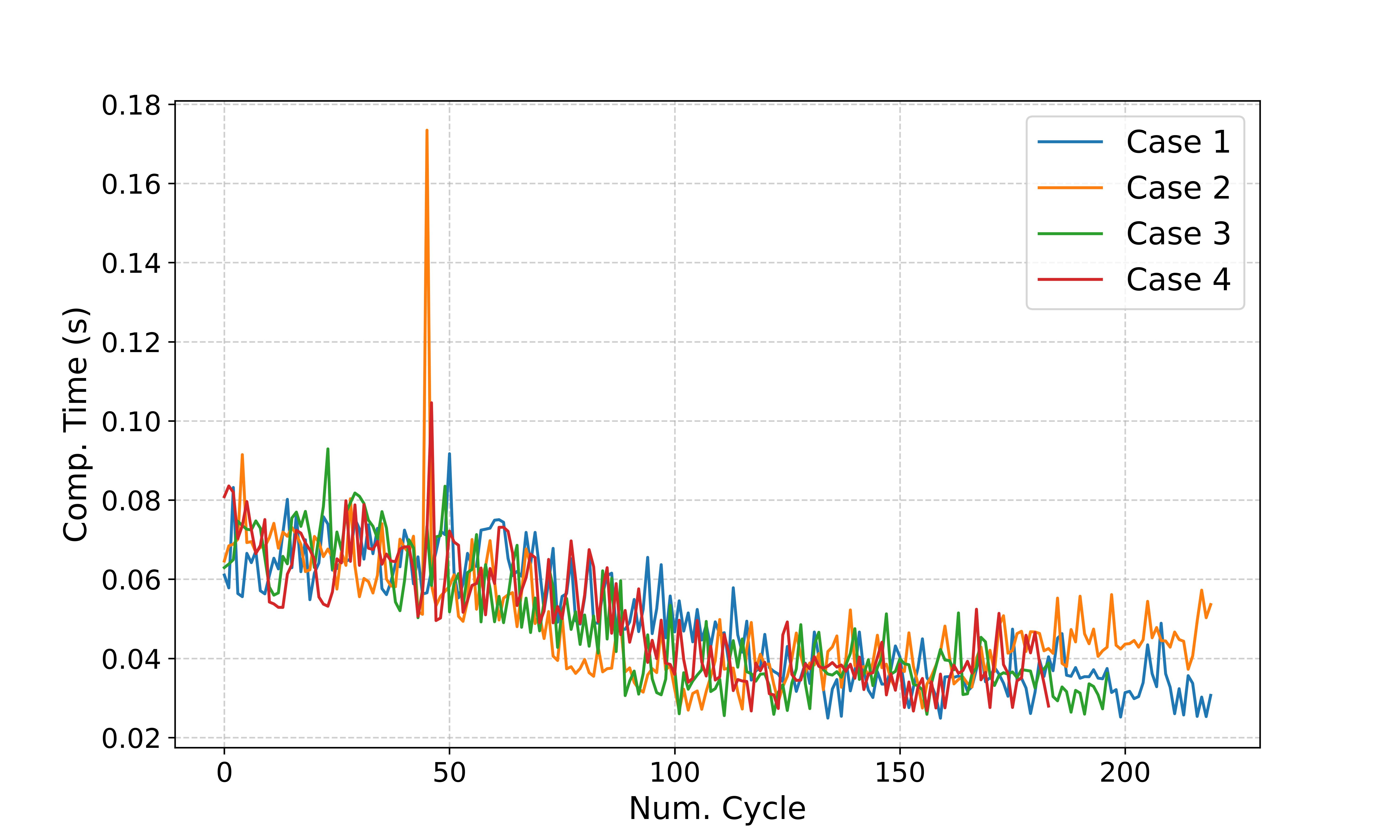}

\caption{Computation time in real-world experiments for the intersection scenario. In particular, Case 1 refers to the case where both OAs are slow agents; Case 2 refers to the case where OA1 is slow and OA2 is fast; Case 3 refers to the case where OA1 is fast and OA2 is slow; Case 4 refers to the case where both OAs are fast.}

\label{fig:realtimes}
\end{figure}

\begin{figure}[!htb]
\centering
\begin{minipage}{0.33\linewidth}
    \includegraphics[scale=0.085]{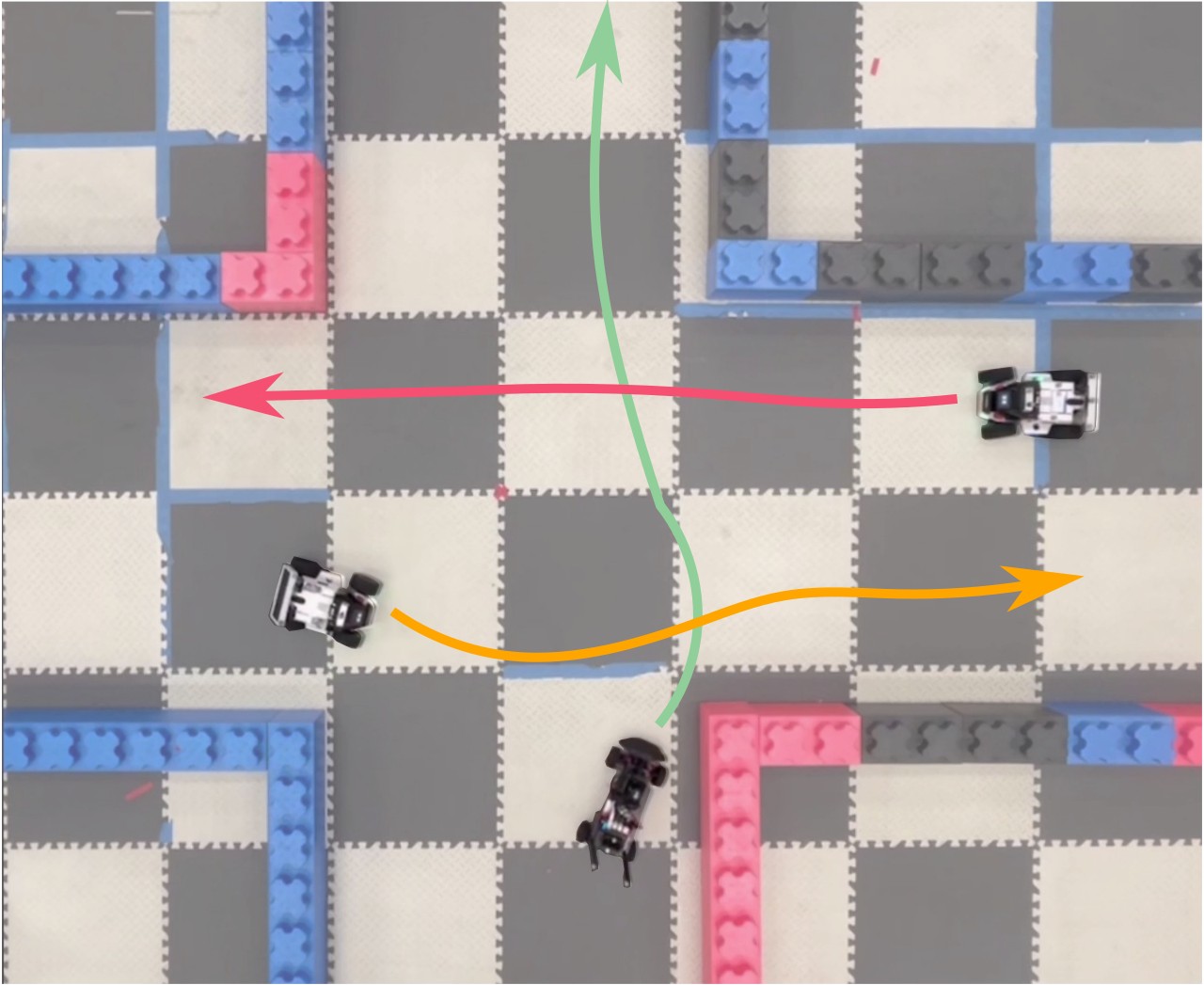}
\end{minipage}\hfill
\begin{minipage}{0.33\linewidth}
    \includegraphics[scale=0.085]{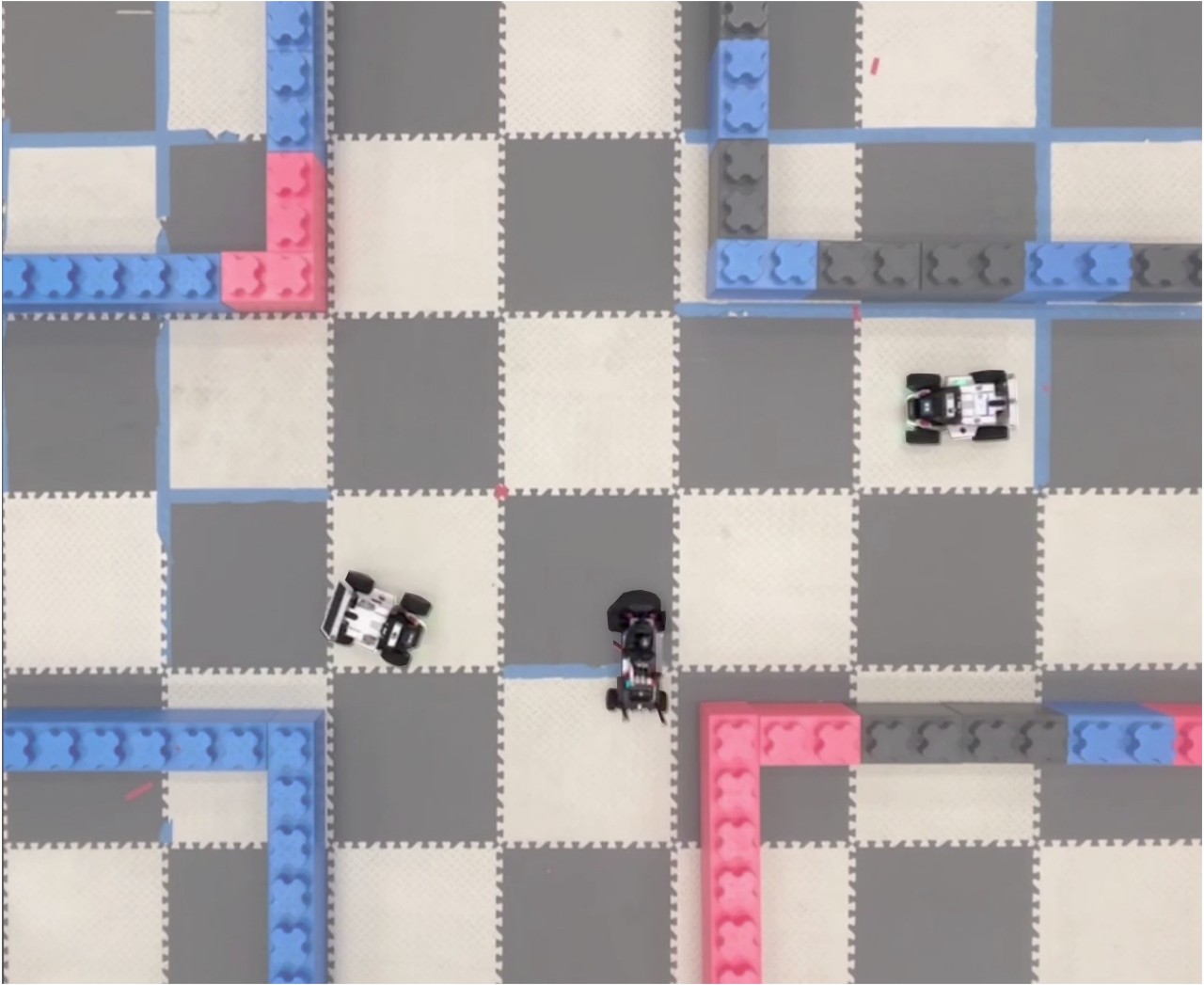}
\end{minipage}\hfill
\begin{minipage}{0.33\linewidth}
    \includegraphics[scale=0.085]{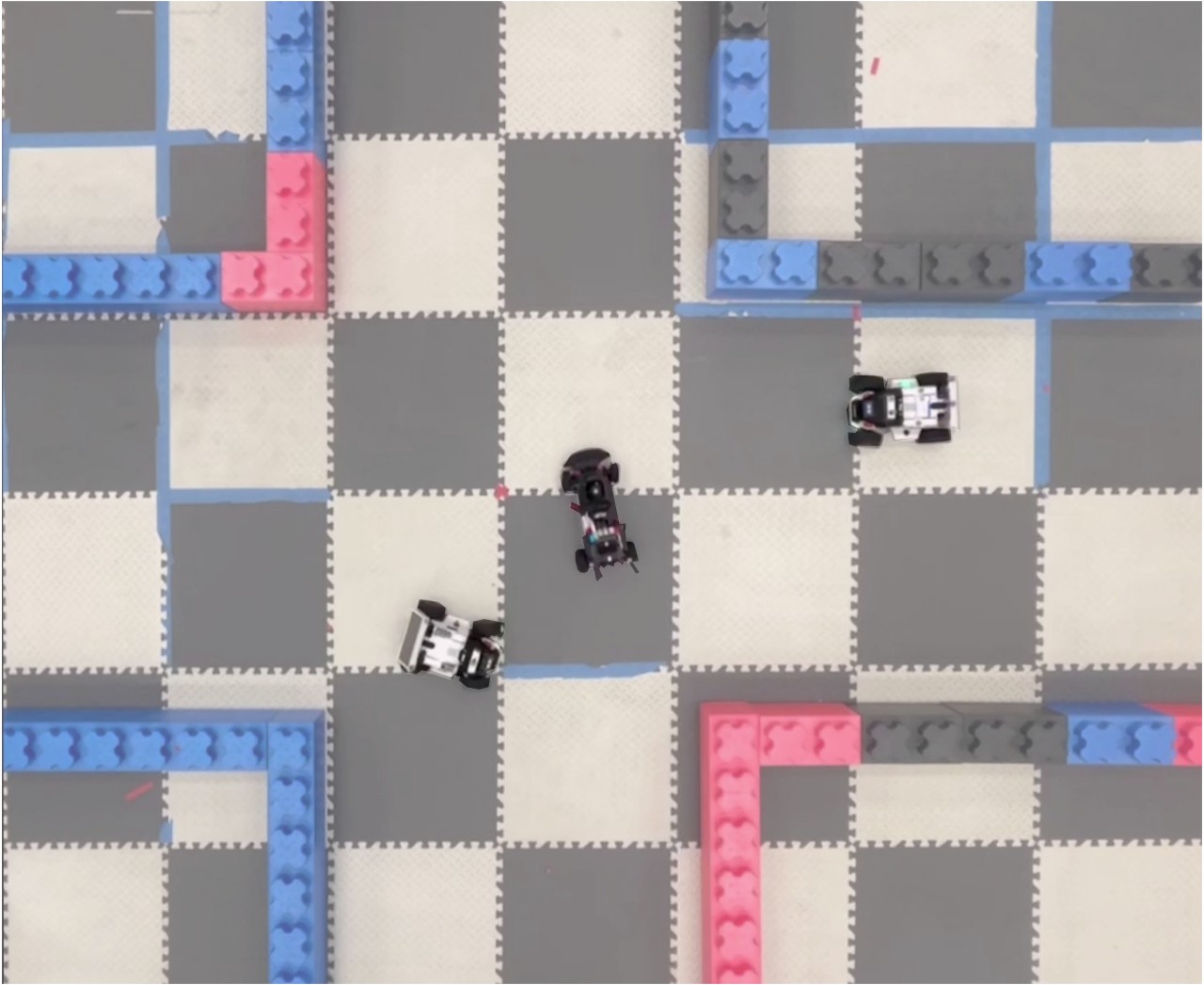}
\end{minipage}

\vspace{0.1cm}

\begin{minipage}{0.33\linewidth}
    \includegraphics[scale=0.085]{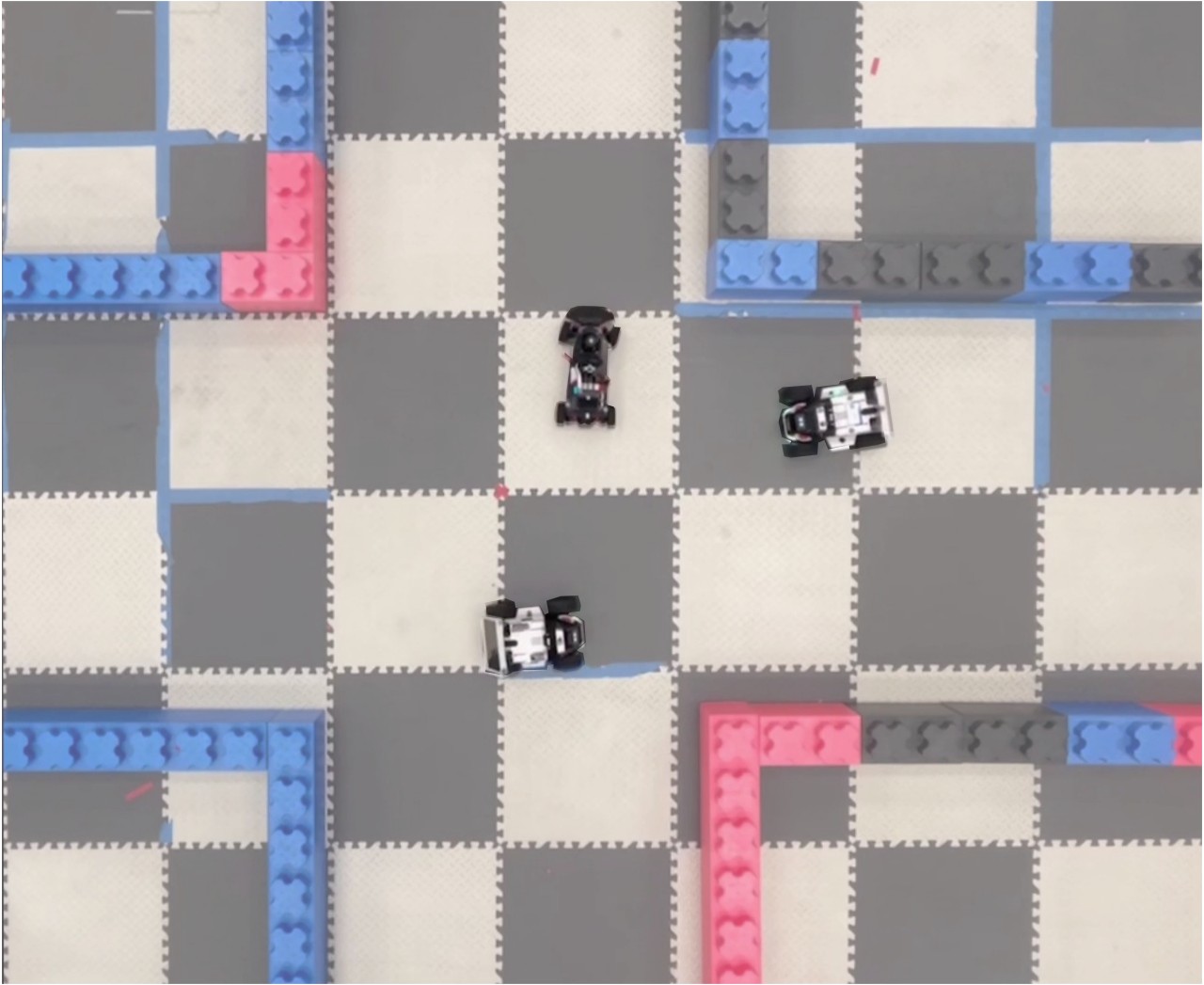}
\end{minipage}\hfill
\begin{minipage}{0.33\linewidth}
    \includegraphics[scale=0.085]{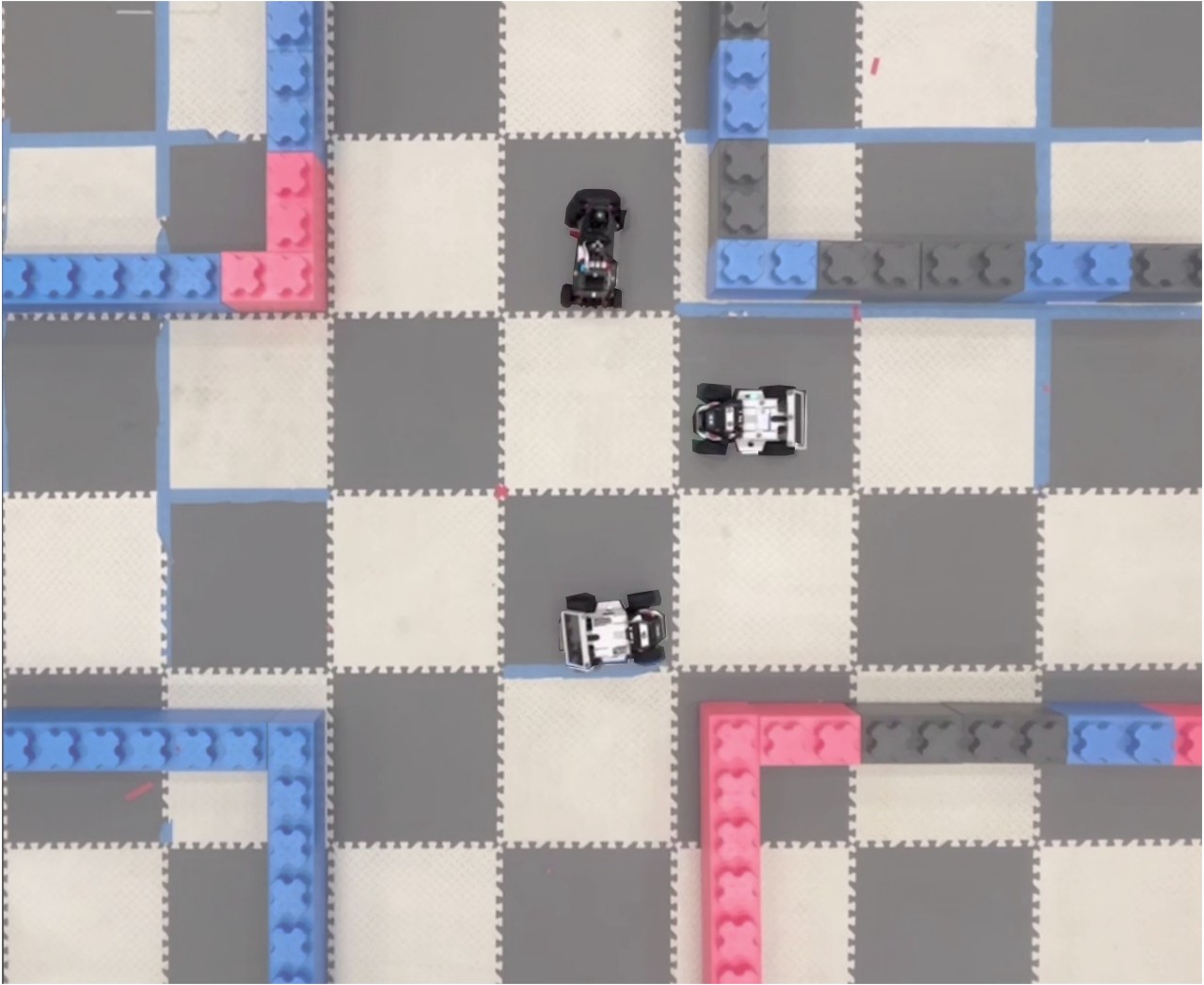}
\end{minipage}\hfill
\begin{minipage}{0.33\linewidth}
    \includegraphics[scale=0.085]{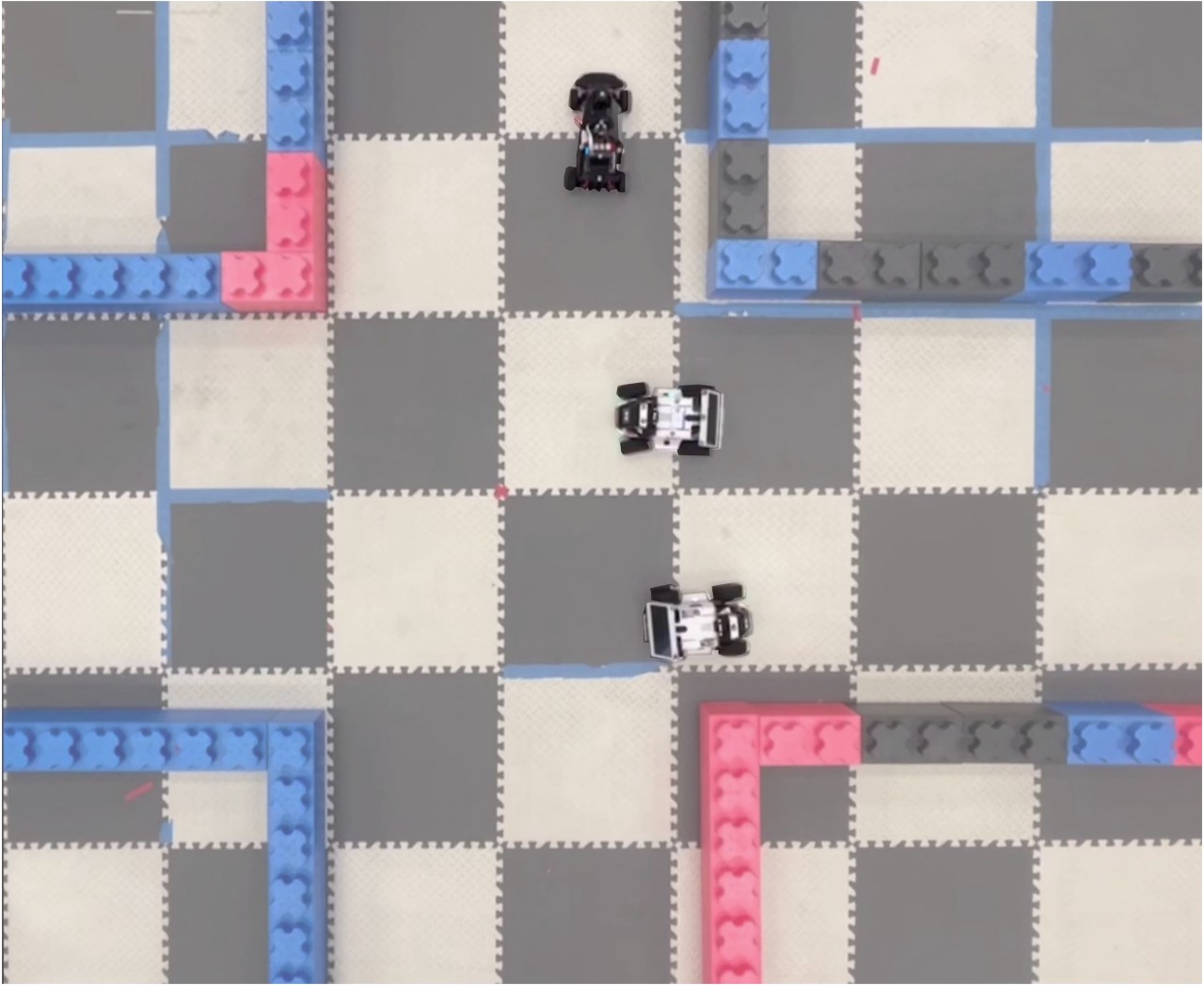}
\end{minipage}
\caption{Snapshots of real-world experiments for the intersection scenario under the case when both OA1 and OA2 are slow agents. In this case, the EA manages to pass through the intersection in front of both OAs.}
\label{fig:real1}
\end{figure}

\begin{figure}[!htb]
\centering
\begin{minipage}{0.33\linewidth}
    \includegraphics[scale=0.085]{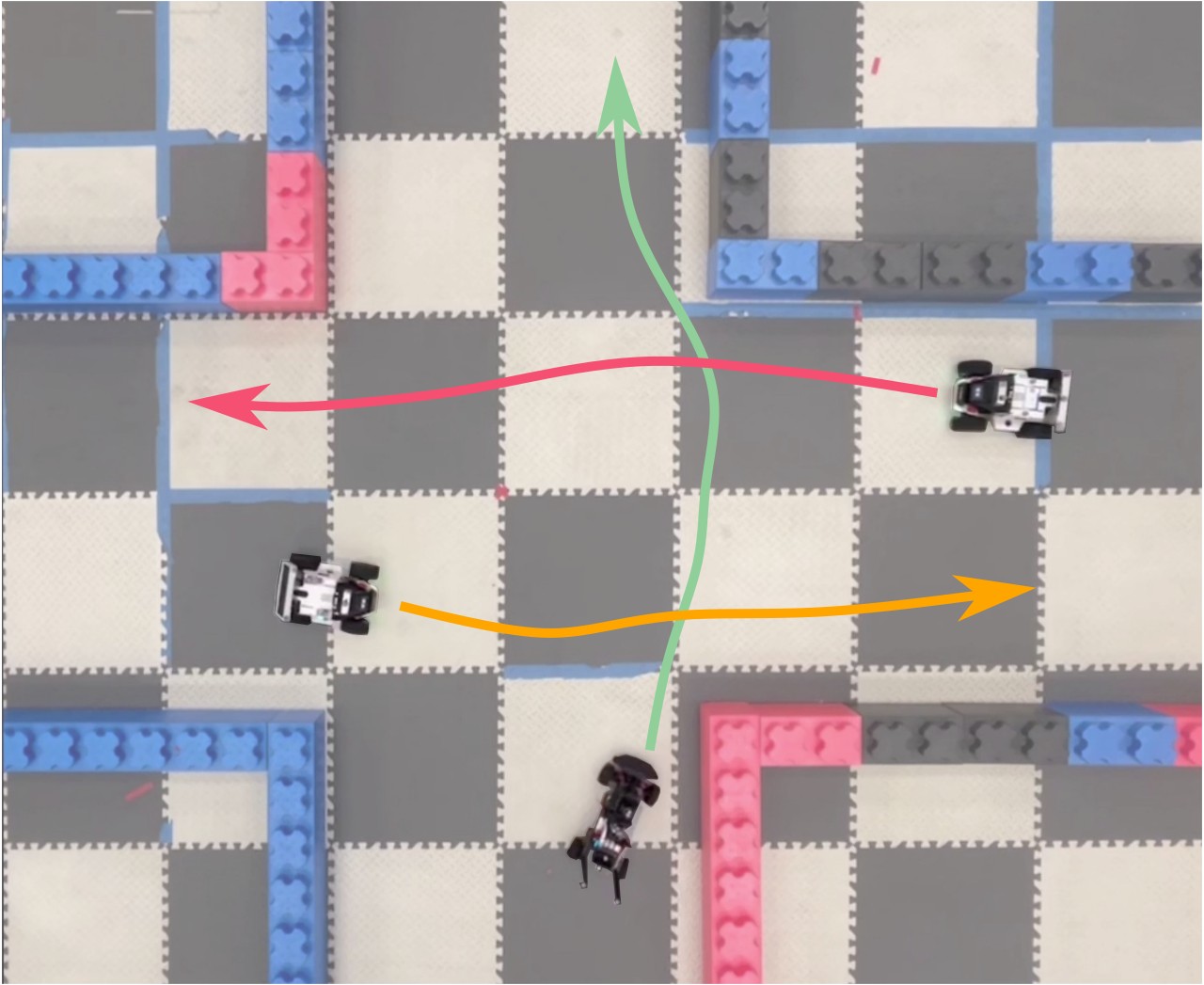}
\end{minipage}\hfill
\begin{minipage}{0.33\linewidth}
    \includegraphics[scale=0.085]{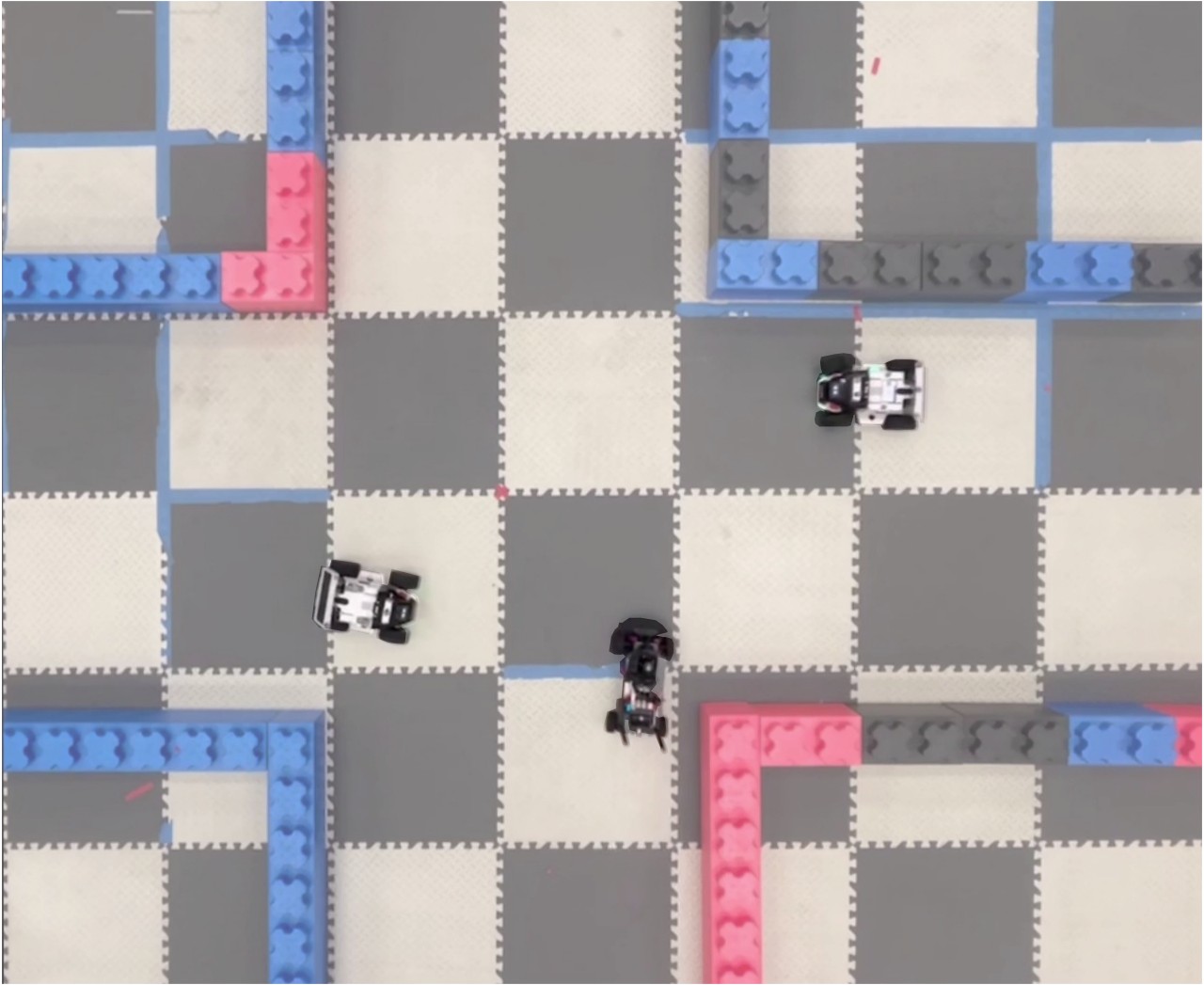}
\end{minipage}\hfill
\begin{minipage}{0.33\linewidth}
    \includegraphics[scale=0.085]{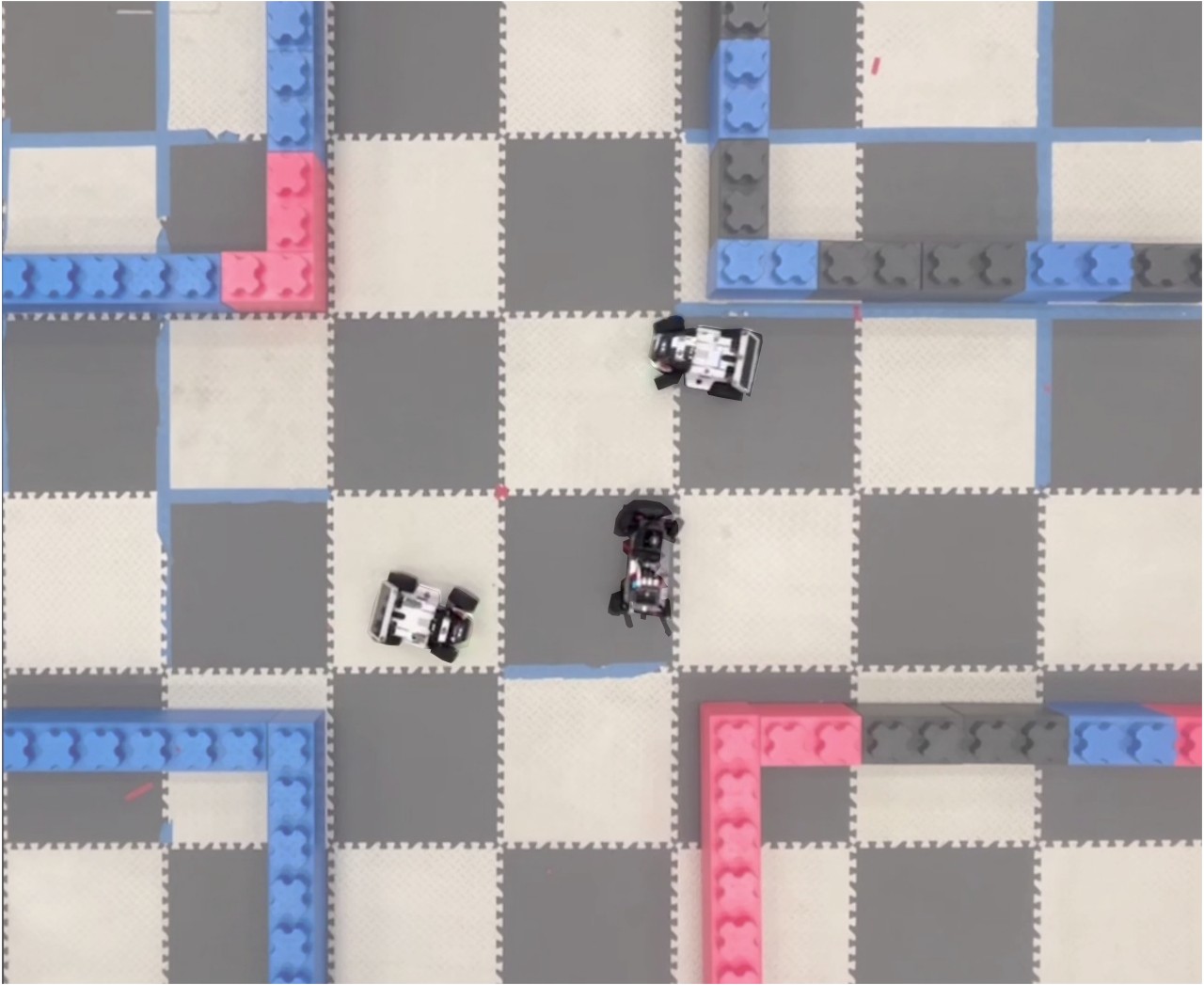}
\end{minipage}

\vspace{0.1cm}

\begin{minipage}{0.33\linewidth}
    \includegraphics[scale=0.085]{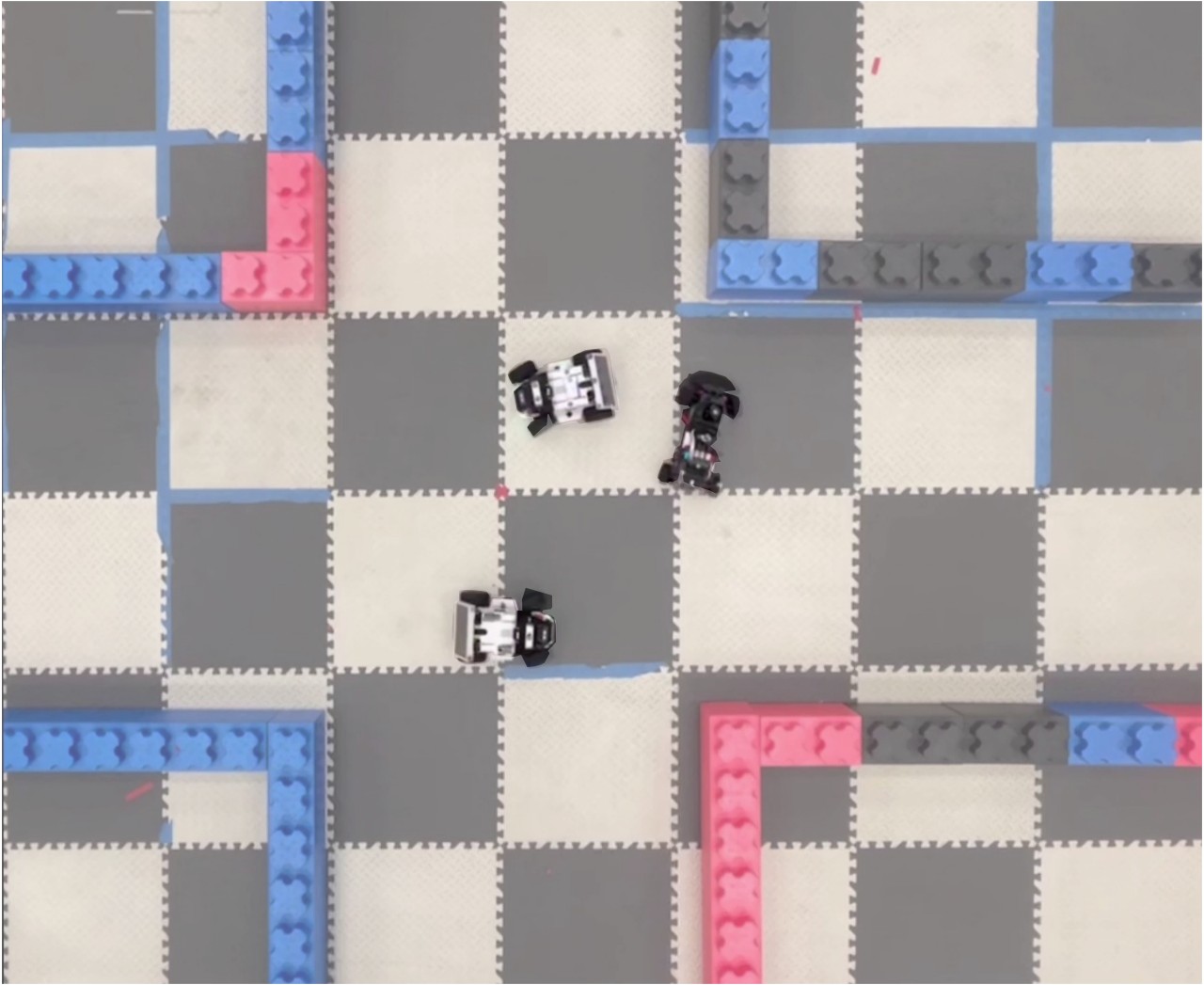}
\end{minipage}\hfill
\begin{minipage}{0.33\linewidth}
    \includegraphics[scale=0.085]{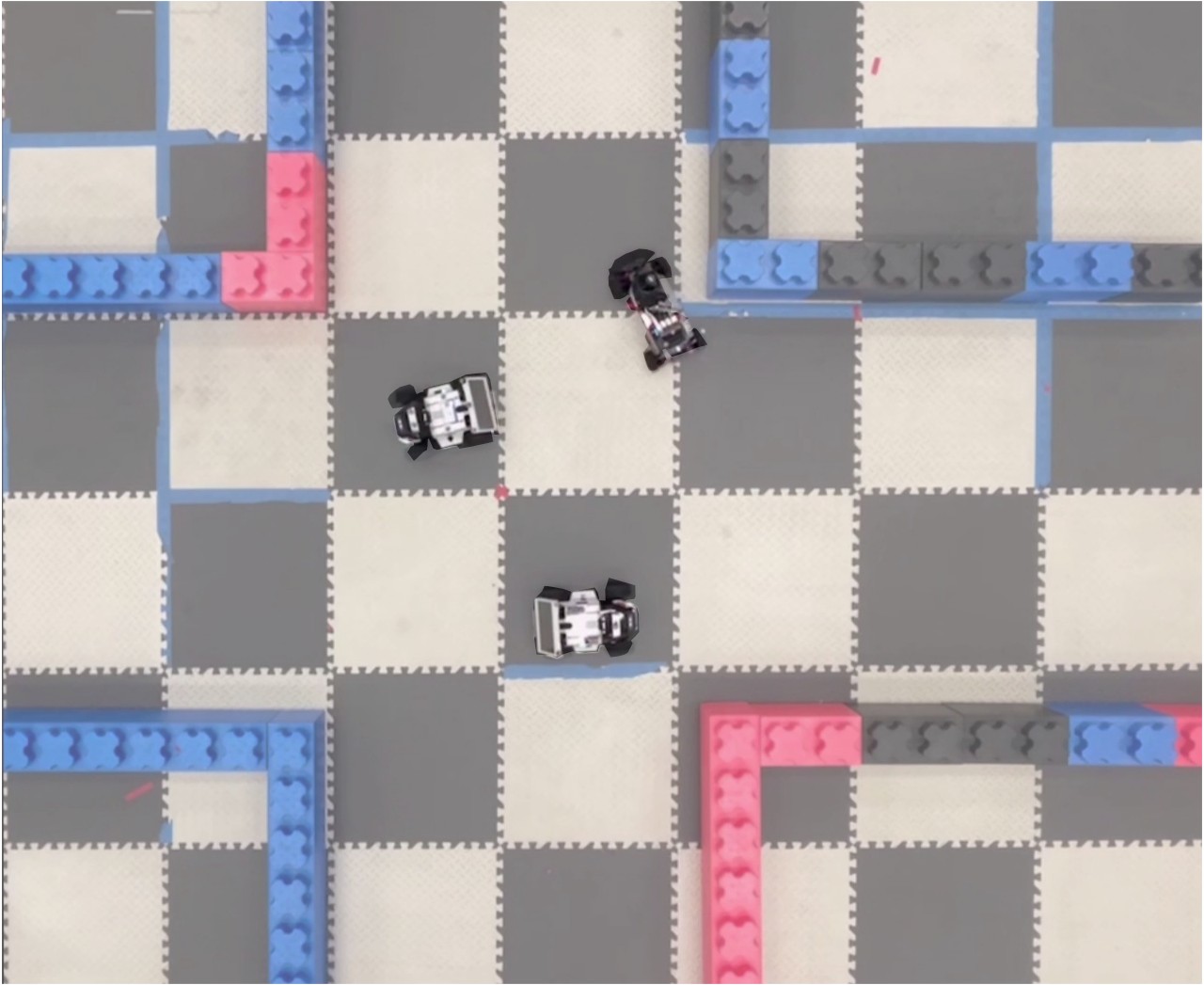}
\end{minipage}\hfill
\begin{minipage}{0.33\linewidth}
    \includegraphics[scale=0.085]{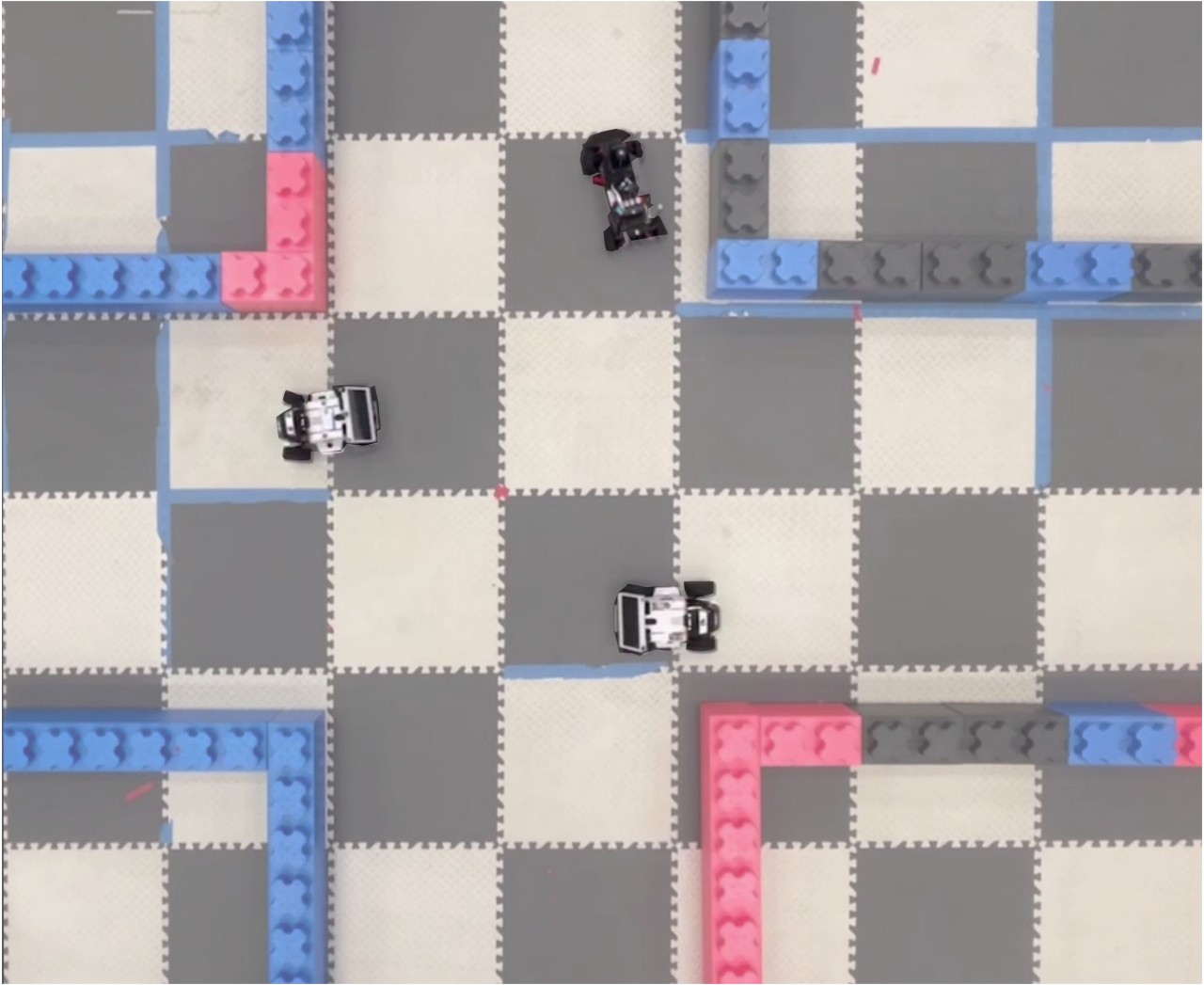}
\end{minipage}
\caption{Snapshots of real-world experiments for the intersection scenario under the case when OA1 is a slow agent and OA2 is a fast agent. In this case, the EA passes through the intersection in front of OA1 and from behind OA2.}
\label{fig:real2}
\end{figure}

\begin{figure}[!htb]
\centering
\begin{minipage}{0.33\linewidth}
    \includegraphics[scale=0.085]{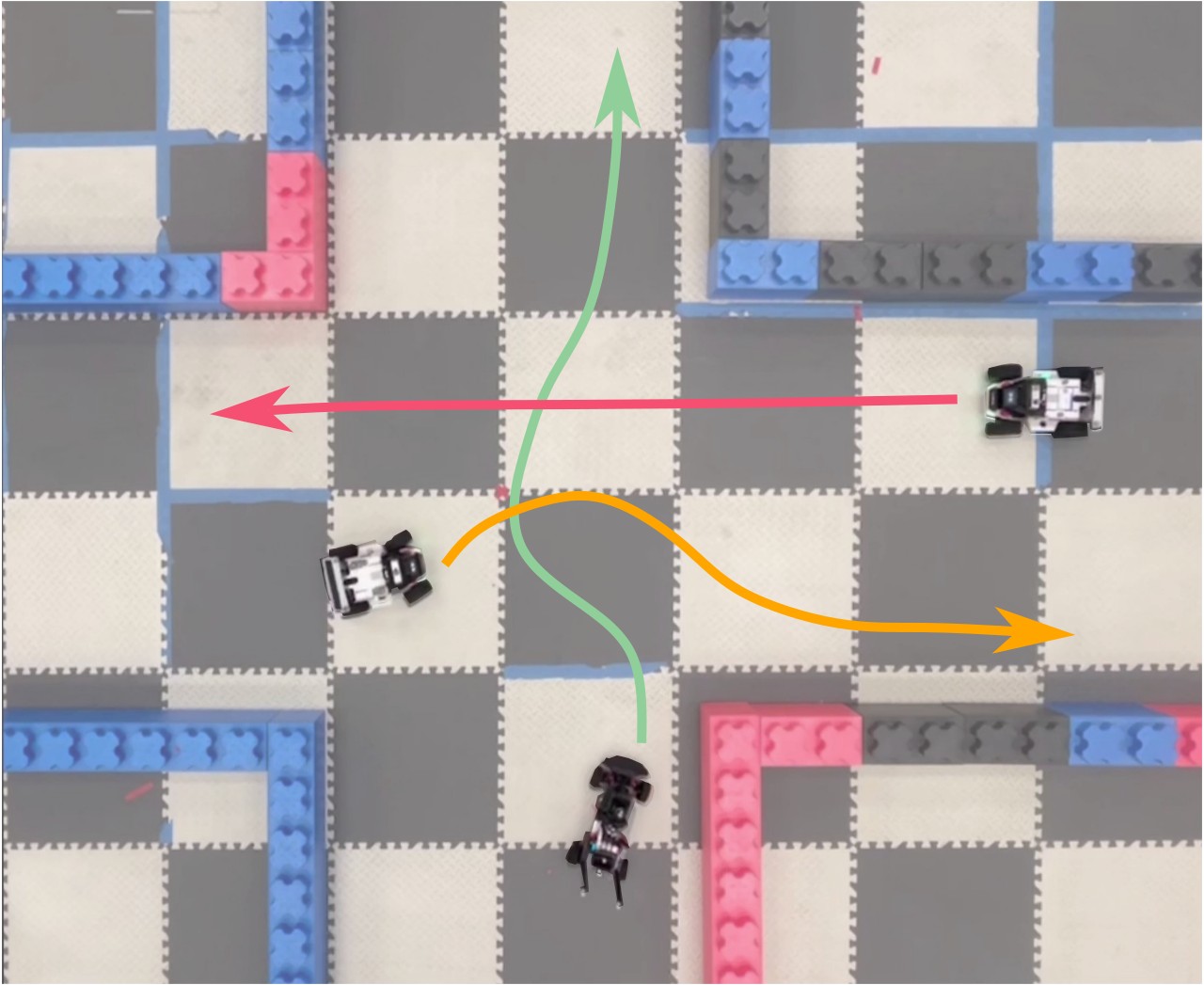}
\end{minipage}\hfill
\begin{minipage}{0.33\linewidth}
    \includegraphics[scale=0.085]{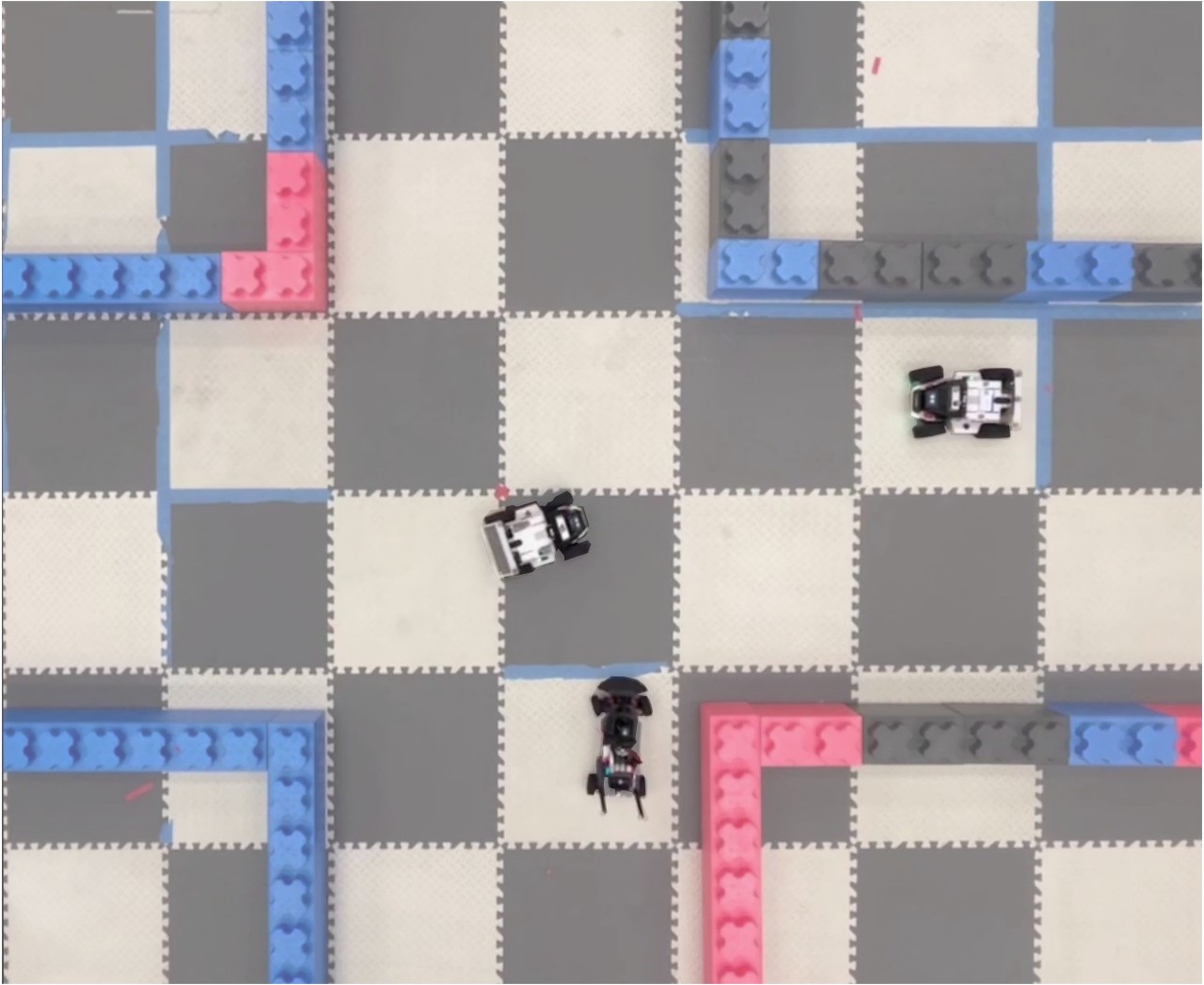}
\end{minipage}\hfill
\begin{minipage}{0.33\linewidth}
    \includegraphics[scale=0.085]{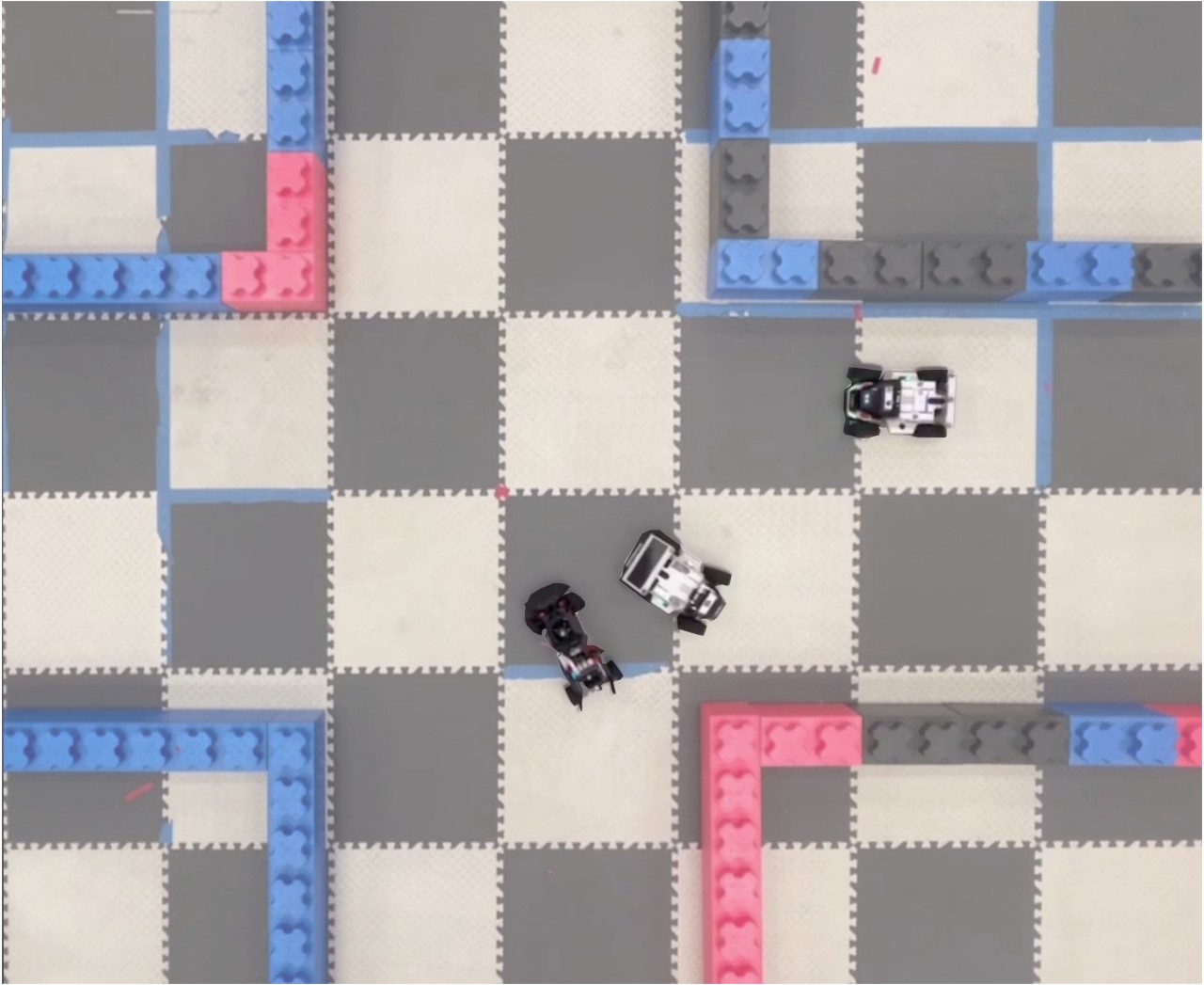}
\end{minipage}

\vspace{0.1cm}

\begin{minipage}{0.33\linewidth}
    \includegraphics[scale=0.085]{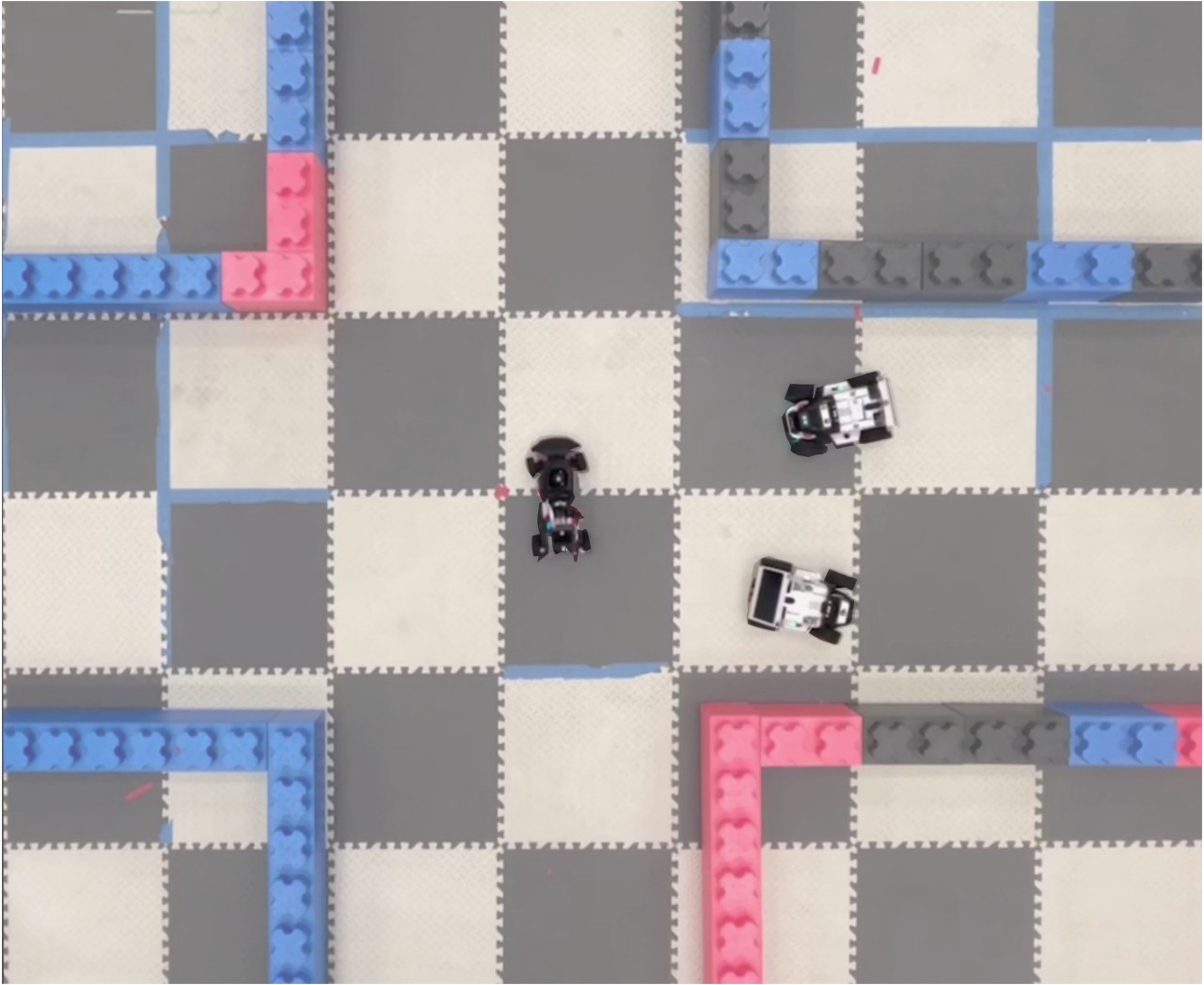}
\end{minipage}\hfill
\begin{minipage}{0.33\linewidth}
    \includegraphics[scale=0.085]{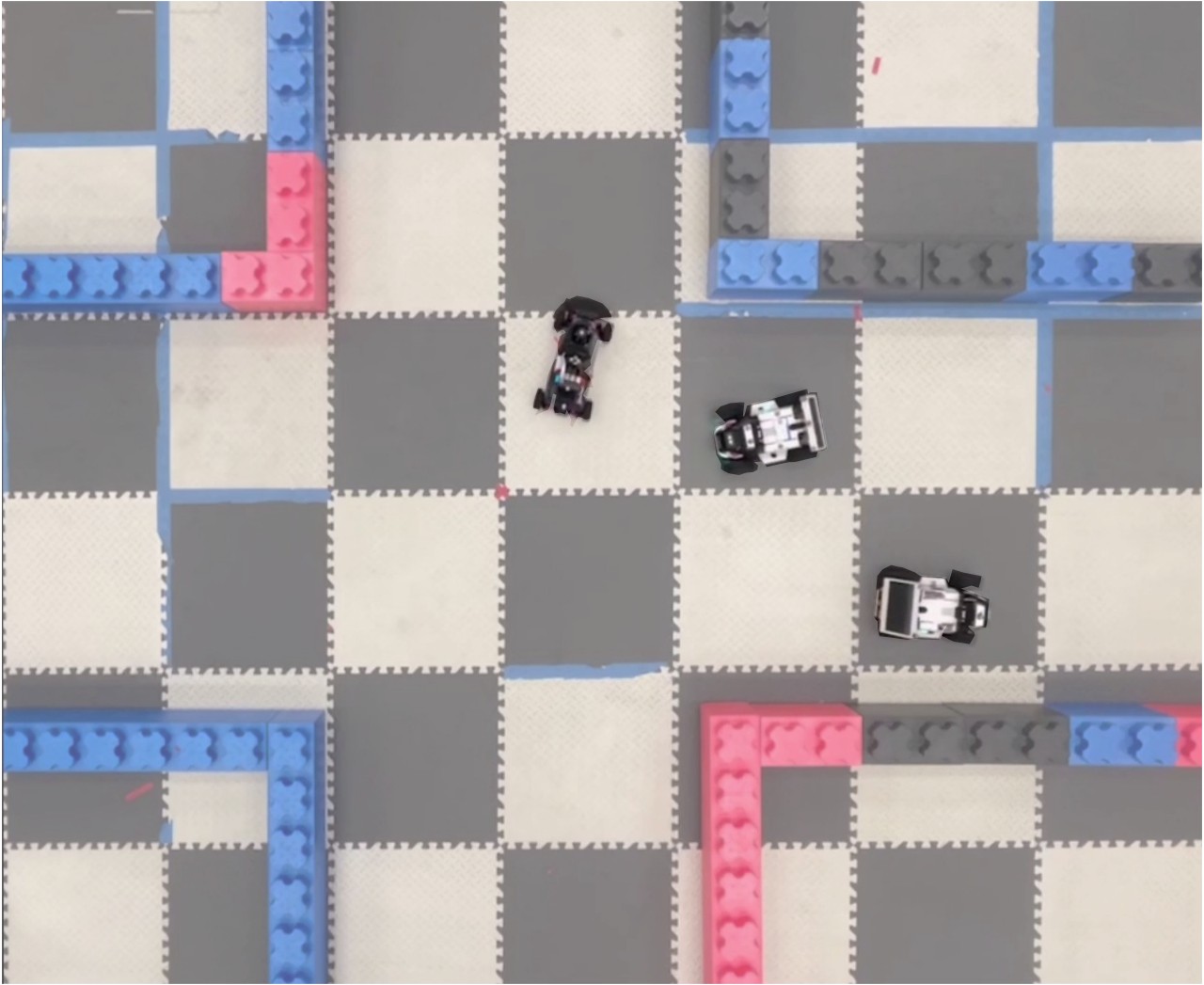}
\end{minipage}\hfill
\begin{minipage}{0.33\linewidth}
    \includegraphics[scale=0.085]{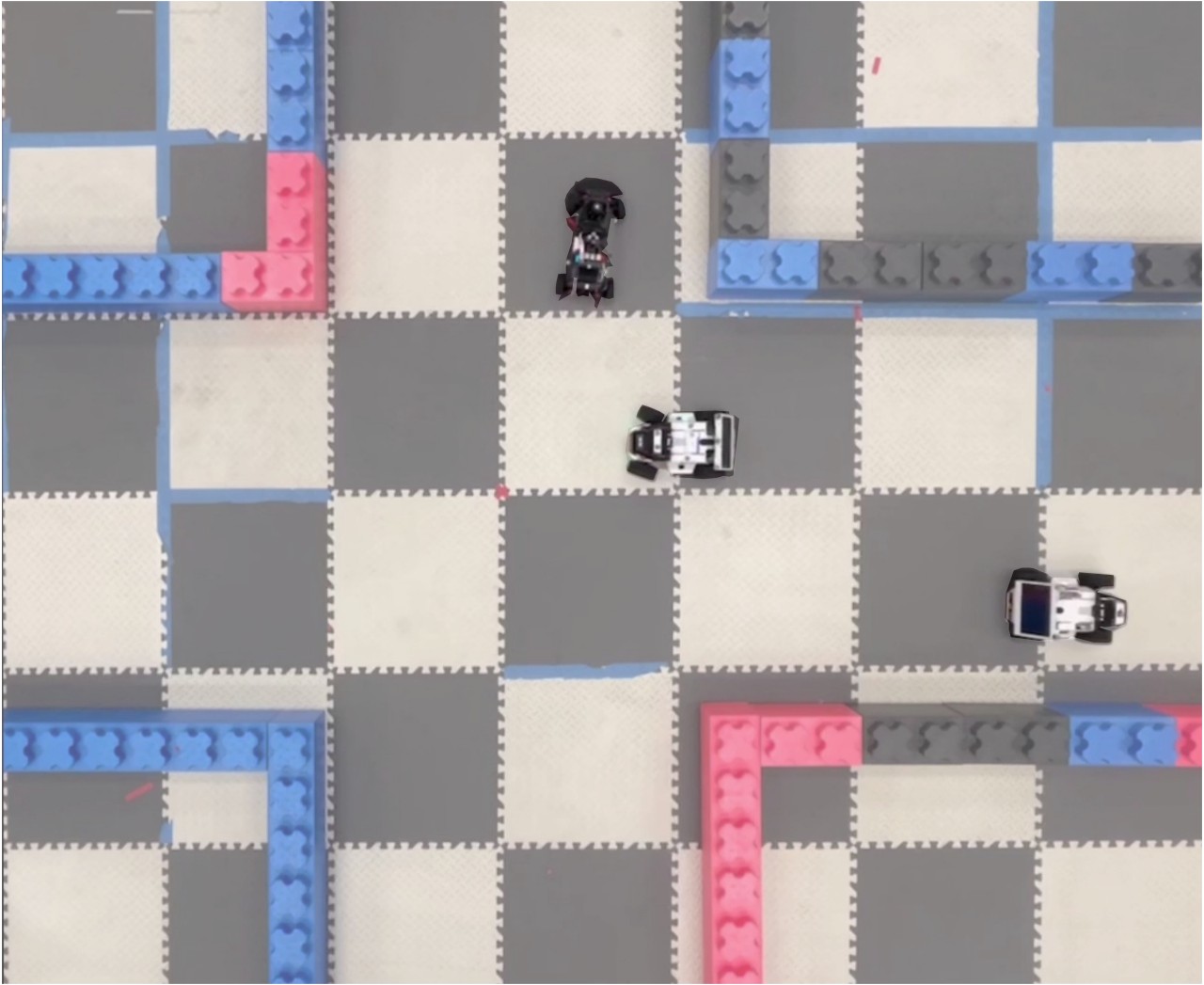}
\end{minipage}
\caption{Snapshots of real-world experiments for the intersection scenario under the case when OA1 is a fast agent and OA2 is a slow agent. In this case, the EA passes through the intersection from behind OA1 and in front of OA2.}
\label{fig:real3}
\end{figure}

\begin{figure}[!htb]
\centering
\begin{minipage}{0.33\linewidth}
    \includegraphics[scale=0.085]{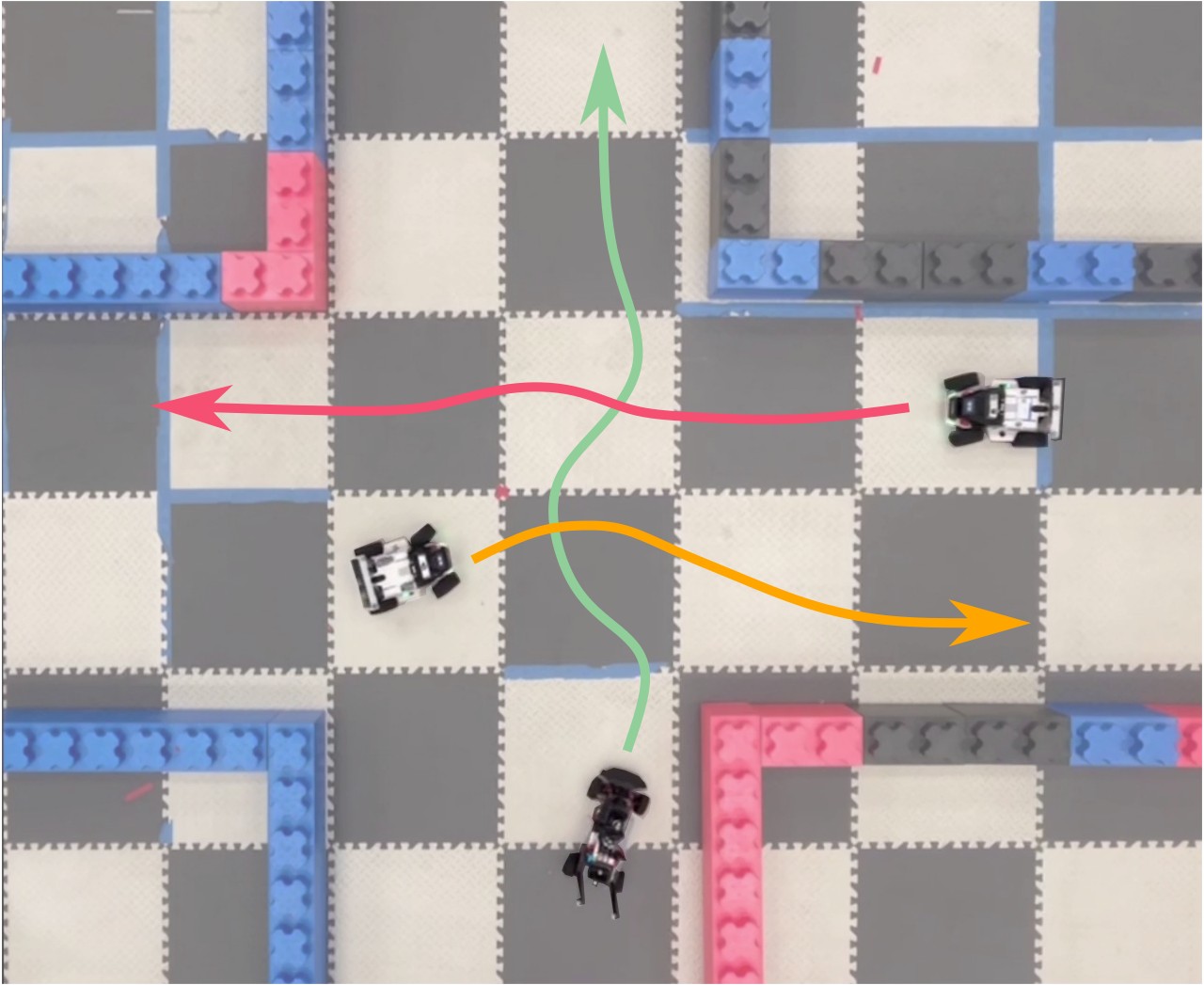}
\end{minipage}\hfill
\begin{minipage}{0.33\linewidth}
    \includegraphics[scale=0.085]{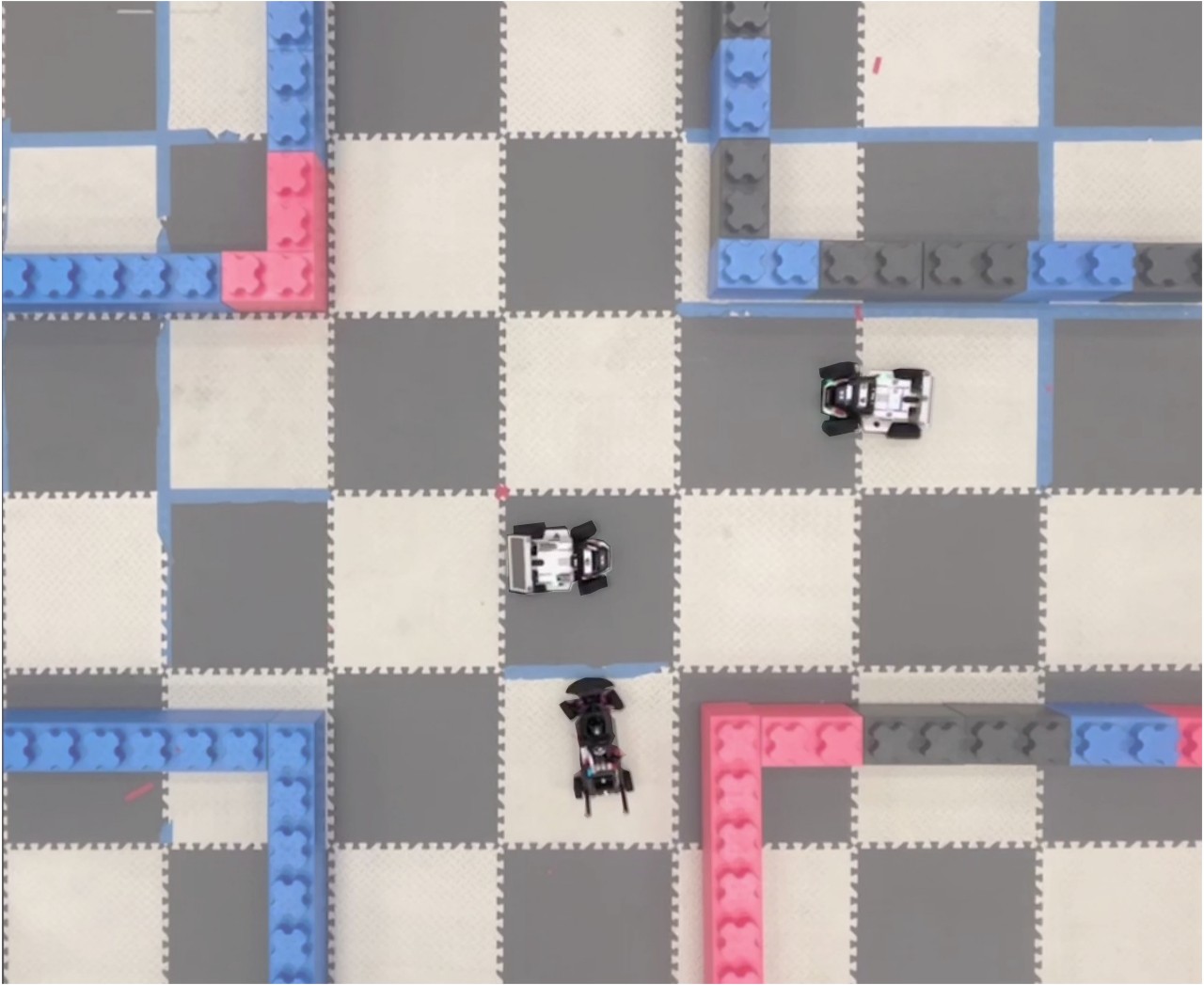}
\end{minipage}\hfill
\begin{minipage}{0.33\linewidth}
    \includegraphics[scale=0.085]{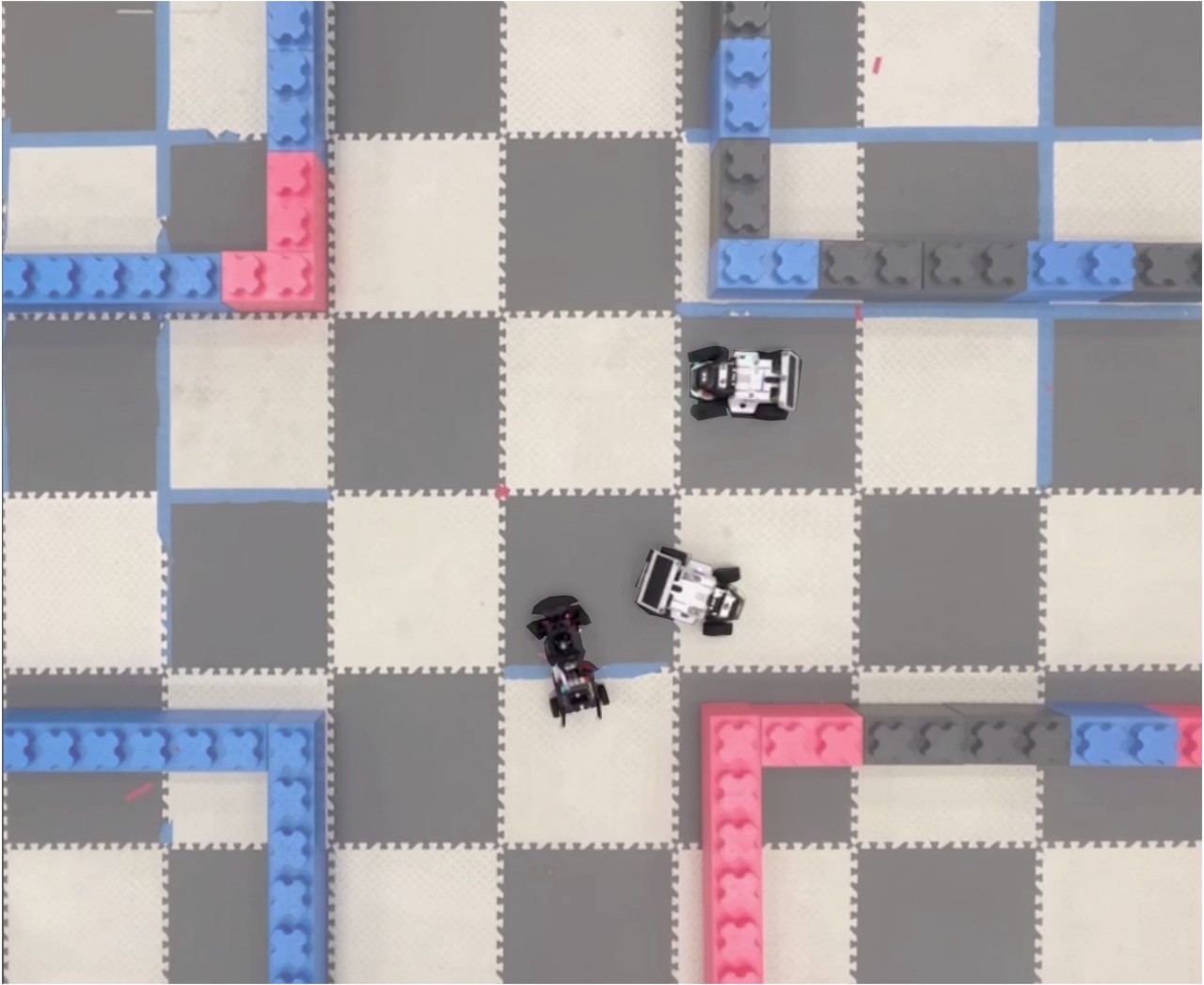}
\end{minipage}

\vspace{0.1cm}

\begin{minipage}{0.33\linewidth}
    \includegraphics[scale=0.085]{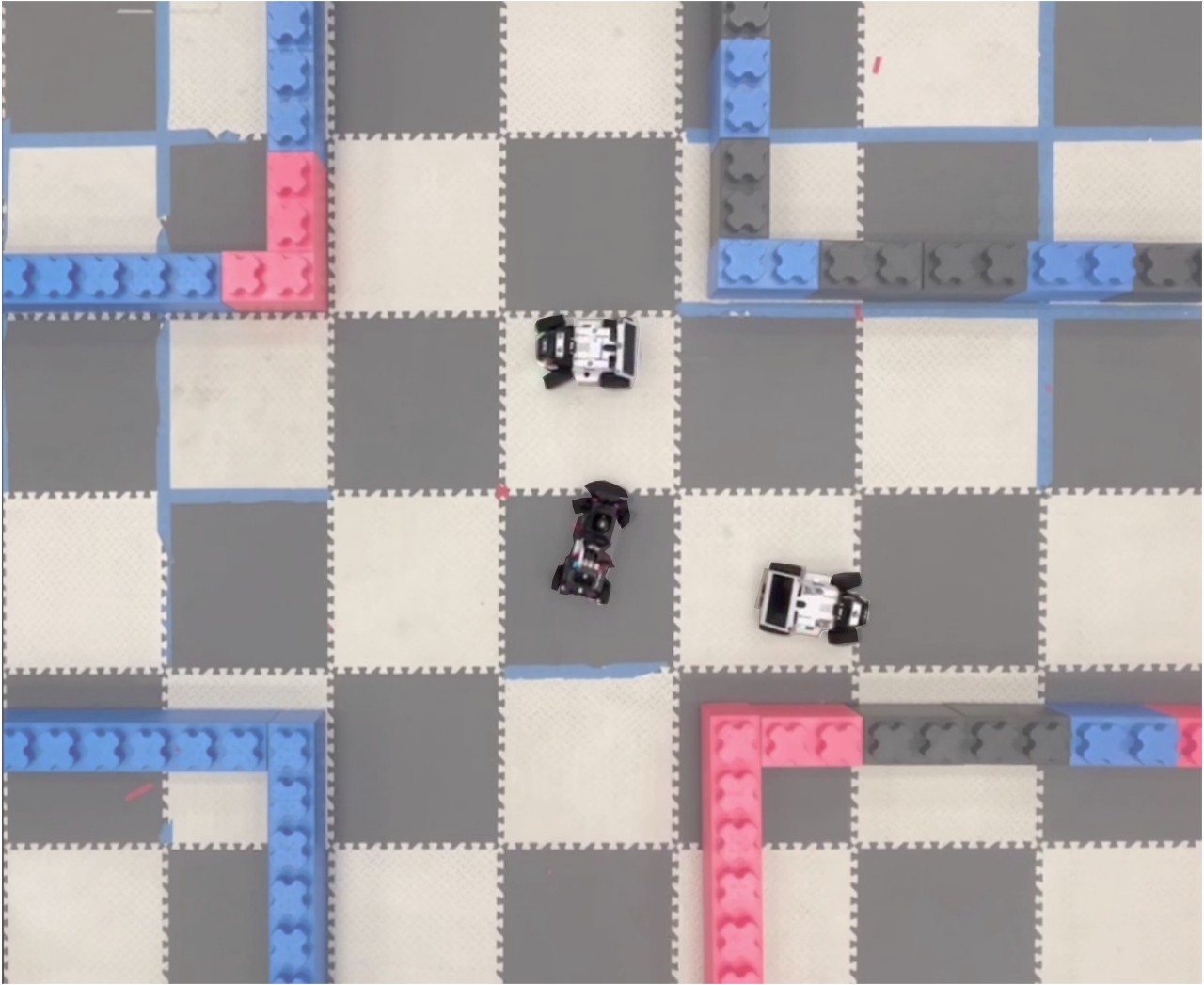}
\end{minipage}\hfill
\begin{minipage}{0.33\linewidth}
    \includegraphics[scale=0.085]{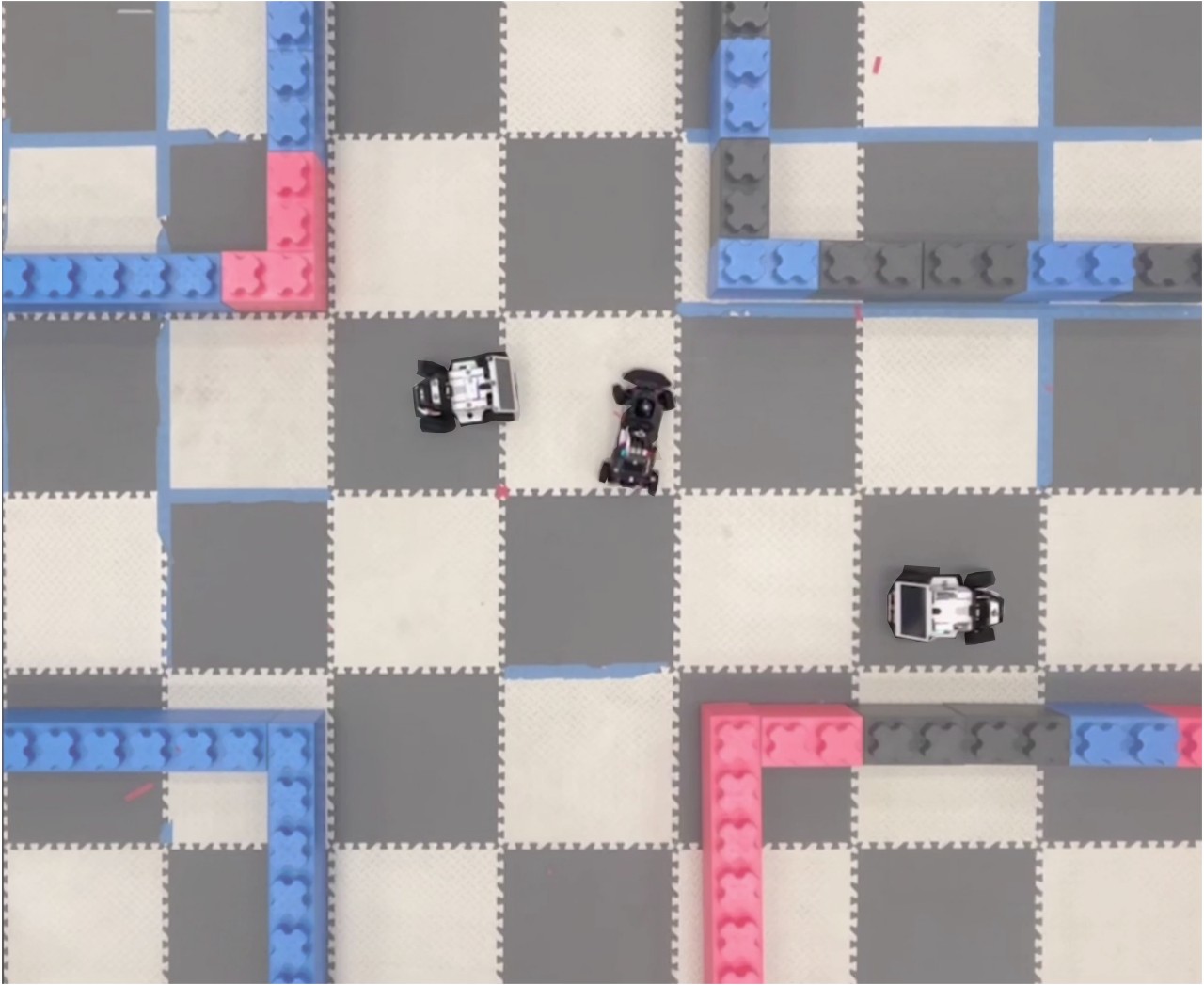}
\end{minipage}\hfill
\begin{minipage}{0.33\linewidth}
    \includegraphics[scale=0.085]{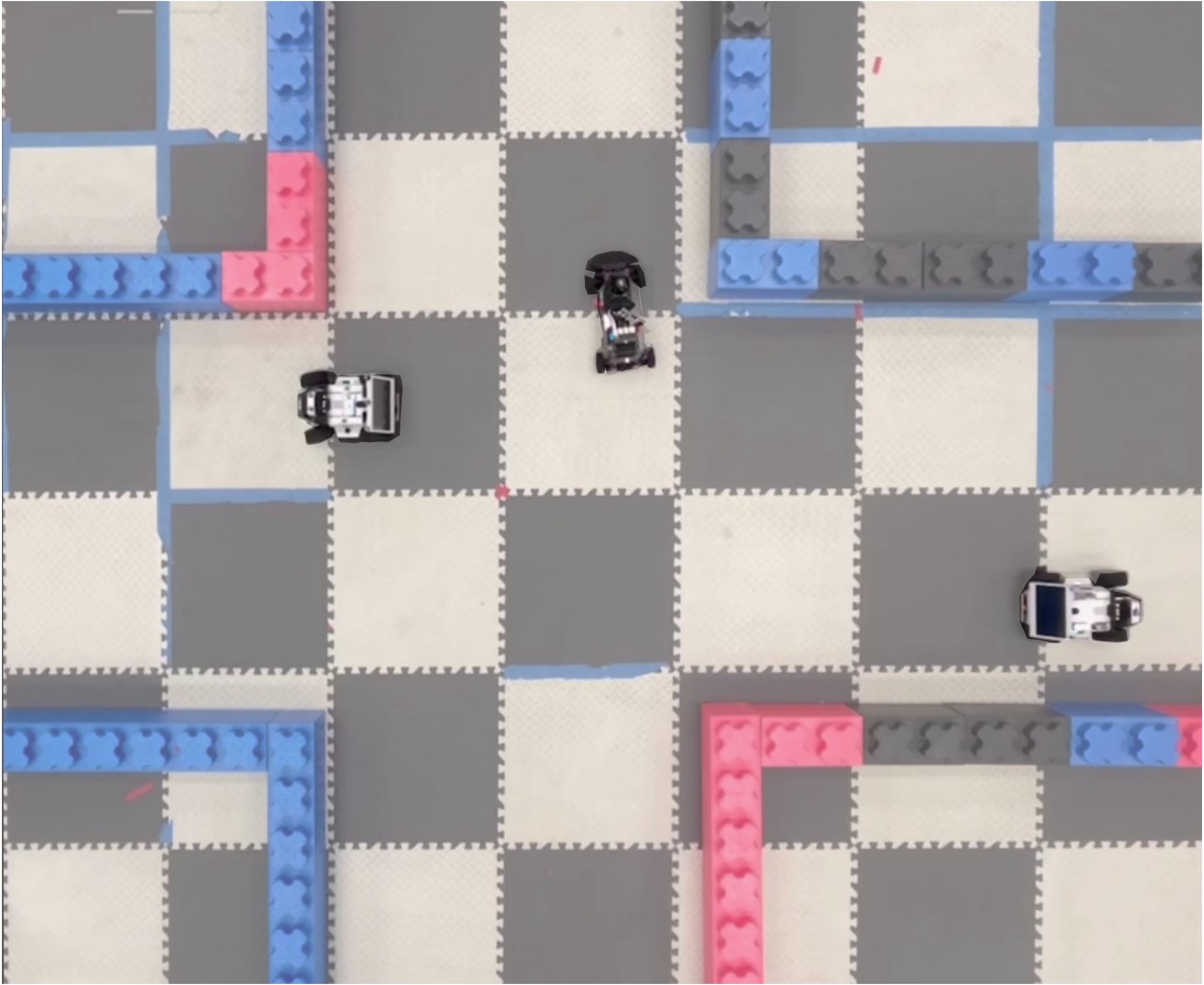}
\end{minipage}
\caption{Snapshots of real-world experiments for the intersection scenario under the case when both OA1 and OA2 are fast agents. In this case, the EA manages to pass through the intersection from behind both OAs.}
\label{fig:real4}
\end{figure}

\section{Conclusion}

In this work, we provide the key insight that the interactive trajectory planning problem with uncertainties, formulated as a Bayesian game, is essentially a potential game, which guarantees the existence and accessibility of the corresponding BNE by solving a unified optimization problem. With the potential game formulation, a novel optimization scheme is introduced to enable parallel solving of the game. Simulation and experimental results demonstrate the superiority of the proposed method in computational efficiency and scalability compared with existing baselines, while real-world experiments show that the proposed method is deployable to real-world systems with the real-time requirement satisfied.

While the proposed method manages to consider the intentional uncertainties of rival players, it does not consider other types of uncertainties that typically occur in interaction processes, such as behavioral uncertainties resulting from bounded rationality and observation uncertainties due to sensor noise. A possible future direction is to combine the proposed method with the maximum entropy model and/or the partially observable Markov decision process to address these uncertainties.

\begin{appendices}
\label{appendices:proof}
\section{Proof of Theorem 2}
\begin{proof}
For an arbitrary player $k\in\mathcal{N}$ and arbitrary type $t'_k\in\Tilde{T}_k$, we decompose the potential function as
\begin{equation}
\begin{aligned}
    P(X)&=P(X^{t'_k},X^{-t'_k})\\
    &=P_{t'_k}(X^{t'_k},X^{-t'_k})+P_{-t'_k}(X^{-t'_k})
\end{aligned}
\end{equation}
where
\begin{equation}
\begin{aligned}
    &P_{-t'_k}(X^{-t'_k})=\\
    &\sum_{i\neq k}\sum_{t_i\in\Tilde{T}_i}p(t_i)c_{t_i}(X^{t_i})+\sum_{t_k\in\Tilde{T}_k,t_k\neq t'_k}p(t_k)c_{t_k}(X^{t_k})\\
    &+\sum_{i\neq k,j\neq k,j> i}\sum_{t_i\in\Tilde{T}_i,t_j\in\Tilde{T}_j}p(t_i,t_j)c_{t_it_j}(X^{t_i},X^{t_j})\\
    &+\sum_{i<k}\sum_{t_i\in\Tilde{T}_i,t_k\in\Tilde{T}_k,t_k\neq t'_k}p(t_i,t_k)c_{t_it_k}(X^{t_i},X^{t_k})\\
    &+\sum_{i>k}\sum_{t_i\in\Tilde{T}_i,t_k\in\Tilde{T}_k,t_k\neq t'_k}p(t_k,t_i)c_{t_kt_i}(X^{t_k},X^{t_i}),
\end{aligned}
\end{equation}
and
\begin{equation}
\label{proof1-3}
\begin{aligned}
&P_{t'_k}(X^{t'_k},X^{-t'_k})\\
&=p(t'_k)c_{t'_k}(X^{t'_k})+\sum_{i<k}\sum_{t_i\in\Tilde{T}_i}p(t_i,t'_k)c_{t_it'_k}(X^{t_i},X^{t'_k})\\
    &+\sum_{i>k}\sum_{t_i\in\Tilde{T}_i}p(t'_k,t_i)c_{t'_kt_i}(X^{t'_k},X^{t_i})\\
    &=p(t'_k)c_{t'_k}(X^{t'_k})+\sum_{i\neq k}\sum_{t_i\in\Tilde{T}_i}p(t'_k,t_i)c_{t'_kt_i}(X^{t'_k},X^{t_i})
\end{aligned}
\end{equation}
where the second equality in (\ref{proof1-3}) follows from Assumption 3. It is easy to verify that $P_{t'_k}(X^{t'_k},X^{-t'_k})$ and $P_{-t'_k}(X^{-t'_k})$ sum up to $P(X)$. Meanwhile, for $C_{t_i}(X^{t_i},X^{-t_i})$, we have

\begin{equation}
\label{cost}
\begin{aligned}
&C_{t_i}(X^{t_i},X^{-t_i})\\
&=\sum_{t_{-i}\in\Tilde{T}_{-i}}p(t_{-i}|t_i)[c_{t_i}(X^{t_i})+\sum_{j\neq i}c_{t_it_j}(X^{t_i},X^{t_j})]\\
&=c_{t_i}(X^{t_i})+\sum_{t_{-i}\in\Tilde{T}_{-i}}p(t_{-i}|t_i)\sum_{j\neq i}c_{t_it_j}(X^{t_i},X^{t_j})\\
&=c_{t_i}(X^{t_i})+\sum_{j\neq i}\sum_{t_{-i}\in\Tilde{T}_{-i}}p(t_{-i}|t_i)c_{t_it_j}(X^{t_i},X^{t_j})\\
&=c_{t_i}(X^{t_i})+\sum_{j\neq i}\sum_{t_j\in\Tilde{T}_j}\sum_{t_{-ij}\in\Tilde{T}_{-ij}}p(t_{-i}|t_i)c_{t_it_j}(X^{t_i},X^{t_j})\\
&=c_{t_i}(X^{t_i})+\sum_{j\neq i}\sum_{t_j\in\Tilde{T}_j}p(t_j|t_i)c_{t_it_j}(X^{t_i},X^{t_j}).
\end{aligned}
\end{equation}
Therefore, we have 
\begin{equation}
P_{t'_k}(X^{t'_k},X^{-t'_k})=p(t'_k)C_{t'_k}(X^{t'_k},X^{-t'_k}),
\end{equation}
and the following conclusion
\begin{equation}
\label{proof}
\begin{aligned}
&P(X^{t'_k},X^{-t'_k})-P(\hat{X}^{t'_k},X^{-t'_k})\\
&=P_{t'_k}(X^{t'_k},X^{-t'_k})-P_{t'_k}(\hat{X}^{t'_k},X^{-t'_k})\\
&=p(t'_k)[C_{t'_k}(X^{t'_k},X^{-t'_k})-C_{t'_k}(\hat{X}^{t'_k},X^{-t'_k})].
\end{aligned}
\end{equation}
Comparing (\ref{proof}) and the cost function in Problem 2, we conclude that Theorem 2 follows from the fact that (\ref{proof}) holds for arbitrary $k$, $t'_k$, $X^{t'_k}$,$\hat{X}^{t'_k}$, and $X^{-t'_k}$.
\end{proof}

\section{Proof of Theorem 3}

\begin{proof}
We first decompose the potential function $P'(x)$ as
\begin{equation}
    P'(X) = P'_{EA}(X^{t^\Theta_{EA}},X^{t^\Theta_{-EA}})+P'_{-EA}(X^{t^\Theta_{-EA}})
\end{equation}
where
\begin{equation}
\begin{aligned}
    &P'_{EA}(X^{t^\Theta_{EA}},X^{t^\Theta_{-EA}})\\
    &=\sum_{\theta\in\Theta}p(\theta)c_{t^\theta_{EA}}(X^{t^\theta_{EA}}) + \sum_{i\neq EA}\sum_{\theta\in\Theta}p(\theta)c_{t^\theta_{EA}t^\theta_{i}}(X^{t^\theta_{EA}},X^{t^\theta_i})\\
    &=\sum_{\theta\in\Theta}p(\theta)C_{t^\theta_{EA}}(X^{t^\theta_{EA}},X^{t^\theta_{-EA}})
\end{aligned}
\end{equation}
and $P'_{-EA}(X^{t^\Theta_{-EA}})$ is irrelevant to $X^{t^\Theta_{EA}}$. This implies that
\begin{equation}
\label{proof1}
\begin{aligned}
    &P'(X^{t^\Theta_{EA}},X^{t^\Theta_{-EA}})-P'(\hat{X}^{t^\Theta_{EA}},X^{t^\Theta_{-EA}})\\
    &=P'_{EA}(X^{t^\Theta_{EA}},X^{t^\Theta_{-EA}})-P'_{EA}(\hat{X}^{t^\Theta_{EA}},X^{t^\Theta_{-EA}})\\
    &=\sum_{\theta\in\Theta}p(\theta)[C_{t^\theta_{EA}}(X^{t^\theta_{EA}},X^{t^\theta_{-EA}})-C_{t^\theta_{EA}}(\hat{X}^{t^\theta_{EA}},X^{t^\theta_{-EA}})].
\end{aligned}
\end{equation}
Meanwhile, for $i\neq EA$ and for arbitrary $\theta\in\Theta$, we have
\begin{equation}
    P'(X) = P'_{i,\theta}(X^{t^\theta_i},X^{t^\theta_{-i}})+P'_{-i,\theta}(X^{t^\theta_{-i}}) + P'_{-\theta}(X^{t^{-\theta}})
\end{equation}
where
\begin{equation}
\begin{aligned}
    &P'_{i,\theta}(X^{t^\theta_i},X^{t^\theta_{-i}})=p(\theta)c_{t^\theta_i}(X^{t^\theta_i})+\sum_{j\neq i}p(\theta)c_{t^\theta_it^\theta_j}(X^{t^\theta_i},X^{t^\theta_j})\\
    &=p(\theta)C_{t^\theta_i}(X^{t^\theta_i},X^{t^\theta_{-i}})
\end{aligned}
\end{equation}
and the other two terms are irrelevant to $X^{t^\theta_i}$. Again, we have
\begin{equation}
\label{proof2}
\begin{aligned}
&P'(X^{t^\theta_i},X^{t^\theta_{-i}},X^{t^{-\theta}})-P'(\hat{X}^{t^{\theta}_i},X^{t^\theta_{-i}},X^{t^{-\theta}})\\
&=P'_{i,\theta}(X^{t^\theta_i},X^{t^\theta_{-i}})-P'_{i,\theta}(\hat{X}^{t^\theta_i},X^{t^\theta_{-i}})\\
&=p(\theta)[C_{t^\theta_i}(X^{t^\theta_i},X^{t^\theta_{-i}})-C_{t^\theta_i}(\hat{X}^{t^\theta_i},X^{t^\theta_{-i}})].
\end{aligned}
\end{equation}
We also notice that fact that (\ref{OA}) is equivalent to the following problem with $p(\theta)>0$:
\begin{equation}
\label{OA'}
\begin{aligned}
\mathcal{S}^{OA}_{i,\theta}(X^{t^\theta_{-i}})&:=\argmin_{X^{t^\theta_i}}\ p(\theta)C_{t^\theta_i}(X^{t^\theta_i},X^{t^\theta_{-i}})\ \textup{s.t.}\ X^{t^\theta_i}\in\mathcal{X}^{t^\theta_i},\\
& \forall \theta\in\Theta, \forall i\in\mathcal{N}, i\neq \textup{EA}.
\end{aligned}
\end{equation}
Comparing (\ref{proof1}) and (\ref{proof2}) with the cost of (\ref{EA}) and (\ref{OA'}), respectively, we conclude that Theorem 3 holds.
\end{proof}

\section{Proof of Theorem 4}
\begin{proof}
Algorithm 1 is derived by directly applying ADMM algorithm to problem (\ref{dual2}), and following \cite[Cor.~28.3]{bauschke2017correction} it holds that the iterates $\{y^k_v\}$ and $\{z^k_v\}$ converges to $y^*_v$ for all $v\in\mathcal{V}$, such that $[y^*_v]_{v\in\mathcal{V}}$ is a minimizer of the dual problem (\ref{dual}). Further, the first-order optimality conditions corresponding to a primal-dual solution of problem (\ref{prime}) are given by
\begin{subequations}
\label{optimality}
\begin{align}
0&\in\partial_{X^v}\mathcal{L}(X,w,y)=\partial {f_v}(X^v)+\sum_{e\in\textup{Adj}(v)}Q_{v,e}^\top y_e, \forall v\in\mathcal{V},\label{optimality1}\\
0&\in\partial_{w_e}\mathcal{L}(X,w,y)=\partial g_e(w_e)-y_e, \forall e\in\mathcal{E},\label{optimality2}\\
0&=\nabla_{y_e}\mathcal{L}(X,w,y)=\sum_{v\in e}Q_{v,e}X^v-w_e.\label{optimality3}
\end{align}
\end{subequations}
We denote $w_e^k:=\sum_{v\in e}E_{v,e}s_v^k$, $w^*_e:=\lim_{k\rightarrow\infty}w_e^k$,$y^*_e:=E_{v,e}y^*_v$ for arbitrary $v\in e$, and $X^{v*}:=\lim_{k\rightarrow\infty}X^{k,v}$. It suffices to proof Theorem 3 by showing that $(X^{v*},w_e^*,y^*_e)$ satisfies (\ref{optimality}).
Step 9 of Algorithm 1 ensures that $X^{k,v}$ is a minimizer of (\ref{Xupdate}) and hence
\begin{equation}
\begin{aligned}
0&\in\partial f_v(X^{k,v})+\frac{1}{\sigma+\rho}Q_v^\top(Q_vX^{k,v}+r^k_v)\\
&=\partial f_v(X^{k,v})+Q_v^\top y^k_v
\end{aligned}
\end{equation}
where the equality follows from (\ref{yupdate}) and the conclusion holds for all $k$ and $v$. Taking the limit $k\rightarrow\infty$ yields
\begin{equation}
    0\in\partial f_v(X^{v*})+Q_v^\top y^*_v=\partial f_v(X^{v*})+\sum_{e\in\textup{Adj}(v)}Q_{v,e}^\top y^*_e
\end{equation}
where the equality follows from $Q_v=[Q_{v,e}]_{e\in\textup{Adj}(v)}$ and $y^*_v=[y^*_e]_{e\in\textup{Adj}(v)}$, thus showing (\ref{optimality1}). Next, Step 10 of Algorithm 1 yields that 
\begin{equation}
\label{optimality_zupdate}
\begin{aligned}
0\in\partial g^*_v(z_v^{k+1})+\sigma (z^{k+1}_v-y^{k+1}_v)-s^k_v
\end{aligned}
\end{equation}
holds for all $v$ and $k$. For a particular $e\in\mathcal{E}$, we pre-multiply (\ref{optimality_zupdate}) by $E_{v,e}$, sum over $v\in e$, and take the limits $k\rightarrow\infty$, which yields
\begin{equation}
\label{proof3}
\begin{aligned}
    0&\in\lim_{k\rightarrow\infty}\sum_{v\in e}E_{v,e}\partial g^*_v(z_v^{k+1})\\
    &+\sigma\sum_{v\in e}E_{v,e}(z^{k+1}_v-y^{k+1}_v)-\sum_{v\in e}E_{v,e}s^k_v\\
    &=\sum_{v\in e}E_{v,e}\partial g^*_v(y_v^*)-w^*_e=\sum_{v\in e}\frac{1}{N_e}\partial g^*_e(y^*_e)-w^*_e.
\end{aligned}
\end{equation}
Essentially, (\ref{proof3}) is obtained from the fact that $y^k_v,z^k_v\rightarrow y^*_v$. From (\ref{proof3}) we have
\begin{equation}
    w^*_e\in\partial g^*_e(y^*_e)\Rightarrow y^*_e\in\partial g_e(w^*_e)
\end{equation}
by invoking \cite[Cor. 16.30]{bauschke2017correction}, which shows (\ref{optimality2}). To show (\ref{optimality3}), we first plug (\ref{rupdate}) into (\ref{yupdate}) to obtain
\begin{equation}
\begin{aligned}
\label{optimality_yupdate}
        y^{k+1}_v&=\frac{1}{\sigma+\rho}(Q_vX^{k+1,v}+\sigma z^k_v-\lambda^k_v-s^k_v\\
        &+\sum_{e'\in\textup{Adj}(v)}\sum_{v'\in e'}\frac{\rho}{N_{e'}}E^\top_{v,e'}E_{v',e'}y^k_{v'}).
\end{aligned}
\end{equation}
From Steps 12 and 13 of Algorithm 1, it is not difficult to see that
\begin{equation}
    \sum_{v\in e}E_{v,e}\lambda^k_{v}=\sum_{v\in e}\lambda^k_{v,e}=0
\end{equation}
for all $k$ and $e$ given zero initialization. We again 
pre-multiply (\ref{optimality_yupdate}) by $E_{v,e}$, sum over $v\in e$, and take the limits $k\rightarrow\infty$, which yields
\begin{equation}
\label{optimality_yupdate2}
\begin{aligned}
N_e y^*_e&=\frac{1}{\sigma+\rho}(\sum_{v\in e}Q_{v,e}X^{v*}+\sigma N_e y^*_e-w^*_e\\
&+\sum_{v\in e}\sum_{e'\in\textup{Adj}(v)}\sum_{v'\in e'}\frac{\rho}{N_{e'}}E_{v,e}E^\top_{v,e'}E_{v',e'}y^*_{v'}).
\end{aligned}
\end{equation}
The last term in (\ref{optimality_yupdate2}) is simplified as 
\begin{equation}
\begin{aligned}
&\sum_{v\in e}\sum_{e'\in\textup{Adj}(v)}\sum_{v'\in e'}\frac{\rho}{N_{e'}}E_{v,e}E^\top_{v,e'}E_{v',e'}y^*_{v'}=\\
&\sum_{v\in e}\sum_{e'\in\textup{Adj}(v)}\sum_{v'\in e'}\frac{\rho}{N_{e'}}E_{v,e}E^\top_{v,e'}y^*_{e'}=\\
&\sum_{v\in e}\sum_{e'\in\textup{Adj}(v)}\rho E_{v,e}E^\top_{v,e'}y^*_{e'}=\\
&\sum_{v\in e}\rho E_{v,e}y^*_{v}=\rho N_e y^*_e.
\end{aligned}
\end{equation}
Plugging this result into (\ref{optimality_yupdate}) yields (\ref{optimality3}). We conclude the proof by establishing that $(X^{v*},w_e^*,y^*_e)$ satisfies (\ref{optimality}).
\end{proof}
\end{appendices}

\bibliographystyle{IEEEtran}
\bibliography{refs}

\end{document}